\documentclass[10pt]{article} 
\usepackage[preprint]{tmlr}

\usepackage{hyperref}
\usepackage{url}
\usepackage{microtype}
\usepackage{graphicx}
\usepackage{subcaption}
\usepackage{booktabs} 
\usepackage{multirow}
\usepackage{adjustbox}
\usepackage{amsmath}
\usepackage{amssymb}
\usepackage{mathtools}
\usepackage{amsthm}

\title{ConquerNet: Convolution-Smoothed Quantile ReLU Neural Networks with Minimax Guarantees}

\author{\name Tianpai Luo \email ltp21@mails.tsinghua.edu.cn \\
      \addr Department of Statistics and Data Science\\
            Tsinghua University\\
            Beijing, 100084, China
      \AND
      \name Fangwei Wu \email wfw19@mails.tsinghua.edu.cn \\
      \addr Department of Statistics and Data Science\\
            Tsinghua University\\
            Beijing, 100084, China
      \AND
      \name Weichi Wu \email wuweichi@tsinghua.edu.cn\\
      \addr Department of Statistics and Data Science\\
            Tsinghua University\\
            Beijing, 100084, China}

\newcommand{\mf}{\mathbf}
\newcommand{\mr}{\mathrm}
\newcommand{\mb}{\mathbb}

\newtheorem{theorem}{Theorem}[section]

\newtheorem{lemma}[theorem]{Lemma}

\newtheorem{definition}{Definition}[section]

\newtheorem{remark}{Remark}[section]

\def\month{MM}  
\def\year{YYYY} 
\def\openreview{\url{https://openreview.net/forum?id=XXXX}} 

\begin{document}

\maketitle

\begin{abstract}
Quantile regression is a fundamental tool for distributional learning but poses significant optimization challenges for deep models due to the non-smoothness of the pinball loss. We propose ConquerNet, a class of \textbf{con}volution-smoothed \textbf{qu}antil\textbf{e} \textbf{R}eLU neural \textbf{net}works, which yield smooth objectives while preserving the underlying quantile structure. We establish general nonasymptotic risk bounds for ConquerNet under mild conditions, providing minimax guarantees over Besov function classes. In numerical studies, we demonstrate that the proposed approach outperforms standard quantile neural networks at multiple quantile levels, showing improved estimation accuracy and training efficiency across the board, with particularly pronounced advantages at high and low quantiles.
\end{abstract}

\section{Introduction}
\label{sec:introduction}
Quantile regression is a widely considered statistical tool for modeling heterogeneous effects and capturing the distributional structure of responses beyond the conditional mean. In many fields such as quantitative finance, survival analysis, and econometrics \citep{baur2019quantile,horowitz1998bootstrap,chernozhukov2005iv}, quantile regression is used to understand the tail behaviors and provide robust outcomes faced with skewed or heavy-tailed data. 
Formally, given quantile level $\tau\in(0,1)$ and i.i.d. samples $\{(\mf x_i,y_i)\}_{i=1}^n$ from the random vector $(X,Y)$ where $X\in\mb R^d, Y\in\mb R$, the conditional $\tau$-quantile function is defined as 
\begin{equation}
f_\tau^*(\mf x_i)= F^{-1}_{y_i|\mf x_i}(\tau)=\inf\{y: \mr P(y_i\leq y|\mf x_i)\geq \tau \}.\label{eq:def f_tau}
\end{equation}
The standard approach estimates the conditional quantile function $f_\tau^*$ by minimizing the empirical quantile loss
\begin{equation}
    \hat{f}_\tau = \mathop{\arg\min}_f\sum_{i=1}^n \rho_\tau(y_i-f(\mf x_i)),\label{eq:quantile loss}
\end{equation}
where $\rho_\tau(u) = \max\{\tau u,(\tau-1)u\}$. We refer to \citet{koenker2017handbook} for an overview of quantile regression models.
Although the estimator in \eqref{eq:quantile loss} is theoretically well-founded, its loss function is non-differentiable. In models with a large parameter space, optimization becomes challenging so that efficient training is difficult. The neural network is a class of such models: even with a fixed input dimension $d$, the number of parameters can be very large. The flexibility of neural networks enables approximation of complex functions and, in theory, achieving minimax-optimal performance over Besov spaces \citep{Suzuki2018AdaptivityOD}. 
The success of neural networks depends heavily on generalization. Gradient-based optimizers such as stochastic gradient descent (SGD) often generalize well \citep{wu2022alignment,dziugaite2017computing}, but sharp minima \citep{hochreiter1997flat} and non-smooth loss surfaces \citep{huang2020understanding, foret2021sharpnessaware} can degrade performance, especially with large-batch training \citep{keskar2016large}. 
This is because the high sensitivity to parameter changes caused by the sharp minima reduces the stability of optimization and model performance, and thus 
weakens generalization.
Unfortunately, the quantile loss $\rho_{\tau}(u)$ is non-differentiable at $u=0$ forming a sharp minima, and its gradient jumps abruptly from $\tau$ to $\tau-1$ nearby. 

To address the difficulties brought by the sharp minima and non-smooth loss functions, smoothing techniques have been developed, as seen in \citet{berrada2018smooth}, which utilizes a smoothed loss function for classification tasks. For the quantile loss function, the convolution-type smoothed quantile loss \citep{fernandes2021smoothing} is particularly attractive, as it preserves both convexity and differentiability while retaining the robustness of quantile regression. Existing applications of convolution-type smoothing have been largely confined to linear models, and it remains unclear whether its advantages carry over to more flexible function classes. 
In this paper, we propose ConquerNet, a novel framework that integrates convolution-type smoothing with neural networks. We demonstrate that this architecture mitigates optimization difficulties, preserves statistical guarantees, and leverages the expressive capacity of deep learning models.

\subsection{Related work}
Our work builds upon smoothed quantile regression, neural networks, and minimax analysis in Besov spaces. For smoothed quantile regression, \cite{horowitz1998bootstrap} introduced a kernel-based smoothing approach that alleviates non-differentiability but sacrifices convexity, leading to challenging optimization problems. More recently, \cite{fernandes2021smoothing} proposed the convolution smoothing method, 
which preserves both convexity and differentiability while gaining smoothness. The convolution-type smoothing framework has since been widely applied in quantile regression, and subsequent works have established its strong statistical guarantees, including minimax optimality and asymptotic as well as nonasymptotic properties; see \cite{kaplan2017smoothed}, \cite{tan2022high}, and \cite{he2023smoothed}. Apart from parametric models, \cite{hu2025estimation} proposed a local linear convolution-type smoothing estimator for time-varying coefficient models. However, the convolution-type smoothing estimator has primarily been developed in linear models, with limited exploration in nonlinear nonparametric settings. To the best of our knowledge, nonlinear extensions of the convolution-type smoothing estimator remain underdeveloped, particularly in the context of modern neural networks. This motivates our focus, since \cite{Suzuki2018AdaptivityOD} showed that deep neural networks can achieve minimax rate in Besov spaces where classical linear estimators such as kernel ridge regression, Nadaraya–Watson, or sieve methods are suboptimal.

Alongside these developments, a growing body of work has applied neural networks to quantile regression. Quantile networks have found applications in credit portfolio analysis, transportation, and survival analysis; see \cite{feng2010robust}, \cite{ rodrigues2020beyond}, and \cite{pearce2022censored}. On the theoretical side, statistical guarantees for quantile networks have been investigated, such as error bounds and minimax optimality; see \cite{padilla2022quantile}, \cite{shen2024nonparametric}, and \cite{shen2025deep}. Especially for minimax rates, \cite{padilla2022quantile} established that ReLU networks achieve near-minimax rates for quantile regression in Besov spaces. More recent work has also considered settings with covariate shift and noncrossing constraints \citep{feng2024deep, shen2025deep}, but these results remain restricted to H\"older classes. Together, this literature suggests that bridging smoothed quantile regression with the expressive power of neural networks is a promising and unexplored direction.

\subsection{Contributions}
In this paper, we demonstrate that the convolution-type smoothing quantile regression technique can be effectively integrated with neural networks to advance quantile regression. {Our study contributes to distributional learning through smoothed quantile objectives, which naturally relate to quantile-based modeling and theoretical guarantees—areas of growing interest within the deep learning community \citep{Suzuki2018AdaptivityOD,sun2022conditional, nishimura2024minimax, kelen2025distributionfree}.} The main contributions are as follows:
\begin{itemize}
    \item[(i)] On the theoretical side, we first prove that the proposed ConquerNet estimator attains the minimax convergence rate over Besov spaces, up to a logarithmic factor. In addition, we establish general nonasymptotic error bounds that apply without assuming specific smoothness conditions on the target function. These results demonstrate that the combination of convolution-type smoothing and neural networks retains the desirable statistical guarantees of nonparametric quantile estimation while leveraging the expressive capacity of deep learning models.
    \item[(ii)] On the methodological and empirical side, we propose the ConquerNet architecture, which extends the convolution-type smoothing framework from linear nonparametric models to deep neural networks, thereby enabling its application in highly nonlinear nonparametric settings. Simulation studies are conducted to assess the performance of the proposed method, and the results show consistent improvements over existing quantile networks in terms of both estimation accuracy and computational efficiency. Together with the theoretical findings, these empirical results highlight the effectiveness of applying convolution-type smoothing to modern neural network architectures.
\end{itemize}

The rest of the article is organized as follows. 
We introduce the convolution-type smoothing framework with ReLU networks in Section \ref{sec:preliminaries}. In Section \ref{sec:theory}, we present the minimax rate for our estimation in Besov spaces and further develop general upper bounds. Simulation studies and real data analysis are conducted in Section \ref{sec:empirical study}. 
Finally, Section~\ref{sec:discussion} concludes with a discussion of potential extensions and future research directions.
All proofs and additional simulation results are given in the Appendix.

\section{Preliminaries}
\label{sec:preliminaries}
Before starting our main result, we specify some notations regarding the ReLU neural networks with layer parameter $L$, node parameter $W$, sparsity constraint $S$, and norm constraint $B$. In specific, we define
\begin{align}
\mathcal{I}(L, W, S, B):= & \left\{\left(A^{(L)} \sigma(\cdot)+b^{(L)}\right) \circ \cdots \circ\left(A^{(1)} x+b^{(1)}\right): A^{(1)} \in \mathbb{R}^{W \times d}, b^{(1)} \in \mathbb{R}^d,\right.\notag\\
&\left. A^{(L)} \in \mathbb{R}^{1 \times W}, b^{(L)} \in \mathbb{R},
A^{(l)} \in \mathbb{R}^{W \times W}, b^{(l)} \in \mathbb{R}^W(1<l<L),\right. \notag\\
& \left.\sum_{l=1}^L\left(\left\|A^{(l)}\right\|_0+\left\|b^{(l)}\right\|_0\right) \leq S, \max_l  \left(\left\|A^{(l)}\right\|_{\infty}\bigvee\left\|b^{(l)}\right\|_{\infty}\right) \leq B\right\} ,\label{def:sparse dnn}
\end{align}
as the class of sparse networks with ReLU activation $\sigma(x)=\max\{x,0\}$, where $\circ$ denotes the composition of functions, $\| A\|_0$ denotes the number of non-zero elements of the matrix $A$, and $\|A\|_\infty$ denotes the maximum of the absolute values of the elements in matrix $A$. {In our paper, we also define $\infty$-norm for function $f(\cdot)$ on the compact domain $\mathcal{X}$, $\|f\|_\infty=\sup_{\mf x\in\mathcal{X}}|f(\mf x)|$.} The sparse network \eqref{def:sparse dnn} has been widely investigated, by for exmaple \cite{Suzuki2018AdaptivityOD}, \cite{schmidt2020nonparametric}, and \cite{padilla2022quantile}. Based on $\mathcal{I}(L,W,S,B)$, our ConquerNet $\hat{f}_h$ is obtained from the ReLU networks by minimizing the convolution-type smoothed quantile loss for $\tau\in(0,1)$, i.e.,
\begin{equation}\label{eq:estimator besov}
  \begin{aligned}
    &\hat{f}_h:=\mathop{\arg\min}\limits_{f\in\mathcal{I}(L,W,S,B),\|f\|_\infty\leq F }\sum_{i=1}^n \ell_h\left(y_i-f(\mf x_i)\right),\\
    &\ell_h(u):=\int_{-\infty}^{\infty} \rho_\tau(v) K_h(v-u) \mathrm{d} v,
  \end{aligned}
\end{equation}
where $K_h(x)=K(x/h)/h$ with bandwidth $h>0$ and $F>0$ is a sufficiently large constant providing technical convenience. The kernel function $K(u)$ is required to be a bounded and nonnegative density function such that $\int uK(u)\mr du=0$, $\int K(u)\mr du=1$, and $\int u^2 K(u)\mr du=\sigma_K^2$ where $\sigma_K^2>0$ is a constant. $K(u)$ can be chosen from commonly used kernel functions including: uniform kernel $K(u)=(1 / 2) \mathbf{1}(|u| \leq 1)$; Gaussian kernel $K(u)=(2 \pi)^{-1 / 2} e^{-u^2 / 2}$; Epanechnikov kernel $K(u)=(3 / 4)\left(1-u^2\right) \mathbf{1}(|u| \leq 1)$, etc. 

For a given neural network  $f\in\mathcal{I}(L,W,S,B),\|f\|_{\infty}\leq F$, we define the $\|\cdot\|_\infty$-projection of $f_\tau^*$ onto $\mathcal{I}(L,W,S,B)$ as
$$
f_n:= \mathop{\arg\min}\limits_{f\in\mathcal{I}(L,W,S,B),\|f\|_\infty\leq F }\|f-f_\tau^*\|_\infty,
$$
where $f_\tau^*$ is assumed to be a function that belongs to Besov spaces defined below. {Note that the architectural parameters $(L,W,S,B)$ are usually chosen as functions of the sample size $n$ (see Theorem~\ref{thm:minimax rate}), and therefore the projection $f_n$ inherits this dependence through the network class $\mathcal{I}(L,W,S,B)$.}

\begin{definition}[Besov space]
\label{def:besov}
For a function $f \in L^p(\mathcal{X})$ and $p \in(0, \infty]$, denote the $r$-modulus of continuity as
$$
w_{r, p}(f, t)=\sup _{\|u\|_2 \leq t}\left\|R_u^r(f)\right\|_p
$$
where
$$
R_u^r(f)= \begin{cases}\sum_{j=0}^r \frac{r!}{j!(r-j)!}(-1)^{r-j} f(x+u j) &, \text { if } x \in \mathcal{X}, x+r u \in \mathcal{X} \\ 0 &, \text { otherwise.}\end{cases}
$$
For $q\in(0,\infty]$ and $\alpha>0, r=\lfloor\alpha\rfloor+1$, we define the Besov space $B_{p, q}^\alpha(\mathcal{X})$ as
$$
B_{p, q}^\alpha(\mathcal{X})=\left\{f \in L^p(\mathcal{X}):\|f\|_{B_{p, q}^\alpha(\mathcal{X})}<\infty\right\}
$$
where $\|f\|_{B_{p, q}^\alpha(\mathcal{X})}=\|f\|_p+|f|_{B_{p, q}^\alpha(\mathcal{X})}$ and
$$
|f|_{B_{p, q}^\alpha(\mathcal{X})}= \begin{cases}\left(\int_0^{\infty}\left(t^{-\alpha} w_{r, p}(f, t)\right)^q t^{-1} \mathrm{d} t\right)^{\frac{1}{q}} & \text { if } q<\infty, \\ \sup _{t>0} t^{-\alpha} w_{r, p}(f, t) & \text { if } q=\infty.\end{cases}
$$
\end{definition}

Throughout this paper, we write $\lceil x\rceil$ for the smallest integer greater than or equal to $x$. For any two positive real sequences $a_n$ and $b_n$, we write $a_n \asymp b_n$ if there exist constants $0<c<C<\infty$ such that $c \leq \liminf _{n \rightarrow \infty} a_n / b_n \leq$ $\limsup _{n \rightarrow \infty} a_n / b_n \leq C$. We write $a_n \lesssim$ $b_n\left(a_n \gtrsim b_n\right)$ if there exists constant $C>0$ such that $a_n \leq C b_n\left(C a_n \geq\right.$ $\left.b_n\right)$ for all $n$. 

 We introduce the performance metric in risk and empirical loss norms. For bounded functions $f$ and $g$, we define 
$$
\Delta_n^2(f,g):=\frac{1}{n}\sum_{i=1}^n D^2(f(\mf x_i)-g(\mf x_i)),
$$
where $D^2(t)=\min \left\{|t|, t^2\right\}$. Furthermore, we define $\Delta^2(f, g):=\mb{E}\left(D^2(f(X)-g(X))\right)$, $\Delta(f, g):=\sqrt{\Delta^2(f, g)}$, 
$\|f-g\|_{\ell_2}:=\sqrt{\mathbb{E}\left((f(X)-g(X))^2\right)}$, and
$$
\|f-g\|_n^2:=\frac{1}{n} \sum_{i=1}^n\left(f\left(\mf x_i\right)-g\left(\mf x_i\right)\right)^2.
$$

\section{Theory}
\label{sec:theory}
In this section, we provide statistical guarantees for our ConquerNet. 
Firstly, we evaluate how well our methodology can estimate the functions in Besov spaces. 
Our main result in Theorem~\ref{thm:minimax rate} shows that our ConquerNet achieves the minimax error rate with only an additional logarithmic factor.
Secondly, we develop a general risk bound on our estimator with respect to an arbitrary architecture of the neural networks, assuming the true quantile function $f^*_\tau(\cdot)$ in a space more general than the Besov spaces, which imposes no smoothness conditions. {This general risk bound accommodates a broad class of network architectures and modeling objectives, 
indicating that our theoretical development extends beyond the minimax analysis and offers an extensible foundation for future methodological advances.}

\subsection{Minimax rate}

In this subsection, we derive the convergence rate for our ConquerNet when the quantile function belongs to Besov spaces. The Besov space is a very general function class which plays an important role in fields like statistical learning \citep{donoho1998minimax,gine2021mathematical,padilla2022quantile} and approximation analysis \citep{temlyakov1993approximation,Suzuki2018AdaptivityOD}.
As defined in Definition~\ref{def:besov}, Besov spaces unify and extend many classical smoothness spaces. In particular, the Sobolev space $W^{\alpha,p}(\mathcal{X})$ coincides with the Besov space $B^{\alpha}_{p,p}(\mathcal X)$ and the H\"older space $C^\alpha(\mathcal{X})$ corresponds to $B^{\alpha}_{\infty,\infty}(\mathcal{X})$. Thus, Besov space provides a more general framework that strictly contains both Sobolev and H\"older spaces that are commonly assumed for smoothness conditions in statistical guarantee analysis \citep{farrell2021deep,montanelli2021deep,schmidt2020nonparametric}. 

Based on the function space and network class specification, we impose the following assumptions on the data generation settings.
\paragraph{Assumption 1.}We assume that $\left\{\left(\mf x_i, y_i\right)\right\}_{i=1}^n$ are i.i.d. samples from $(X, Y)$. Write $F_{y_i \mid \mf x_i}$ is cumulative distribution function of $y_i$ conditioning on $\mf x_i$ for $i=1, \ldots, n$.

\paragraph{Assumption 2.} There exists a constant $\kappa>0$ such that for $\delta \in \mathbb{R}$ satisfying $|\delta| \leq \kappa$ we have that, a.s.,
\begin{equation}
\left|F_{y_i \mid \mf x_i}\left(f_\tau^*\left(\mf x_i\right)+\delta\right)-F_{y_i \mid \mf  x_i}\left(f_\tau^*\left(\mf x_i\right)\right)\right| \geq \underline{p}\left|\delta\right|\label{eq:assump1 lip lower}    
\end{equation}
for some constant $\underline{p}>0$. Denote $p_{y_i \mid \mf x_i}(\cdot)$ as the probability density function of $y_i$ conditioning on $\mf x_i$ uniformly for all $i$. We also require that $p_{y_i|\mf x_i}(\cdot)$ is continuously differentiable and the derivative $p^\prime_{y_i\mid \mf x_i}(\cdot)$ satisfies almost surely that 
\begin{equation}
|p^\prime_{y_i \mid \mf x_i}(f_\tau^*(\mf x_i)+\delta)-p^\prime_{y_i \mid \mf x_i}(f_\tau^*(\mf x_i))| \leq l_0|\delta|,  \label{eq:assump1 lip upper}
\end{equation}
for some constant $l_0>0$ uniformly for all $i$.

\paragraph{Assumption 3.}  {Assume that $X$ has a bounded probability density function $g_X(\cdot)\leq c_2,c_2>0$ with support in $[-H,H]^d$.}

\paragraph{Assumption 4.} The quantile function satifies $f_\tau^* \in B_{p, q}^s\left([-H,H]^d\right),\left\|f_\tau^*\right\|_{\infty} \leq F$, where for $0<p, q \leq \infty$, and $0<s<\infty$ we have $s \geq d / p$. Furthermore, there exists $m \in \mathbb{N}$ such that $0<s<\min \{m, m-1+1 / p\}$. Here, $B_{p, q}^s\left([-H,H]^d\right)$ is a Besov space in $[-H,H]^d$ as in Definition \ref{def:besov}.

{\begin{remark}
Assumption 2 ensures the density around the target quantile is bounded away from zero.
When the density around the target quantile is very small, the quantile becomes ill-conditioned, leading to inflated variance and unstable optimization.
\begin{enumerate}
    \item \emph{Inflated variance.}
    Standard quantile regression theory (e.g. \cite{koenker2006quantile}) shows that the 
    asymptotic variance involves an inverse-density factor 
    $1/p_{Y|X}(q_\tau(x))$. Hence, near-zero densities amplify the 
    estimation error for any quantile estimator. In our neural network 
    analysis, the constant appearing in \eqref{eq:f-f_n delta} of Lemma~\ref{lem:f-f_n} which bounds 
    $\Delta^2(f_n,f)$, is proportional to the lower density bound 
    $1/\underline{p}$ where $\underline{p}$ is imposed in Assumption~2. This reflects the same 
    phenomenon: extremely small density around the target quantile leads to 
    an unfavorable constant in the convergence rate.
    \item \emph{Optimization instability.}
    When the density is very small, the derivative of the quantile loss is 
    nearly constant, resulting in weak curvature and slow convergence. 
    Smoothing via the convolution kernel can mitigate this effect to some 
    extent, but it cannot fully eliminate the ill-posedness caused by a 
    vanishing density.
\end{enumerate}
\end{remark}}

Condition \eqref{eq:assump1 lip lower} in Assumption 2 ensures local identifiability, which is a necessary condition that the quantile is estimatable.
Similar conditions are widely imposed in quantile regression literature \citep{pollard1991asymptotics, Chernozhukov2011quantileEP,padilla2022quantile}. 
Meanwhile, condition \eqref{eq:assump1 lip upper} requires the conditional density $p_{Y|X}(t)$ to be smooth in a neighborhood of the quantile. Such Lipschitz-type conditions are common in quantile regression \citep{koenker2006quantile,zhou2010Non,he2023smoothed}.

Assumptions 3 and 4 are mild and commonly assumed in both quantile regression \citep{Wu2017quantile,he2023smoothed} and deep learning studies \citep{padilla2022quantile,Suzuki2018AdaptivityOD}. 
{Assumption 3 now only requires that $g_X(\cdot)$ is bounded on $[-H,H]^d$. 
The previous global lower bound $c_1$, as assumed in \cite{padilla2022quantile}, 
is unnecessary for our results because regions with 
$g_X(\mf x)=0$ do not contribute to the $\ell_2$ norm. This simplification does not affect identifiability or our main results, as the required local density condition is already guaranteed by Assumption~2.}
It is worth pointing out that Assumption 4 requires that the quantile function $f_\tau^*$ belongs to a Besov space, which, as discussed, provides a more general smoothness framework than commonly assumed Sobolev or Hölder spaces, and is flexible enough to allow even discontinuous functions.

\begin{theorem}
    \label{thm:minimax rate}
Suppose that Assumptions 1-4 hold. Given $N \asymp n^{\frac{d}{2 s+d}}$, {let $\epsilon\asymp N^{-s / d}+\log^{-1} N$} with $v= sp/d-1$, suppose $h^2 \lesssim n^{-\frac{s}{2s+d}}$ and the class $\mathcal{I}(L, W, S, B)$ satisfies
\begin{align}
& L=3+2\left\lceil\log _2\left(\frac{3^{\max \{d, m\}}}{\epsilon c_{d, m}}\right)+5\right\rceil\left\lceil\log _2 \max \{d, m\}\right\rceil, \notag\\
& W=W_0 N ,S=(L-1) W_0^2 N+N, B=O\left(N^{1/v+1/d}\right),\label{eq:layer condition}
\end{align}
for a constant $c_{d, m}>0$ that depends on $d$ and $m$.
{Then there exists a constant $C>0$ such that
$$
\mr P\left(\max \left\{\left\|\hat{f}_h-f_\tau^*\right\|_{\ell_2}^2,\left\|\hat{f}_h-f_\tau^*\right\|_n^2\right\} \leq C(\log n)^3 \max\{\delta n^{-1},n^{\frac{-2s}{2s+d}}\} \right)\geq 1-e^{-\delta}\log n.
$$
}
\end{theorem}
{\begin{remark}
The expression for $L$ in \eqref{eq:layer condition} is not meant to prescribe a 
single fixed value. Since the quantity $\epsilon$ may be chosen within 
the asymptotic range $\epsilon \asymp N^{-s/d} + (\log N)^{-1}$ with 
$N \asymp n^{d/(2s+d)}$, the resulting depth $L$ can vary accordingly. 
Thus, \eqref{eq:layer condition} should be interpreted as describing an 
admissible range of depths that ensures the minimax rate, rather than requiring 
$L$ to take a specific exact value.
\end{remark}}
{\begin{remark}
   The constant $C$ in the upper bound depends on the regularity assumptions on the conditional density. In particular, $C$ is proportional to $c_2$ in Assumption 3 and inversely proportional to the lower bound $\underline{p}$ in Assumption 2.
\end{remark}}
Note that the rate $n^{-\frac{2s}{2s+d}}$ is known to be minimax optimal for function estimation in Besov spaces \citep{kerkyacharian1992density,donoho1998minimax,Suzuki2018AdaptivityOD}. {Theorem \ref{thm:minimax rate} explicitly yields the minimax rate $O_p(n^{-\frac{2s}{2s+d}}\log^3 n)$ when $\delta\asymp\log n$ in Besov space up to constants and logarithmic factors.}
Our estimator attains this minimax rate up to a logarithmic factor $\log^3 n$, implying that no estimator can substantially improve upon the performance of our ConquerNet beyond this logarithmic term. 
The minimax rate also clarifies the role of the bandwidth parameter $h$. Specifically, the condition $h^2\lesssim n^{-\frac{s}{2s+d}}$ ensures that the smoothing bias remains negligible, so that the estimator still targets the true quantile function. Conversely, when $h$ is extremely small, the convolution-type smoothed loss function $\ell_h(y_i-f(\mf x_i))$ effectively reduces to the original quantile loss $\rho_\tau(y_i-f(\mf x_i))$; see Figure~\ref{fig:loss diff}. 
Smaller $h$ leads to a sharper minimum and a smaller smoothing region, which makes the gradients change dramatically in this region. The excessive sharpness induced by very small $h$ will increase sensitivity to parameter changes, which reduces optimization stability and leads to poor generalization.
\begin{figure}[htpb]
    \centering
    \includegraphics[width=0.8\linewidth]{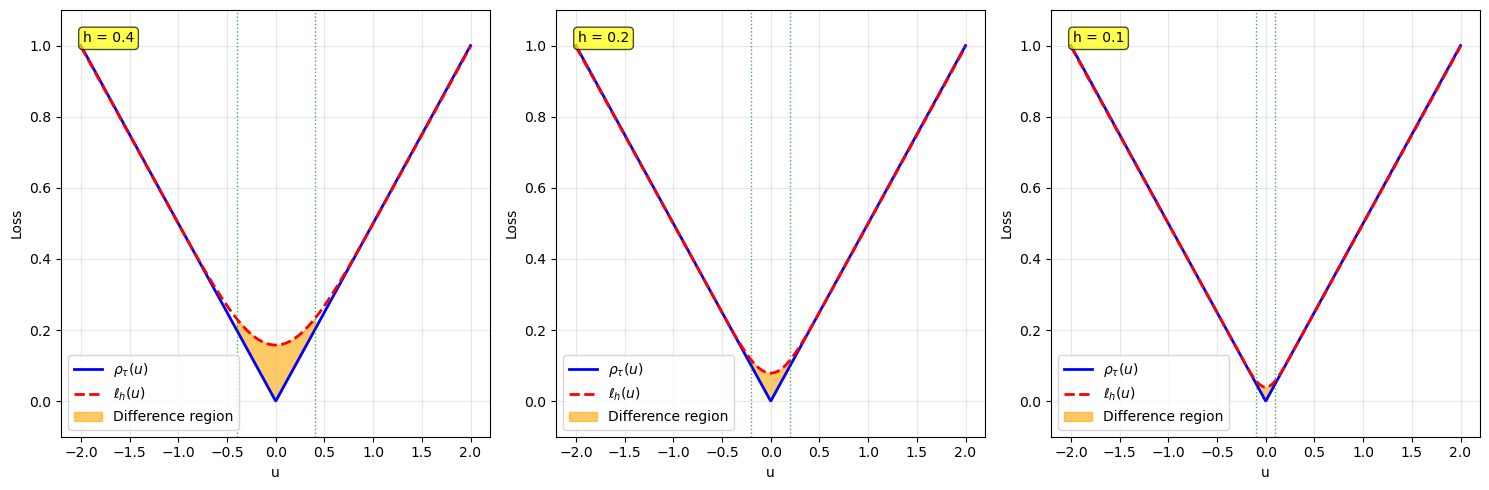}
    \caption{Relationship between loss functions $\rho_\tau(u)$ (solid lines) and $\ell_h(u)$ (dashed lines) for $h=0.4,0.2,0.1$, with $\tau=0.5$. Smaller $h$ yields smaller difference region between $\rho_\tau(u)$ and $\ell_h(u)$.}
    \label{fig:loss diff}
\end{figure}

\subsection{General upper bound}
While the minimax analysis in the previous subsection relies on structural assumptions on the neural networks and target function, it is also of interest to study performance guarantees in a more general setting without such restrictions. To this end, we provide a general upper bound for our ConquerNet estimator without imposing constraints on the width, depth, magnitude of parameters, and sparsity of the network class, or on the smoothness assumptions (such as belonging to a Besov space) on the target quantile function $f_\tau^*$.

Consider a general neural network function with ReLU activation $f\in \mathcal F(P,U,L)$ denoted as
\begin{equation}
   \mathcal F(P,U,L):=\{ f: \mathbb{R}^{p_0} \rightarrow \mathbb{R}^{p_{L+1}}, \quad \mathbf{x} \mapsto f(\mathbf{x})=W_L \sigma_{\mathbf{v}_L}\circ W_{L-1} \sigma_{\mathbf{v}_{L-1}}\circ \cdots\circ W_1 \sigma_{\mathbf{v}_1}\circ W_0 \mathbf{x} \},\label{eq:def MLP}
\end{equation}
where $W_i$ is a $p_{i+1} \times p_i$ weight matrix, $\sigma_{\mf v_i}$ is a shifted activation function with shifting vector $\mathbf{v}_i =(v_{i,1},...,v_{i,p_i})^\top\in \mathbb{R}^{p_i}$, i.e.,
$$
\sigma_{\mf v_i}\left(\begin{array}{c}
a_1 \\
\vdots \\
a_{p_i}
\end{array}\right)=\left(\begin{array}{c}
\sigma\left(a_1-v_{i,1}\right) \\
\vdots \\
\sigma\left(a_{p_i}-v_{i,p_i}\right)
\end{array}\right) .
$$

Notice that the number of parameters is $P=\sum_{l=0}^L(p_{l+1}p_l+p_{l+1})$ with $p_0=d,p_{L+1}=1$, the number of nodes is $U=\sum_{l=1}^Lp_l$, and the number of layers is $L$. For $f\in \mathcal{F}(P,U,L)$, we redefine its ConquerNet estimator
\begin{equation}
    \hat{f}_h:=\mathop{\arg\min}\limits_{f\in\mathcal{F}(P,U,L),\|f\|_\infty\leq F }\sum_{i=1}^n \ell_h\left(y_i-f(\mf x_i)\right),\label{eq:estimator general}
\end{equation}
and the approximation error is defined as
\begin{equation}
    err_1:=\mathbb{E}\left[\frac{1}{n} \sum_{i=1}^n \ell_h\left(y_i-f_n\left(\mf x_i\right)\right)-\frac{1}{n} \sum_{i=1}^n \ell_h\left(y_i-f_\tau^*\left(\mf x_i\right)\right)\right]\label{eq:approx error}
\end{equation}
where
\begin{equation}
f_n := \underset{f \in \mathcal{F}(P, U, L),\|f\|_{\infty} \leq F}{\arg \min } \mathbb{E}\left[\sum_{i=1}^n \ell_h\left(y_i-f\left(\mf x_i\right)\right)\right].\label{eq:f_n} 
\end{equation}
\begin{theorem}
\label{thm:general upper bound}
Suppose that Assumptions 1-2 hold and $n \geq C L P \log (U)$ for a sufficiently large $C>0$. Then with $c_1>0$ a constant, $\hat{f}_h$ defined in \eqref{eq:estimator general} satisfies
\begin{equation}
\mathbb{E}\left[\Delta_n^2\left(\hat{f}_h,f_\tau^*\right) \right]\leq c_1\left[ F\left(\frac{L P \log U \cdot \log n}{n}\right)^{1 / 2}+h^2\mb E\|\hat{f}_h- f_\tau^*\|_{n,1} +err_1\right],\label{eq:general bound delta}
\end{equation}
where $\|f-g\|_{n,1}=\frac{1}{n}\sum_{i=1}^n|f(\mf x_i)-g(\mf x_i)|$.
Furthermore, if $h^2=o\left( \sqrt{\mb E\|\hat{f}_h-f_\tau^* \|_n^2}\right)$, it also holds that
\begin{equation}
    \mathbb{E}\left[\left\|\hat{f}_h-f_\tau^*\right\|_n^2 \right] \leq 2c_1\max\{1,F\}\left[ F\left(\frac{L P \log U \cdot \log n}{n}\right)^{1 / 2}+err_1\right].\label{eq:general bound l2}
\end{equation}
\end{theorem}
{\begin{remark}
    Our general risk bound in Theorem 3.2 is quite flexible and does not rely on a very rigid architecture, but rather on general capacity measures (e.g., network class size, sparsity, norm bounds) and approximation properties. Therefore, in principle, it can accommodate more complex architectures, including residual (skip-connected) MLPs, as long as we can control the relevant complexity measures. However, to the best of our knowledge, there is no existing minimax-rate analysis for quantile regression specifically with residual-based (ResNet/skip-connection) networks. The most closely related work is the minimax-rate analysis for deep ReLU networks under quantile loss by \cite{padilla2022quantile}, which considers plain feedforward (non-skip) ReLU networks. Given the lack of minimax theory for residual architectures in quantile regression, we focused on the simpler network class to establish minimax optimality for our proposed ConquerNet method, and leave for the possible extension to MLPs with skip connections as a promising future work.
\end{remark}}

Theorem \ref{thm:general upper bound} drops out Assumptions 3-4 in Theorem \ref{thm:minimax rate}
and yields a general error bound that depends on parameters $P,U,L$, the approximation error $err_1$, and the sample size $n$. As long as $h^2=o(\sqrt{\mathbb{E}\|\hat{f}_h-f_\tau^*\|_n^2})$, the risk bound in \eqref{eq:general bound l2} remains essentially unaffected, implying that taking $h$ sufficiently small does not deteriorate the estimator’s performance. However, from an optimization perspective, a very small $h$ causes the smoothed loss $\ell_h(\cdot)$ to behave almost identically to the original quantile loss $\rho_\tau(\cdot)$; see Figure \ref{fig:loss diff}, which raises sharp minima concerns discussed in Section~\ref{sec:introduction}.

\section{Empirical study}
\label{sec:empirical study}
In this section, we empirically evaluate the proposed ConquerNet through comprehensive simulation studies and a real data application.
We use the convolution-type smoothed quantile losses defined in \eqref{eq:estimator besov} by three different kernel functions: Gaussian, uniform, and Epanechnikov kernels. We compare them against the baseline quantile ReLU network in \citet{padilla2022quantile} under different sample sizes and quantile levels. The simulation results show that our ConquerNet perform better in MSE with less training time, especially for extreme quantile levels. We also discuss the impact of different choices of the bandwidth $h$ and propose a data-driven rule for bandwidth selection, which shows the stability of our ConquerNet. We apply the ConquerNet to the BMI (body mass index) dataset and get better performance compared to the baseline model. In addition, we explore the loss landscape plots of the two networks, which provide an intuitive way to understand the advantage of the ConquerNet during optimization. Due to the page limit, we only present part of our experimental results in this section. See Appendix~\ref{sec:additional experiments} for extra experiments, including the tables of MAE (mean absolute error), joint estimation of multiple quantile levels under non-crossing constraints, residual-based networks, and the plot of MSEs by sample sizes.

We have made efforts to ensure reproducibility of our results. The experimental setup, including simulation designs, model configurations, and hyperparameter choices, is described in this section and Appendix \ref{sec:additional experiments}. Random seeds are set to ensure reproducibility of the experiment. Upon publication, we will release the full implementation, including training scripts and simulation setups, in a GitHub repository. An anonymous version of the GitHub repository is available during the review, see \url{https://anonymous.4open.science/r/conquernn-F625/}.

\subsection{Scenario settings}
\label{sec:scenario settings}
Under the comparable number of parameters, we consider two networks with different shapes. One has 5 hidden layers of 70 nodes each,  denoted by Model A, and the other one has 10 hidden layers of 50 nodes each, denoted by Model B. We consider both smooth and piecewise continuous quantile functions with heavy-tailed noises. Specifically, the data for the simulation are generated by the following mechanism,
$$
y_i = g(\mf{x}_i)+\varepsilon_i, \quad i=1,\ldots,n,
$$
where $g(\cdot):[0,1]^d\rightarrow \mathbb{R}$, $\{\mf{x}_i\}_{i=1}^n$ are independently sampled from uniform distribution on $[0,1]^d$ and $d$ is the dimension of $\mf{x}_i$. We consider the following 3 scenarios of data generation:

\textbf{Scenario 1 (S1).} We set $d=2$, $\mf{z}\in[0,1]^2$, $g(\mf{z}) = \cos(2\pi z_1^2) + \sin\left(\sqrt{z_1^2 + 2z_2} + 2\right)$, and $\varepsilon_i = \|\mf{x}_i-(1,0)^\top\|t_i/2$, where $t_i\overset{\text{iid}}{\sim}t(2)$ for $i=1,\ldots,n$ and $t(2)$ is the t-distribution with 2 degrees of freedom.

\textbf{Scenario 2 (S2).} We set $d=5$, $\mf{z}\in[0,1]^5$, $g(\mf{z}) = \sqrt{z_1 + 2z_2 + z_3 + 2z_4 + z_5}$, and $\varepsilon_i = \sqrt{\mf{x}_i^\top \mf{\eta}}v_i$ for $i=1,\ldots,n$, where ${\eta}=(1/2,0,1/2,0,1/2)^\top$ and $v_i\overset{\text{iid}}{\sim}t(3)$, $t(3)$ is the t-distribution with 3 degrees of freedom.

\textbf{Scenario 3 (S3).} We set $d=5$, $z\in[0,1]^5$, $g(\mf{z}) = g_2\circ g_1(\mf{z})$, where $g_1(\mf{z}) = (z_1 + 3z_2, \cos(2\pi(z_3 + z_4)), z_2 + \sqrt{z_3} + 2z_5)^\top$ and
$$
    g_2(\mf{z})= 
        \begin{cases} 
        z_1 + \sqrt{z_2^2 + z_3} & \text{if } z_2 < 0, \\
        \sqrt{z_1 + z_2} + 0.5z_3 & \text{otherwise},
        \end{cases}
$$
with $\varepsilon_i\overset{\text{iid}}{\sim} \text{Laplace}(0,2)$ for $i=1,\ldots,n$.

\subsection{Experiment details}
Our training data set is denoted by $\{(\mf{x}_i,y_i)\}_{i=1}^n$ and the test data set is denoted by $\{\tilde{\mf{x}}_i\}_{i=1}^n$.
For a fixed quantile $\tau\in(0,1)$, we train the baseline network and ConquerNet from \eqref{eq:estimator besov}, and get the trained quantile estimate $\hat{f}_h$. By the data generating mechanism described above, we can calculate the true quantile $f_{\tau}^*(\tilde{\mf{x}}_i) = g(\tilde{\mf{x}}_i)+F_{\tilde{\varepsilon}_i|\tilde{\mf{x}}_i}^{-1}(\tau)$. Then the MSE of quantile estimator is obtained by $\sum_{i=1}^T (f_{\tau}^*(\tilde{\mf{x}}_i)-\hat{f}_h(\tilde{\mf{x}}_i))^2/T$, which is evaluated by using $\{\tilde{\mf x}_i\}_{i=1}^{T}$ of size $T=10000$. We set training sample sizes $n\in\{1000,5000,10000\}$ and quantile levels $\tau \in \{0.05,0.25,0.5,0.75,0.95\}$. 
For each experiment setting, we run the experiment 50 times independently and get the MSE results 50 times. 
Table~\ref{tab:mse} shows the averaged MSE results over 50 trials under different experiment settings. 
{The bold fonts represent that the results of the ConquerNet are better than those of the baseline models. In addition, to make the representation clearer, we box the baseline result if it outperforms our ConquerNet.}

\begin{table}[htpb]
\centering
\caption{Mean squared error (MSE) performances for scenario 1-3, model A and B  under different sample sizes, quantile levels, and smoothing kernels. The MSEs are averaged over 50 independent trials.}
\label{tab:mse}
\begin{adjustbox}{width=\textwidth,center}
\begin{tabular}{@{}ccclllllllllllllll@{}}
\toprule
\multicolumn{2}{c}{\multirow{2}{*}{}}          & \multirow{2}{*}{Method} & \multicolumn{5}{c}{$n$=1000}                                                                                                                                                & \multicolumn{5}{c}{$n$=5000}                                                                                                                                                & \multicolumn{5}{c}{$n$=10000}                                                                                                                                          \\ \cmidrule(l){4-18} 
\multicolumn{2}{c}{}                           &                         & \multicolumn{1}{c}{$\tau$=0.05} & \multicolumn{1}{c}{$\tau$=0.25} & \multicolumn{1}{c}{$\tau$=0.5} & \multicolumn{1}{c}{$\tau$=0.75} & \multicolumn{1}{c|}{$\tau$=0.95}     & \multicolumn{1}{c}{$\tau$=0.05} & \multicolumn{1}{c}{$\tau$=0.25} & \multicolumn{1}{c}{$\tau$=0.5} & \multicolumn{1}{c}{$\tau$=0.75} & \multicolumn{1}{c|}{$\tau$=0.95}     & \multicolumn{1}{c}{$\tau$=0.05} & \multicolumn{1}{c}{$\tau$=0.25} & \multicolumn{1}{c}{$\tau$=0.5} & \multicolumn{1}{c}{$\tau$=0.75} & \multicolumn{1}{c}{$\tau$=0.95} \\ \midrule
\multirow{8}{*}{S1} & \multirow{4}{*}{Model A} & Baseline                & \fbox{0.3820}                          & 0.0402                          & 0.0278                         & 0.0374                          & \multicolumn{1}{l|}{0.3784}          & \fbox{0.0996}                          & 0.0120                          & 0.0086                         & 0.0108                          & \multicolumn{1}{l|}{\fbox{0.0939}}          & \fbox{0.0618}                          & 0.0067                          & 0.0047                         & 0.0063                          & 0.1366                          \\ \cmidrule(l){3-18} 
                    &                          & Gaussian                & 0.4354                          & \textbf{0.0383}                 & \textbf{0.0224}                & 0.0382                          & \multicolumn{1}{l|}{0.5035}          & 0.1124                          & \textbf{0.0087}                 & \textbf{0.0055}                & \textbf{0.0097}                 & \multicolumn{1}{l|}{0.1081}          & 0.0678                          & \textbf{0.0064}                 & \textbf{0.0034}                & \textbf{0.0055}                 & \textbf{0.0625}                 \\
                    &                          & Uniform                 & 0.4448                          & \textbf{0.0385}                 & \textbf{0.0224}                & \textbf{0.0328}                 & \multicolumn{1}{l|}{\textbf{0.3721}} & 0.1079                          & \textbf{0.0087}                 & \textbf{0.0059}                & \textbf{0.0104}                 & \multicolumn{1}{l|}{0.1312}          & 0.0789                          & \textbf{0.0058}                 & \textbf{0.0036}                & \textbf{0.0057}                 & \textbf{0.0543}                 \\
                    &                          & Epanechnikov            & 0.6733                          & \textbf{0.0390}                 & \textbf{0.0222}                & \textbf{0.0373}                 & \multicolumn{1}{l|}{0.5913}          & 0.1251                          & \textbf{0.0093}                 & \textbf{0.0055}                & \textbf{0.0095}                 & \multicolumn{1}{l|}{0.1169}          & 0.0653                          & \textbf{0.0061}                 & \textbf{0.0037}                & \textbf{0.0053}                 & \textbf{0.0665}                 \\ \cmidrule(l){3-18} 
                    & \multirow{4}{*}{Model B} & Baseline                & \fbox{0.3842}                          & \fbox{0.0527}                          & \fbox{0.0319}                         & \fbox{0.0475}                          & \multicolumn{1}{l|}{\fbox{0.4222}}          & 0.1143                          & 0.0149                          & 0.0107                         & 0.0151                          & \multicolumn{1}{l|}{0.1277}          & 0.1691                          & 0.0099                          & 0.0066                         & 0.0082                          & 0.3882                          \\ \cmidrule(l){3-18} 
                    &                          & Gaussian                & 0.4158                          & 0.0665                          & 0.0383                         & 0.0633                          & \multicolumn{1}{l|}{0.5202}          & 0.1172                          & \textbf{0.0144}                 & \textbf{0.0097}                & \textbf{0.0142}                 & \multicolumn{1}{l|}{\textbf{0.1066}} & \textbf{0.1062}                 & \textbf{0.0095}                 & \textbf{0.0055}                & 0.0086                          & \textbf{0.0726}                 \\
                    &                          & Uniform                 & 0.5652                          & 0.0612                          & 0.0378                         & 0.0614                          & \multicolumn{1}{l|}{0.5685}          & \textbf{0.1033}                 & \textbf{0.0145}                 & \textbf{0.0091}                & 0.0154                          & \multicolumn{1}{l|}{\textbf{0.1252}} & \textbf{0.0974}                 & \textbf{0.0093}                 & \textbf{0.0055}                & \textbf{0.0080}                 & \textbf{0.0736}                 \\
                    &                          & Epanechnikov            & 0.4275                          & 0.0582                          & 0.0383                         & 0.0626                          & \multicolumn{1}{l|}{0.6050}          & \textbf{0.1115}                 & \textbf{0.0146}                 & \textbf{0.0094}                & \textbf{0.0145}                 & \multicolumn{1}{l|}{\textbf{0.1162}} & \textbf{0.1294}                 & \textbf{0.0089}                 & \textbf{0.0057}                & 0.0092                          & \textbf{0.0663}                 \\ \midrule
\multirow{8}{*}{S2} & \multirow{4}{*}{Model A} & Baseline                & 0.8292                          & 0.0868                          & 0.0619                         & 0.0839                          & \multicolumn{1}{l|}{\fbox{0.7874}}          & 0.2752                          & \fbox{0.0275 }                         & 0.0222                         & 0.0308                          & \multicolumn{1}{l|}{0.2747}          & 0.1704                          & 0.0205                          & 0.0145                         & 0.0223                          & 0.1670                          \\ \cmidrule(l){3-18} 
                    &                          & Gaussian                & \textbf{0.7994}                 & \textbf{0.0711}                 & \textbf{0.0587}                & \textbf{0.0778}                 & \multicolumn{1}{l|}{1.1466}          & \textbf{0.2202}                 & 0.0276                          & \textbf{0.0169}                & \textbf{0.0273}                 & \multicolumn{1}{l|}{\textbf{0.2598}} & \textbf{0.1316}                 & \textbf{0.0176}                 & \textbf{0.0128}                & \textbf{0.0193}                 & \textbf{0.1537}                 \\
                    &                          & Uniform                 & 0.8892                          & \textbf{0.0721}                 & \textbf{0.0471}                & 0.0964                          & \multicolumn{1}{l|}{0.8566}          & \textbf{0.2308}                 & 0.0286                          & \textbf{0.0181}                & \textbf{0.0275}                 & \multicolumn{1}{l|}{\textbf{0.2586}} & \textbf{0.1457}                 & \textbf{0.0180}                 & \textbf{0.0128}                & \textbf{0.0191}                 & \textbf{0.1467}                 \\
                    &                          & Epanechnikov            & 0.9192                          & \textbf{0.0787}                 & \textbf{0.0522}                & 0.1048                          & \multicolumn{1}{l|}{0.9074}          & \textbf{0.2415}                 & 0.0284                          & \textbf{0.0177}                & \textbf{0.0265}                 & \multicolumn{1}{l|}{\textbf{0.2492}} & \textbf{0.1430}                 & \textbf{0.0162}                 & \textbf{0.0126}                & \textbf{0.0191}                 & \textbf{0.1388}                 \\ \cmidrule(l){3-18} 
                    & \multirow{4}{*}{Model B} & Baseline                & 0.4930                          & \fbox{0.0583}                          & \fbox{0.0493 }                        & 0.0840                          & \multicolumn{1}{l|}{\fbox{0.5898}}          & 0.1732                          & 0.0257                          & 0.0178                         & 0.0300                          & \multicolumn{1}{l|}{0.1966}          & 0.1323                          & 0.0202                          & 0.0129                         & 0.0220                          & 0.1358                          \\ \cmidrule(l){3-18} 
                    &                          & Gaussian                & 0.7367                          & 0.0639                          & 0.0500                         & \textbf{0.0759}                 & \multicolumn{1}{l|}{0.6273}          & 0.1800                          & \textbf{0.0246}                 & \textbf{0.0155}                & \textbf{0.0245}                 & \multicolumn{1}{l|}{0.2157}          & \textbf{0.1085}                 & \textbf{0.0160}                 & \textbf{0.0119}                & \textbf{0.0178}                 & \textbf{0.1144}                 \\
                    &                          & Uniform                 & \textbf{0.4515}                 & 0.0717                          & 0.0503                         & 0.0935                          & \multicolumn{1}{l|}{0.6034}          & \textbf{0.1706}                 & \textbf{0.0216}                 & \textbf{0.0176}                & \textbf{0.0241}                 & \multicolumn{1}{l|}{0.2239}          & \textbf{0.1172}                 & \textbf{0.0151}                 & \textbf{0.0118}                & \textbf{0.0175}                 & \textbf{0.1342}                 \\
                    &                          & Epanechnikov            & 0.5800                          & 0.0723                          & 0.0546                         & \textbf{0.0819}                 & \multicolumn{1}{l|}{0.7484}          & 0.2151                          & 0.0263                          & \textbf{0.0159}                & \textbf{0.0289}                 & \multicolumn{1}{l|}{\textbf{0.1802}} & \textbf{0.1038}                 & \textbf{0.0146}                 & \textbf{0.0107}                & \textbf{0.0173}                 & \textbf{0.1302}                 \\ \midrule
\multirow{8}{*}{S3} & \multirow{4}{*}{Model A} & Baseline                & 3.0766                          & 0.7564                          & 0.5350                         & 0.7407                          & \multicolumn{1}{l|}{\fbox{2.8969}}          & 1.2870                          & 0.3854                          & 0.2146                         & 0.3387                          & \multicolumn{1}{l|}{1.3495}          & 0.9232                          & 0.2420                          & 0.1459                         & 0.2413                          & 0.9960                          \\ \cmidrule(l){3-18} 
                    &                          & Gaussian                & \textbf{2.8841}                 & 0.7776                          & \textbf{0.4788}                & \textbf{0.7056}                 & \multicolumn{1}{l|}{3.4222}          & 1.3139                          & \textbf{0.3379}                 & \textbf{0.1993}                & \textbf{0.3269}                 & \multicolumn{1}{l|}{\textbf{1.1897}} & \textbf{0.8395}                 & \textbf{0.2040}                 & \textbf{0.1295}                & \textbf{0.2145}                 & \textbf{0.7477}                 \\
                    &                          & Uniform                 & 3.3730                          & 0.7735                          & \textbf{0.4876}                & 0.7832                          & \multicolumn{1}{l|}{3.3576}          & \textbf{1.1577}                 & \textbf{0.3449}                 & \textbf{0.1912}                & \textbf{0.3082}                 & \multicolumn{1}{l|}{\textbf{1.2392}} & \textbf{0.8790}                 & \textbf{0.2064}                 & \textbf{0.1300}                & \textbf{0.2146}                 & \textbf{0.7824}                 \\
                    &                          & Epanechnikov            & 3.4479                          & \textbf{0.7308}                 & \textbf{0.4679}                & 0.8153                          & \multicolumn{1}{l|}{3.3156}          & \textbf{1.1229}                 & \textbf{0.3615}                 & \textbf{0.1980}                & \textbf{0.3112}                 & \multicolumn{1}{l|}{\textbf{1.3312}} & \textbf{0.8230}                 & \textbf{0.2129}                 & \textbf{0.1349}                & \textbf{0.2206}                 & \textbf{0.7998}                 \\ \cmidrule(l){3-18} 
                    & \multirow{4}{*}{Model B} & Baseline                & 2.8166                          & \fbox{0.7558 }                         & 0.5175                         & 0.7665                          & \multicolumn{1}{l|}{2.2839}          & 1.0061                          & \fbox{0.3596}                          & \fbox{0.2196}                         & 0.3551                          & \multicolumn{1}{l|}{1.0990}          & 0.7786                          & \fbox{0.2306}                          & 0.1391                         & 0.2380                          & 0.7249                          \\ \cmidrule(l){3-18} 
                    &                          & Gaussian                & \textbf{2.3193}                 & 0.8405                          & \textbf{0.5142}                & \textbf{0.7543}                 & \multicolumn{1}{l|}{2.6520}          & 1.0602                          & 0.4263                          & 0.2304                         & \textbf{0.3525}                 & \multicolumn{1}{l|}{\textbf{1.0681}} & 0.8256                          & 0.2518                          & 0.1416                         & 0.2444                          & 0.7536                          \\
                    &                          & Uniform                 & \textbf{2.1946}                 & 0.8316                          & 0.5193                         & \textbf{0.7352}                 & \multicolumn{1}{l|}{2.5765}          & 1.2056                 & 0.4038                          & 0.2320                         & \textbf{0.3373}                 & \multicolumn{1}{l|}{\textbf{1.0528}} & \textbf{0.7167}                 & 0.2409                          & \textbf{0.1372}                & 0.2493                          & \textbf{0.6920}                 \\
                    &                          & Epanechnikov            & 2.9702                          & 0.9008                          & \textbf{0.4892}                & 0.8702                          & \multicolumn{1}{l|}{\textbf{2.1331}} & \textbf{1.0053}                 & 0.3750                          & 0.2297                         & \textbf{0.3473}                 & \multicolumn{1}{l|}{\textbf{1.0895}} & \textbf{0.7502}                 & 0.2454                          & \textbf{0.1390}                & \textbf{0.2315}                 & \textbf{0.6683}                 \\ \bottomrule
\end{tabular}
\end{adjustbox}
\end{table}

From Table \ref{tab:mse}, first, we can see a significant improvement in MSE as the sample size increases. A sample size of 1000 is inadequate for model training in our settings and is much smaller than the number of parameters of the networks. For sample size 10000, our ConquerNet outperform the baseline model in most scenarios and quantile levels, regardless of which kernel is used, especially for low ($\tau=0.05$) and high ($\tau=0.95$) quantiles. It is reasonable because the original quantile loss is not differentiable at the origin point, which leads to biased subgradients for parameter update. Second, Scenario 3 has discontinuity points in the function $g(\cdot)$, while $g(\cdot)$ is smooth in Scenario 1 and Scenario 2. The MSE results are consistent with Theorem \ref{thm:minimax rate} in the sense that MSE is smaller when the smoothness $s$ increases. Third, Model B has more hidden layers than Model A, while the number of parameters is close. Meanwhile, $d=2$ in Scenario 1, $d=5$ in Scenario 2 and 3. The MSE results show that when $d$ is small, shallow networks are more suitable. In contrast, for $d=5$, deep networks perform better, which confirms the layer condition for the minimax rate in \eqref{eq:layer condition} of Theorem \ref{thm:minimax rate} to some extent.

We also record the training time for each setting and for 50 trials. Due to the page limit, we present the result for $\tau=0.05$ in Figure \ref{fig:time_tau=0.05}. The rest results are shown in Figures \ref{fig:time_tau=0.25}-\ref{fig:time_tau=0.95} in Appendix \ref{sec:additional experiments}. We can conclude that our ConquerNet are generally faster than the baseline, with only 80\% of the training time consumed, and are stable enough.
\begin{figure}[htpb]
    \centering
\includegraphics[width=0.8\linewidth]{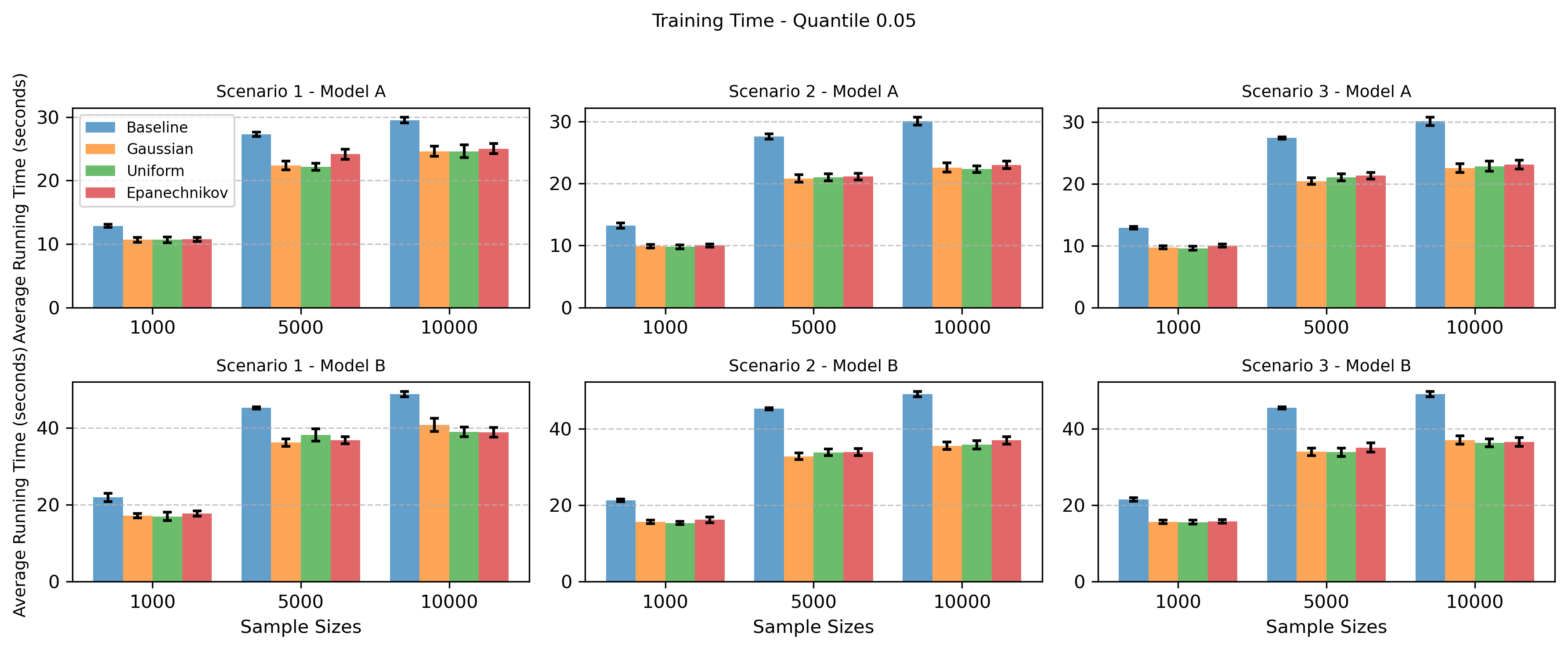}
    \caption{Bar chart with error bars with average training time over 50 trials under quantile level $\tau=0.05$. The error bars represent 95\% confidence intervals of training time for each setting.}
    \label{fig:time_tau=0.05}
\end{figure}

\subsection{Choice of bandwidth $h$}

The bandwidth $h$ should be chosen properly based on the theoretical results in Section \ref{sec:theory} and the experiments. As shown in Figure \ref{fig:loss diff}, the difference between smoothed loss $\ell_h(u)$ and quantile loss $\rho_{\tau}(u)$ grows larger with $h$ increases. 
Furthermore, as $h$ increases, the smoothed area becomes larger and the gradients become smaller, which affects the efficiency of SGD. 
On the other hand, too small $h$ reduces to the original quantile loss and increases the sharpness of the minima, which tends to result in poor generalization. Therefore, a proper choice of bandwidth is necessary. By Theorem \ref{thm:minimax rate}, we accept a bigger $h$ when the sample size $n$ is small and a smaller $h$ when $n$ is big. For example in Table~\ref{tab:mse}, we take $h=0.01/0.005/0.001$ for $n=1000/5000/10000$ in Scenario 2, Model A. We find that the performance remains outstanding with a wide range of bandwidth $h$ for $n=10000$, see Table~\ref{tab:mse with h} in Appendix~\ref{sec:additional experiments}, which shows the stability of the results for different bandwidth choices.

We also make it clear to the data-driven rules for bandwidth selection. In detail, we propose the K-fold cross-validation algorithm for the bandwidth selection. For a candidate list of bandwidths, $h\in\{0.001,0.005,0.01,0.05,0.1\}$ for example, we train the models on the training set and calculate the pinball loss on the validation set for K times. Select the bandwidth with the minimum mean validation loss. Simulation studies have been implemented for Scenario 2, and Table~\ref{tab:mse-single-nores-CV-S2-A} presents the MSE results for model A. The complete MSE results are shown in Table~\ref{tab:mse-single-nores-CV-S2} in Appendix~\ref{sec:additional experiments}. By the cross-validation algorithm, our ConquerNet still outperform the baseline models, especially for large sample sizes, which shows the stability of our method.

\begin{table}[htpb]
\centering
\caption{Mean squared error (MSE) performances for scenario 2, model A with 5-fold cross-validation under different sample sizes, quantile levels, and smoothing kernels. The MSEs are averaged over 50 independent trials.}
\label{tab:mse-single-nores-CV-S2-A}
\begin{adjustbox}{width=0.6\textwidth,center}
\begin{tabular}{@{}cccccccc@{}}
\toprule
\multirow{2}{*}{$n$}   & \multirow{2}{*}{Method} & \multicolumn{5}{c}{Scenario 2, Model A} \\ \cmidrule(l){3-7}
                       &                         & $\tau$=0.05     & $\tau$=0.25     & $\tau$=0.5      & $\tau$=0.75     & $\tau$=0.95     \\ \midrule
\multirow{4}{*}{1000}  & Baseline                & 0.9195          & \fbox{0.0714}          & 0.0586          & 0.0894          & 0.8167          \\ \cmidrule(l){2-7}
                       & Gaussian                & \textbf{0.7859} & 0.0763          & \textbf{0.0498} & \textbf{0.0890} & 0.8510          \\
                       & Uniform                 & 0.9437          & 0.0827          & \textbf{0.0569} & \textbf{0.0744} & \textbf{0.7308} \\
                       & Epanechnikov            & 1.0141          & 0.0759          & \textbf{0.0462} & 0.0919          & \textbf{0.6248} \\ \midrule
\multirow{4}{*}{5000}  & Baseline                & 0.3446          & 0.0258          & 0.0241          & 0.0335          & 0.3439          \\ \cmidrule(l){2-7}
                       & Gaussian                & \textbf{0.2358} & 0.0275          & \textbf{0.0180} & \textbf{0.0302} & \textbf{0.3064} \\
                       & Uniform                 & \textbf{0.2477} & 0.0310          & \textbf{0.0184} & \textbf{0.0268} & \textbf{0.2732} \\
                       & Epanechnikov            & \textbf{0.2556} & \textbf{0.0237} & \textbf{0.0169} & \textbf{0.0273} & \textbf{0.2657} \\ \midrule
\multirow{4}{*}{10000} & Baseline                & 0.1511          & 0.0193          & 0.0164          & 0.0229          & 0.1853          \\ \cmidrule(l){2-7}
                       & Gaussian                & 0.1577          & \textbf{0.0162} & \textbf{0.0139} & \textbf{0.0184} & \textbf{0.1401} \\
                       & Uniform                 & \textbf{0.1489} & \textbf{0.0153} & \textbf{0.0120} & \textbf{0.0210} & \textbf{0.1508} \\
                       & Epanechnikov            & \textbf{0.1431} & \textbf{0.0165} & \textbf{0.0121} & \textbf{0.0187} & \textbf{0.1852} \\ \bottomrule
\end{tabular}
\end{adjustbox}
\end{table}

\subsection{Real data analysis}
We employ the ConquerNet to analyze the BMI (body mass index) data (\url{https://www.kaggle.com/datasets/chik0di/health-and-lifestyle-dataset/}). We consider the samples without family history of lifestyle diseases, group the samples by gender, and predict BMI with age, daily steps, sleep hours, water intake and calories consumed. We get 35011 male samples and 35074 female samples, train the baseline and ConquerNet with 80\% of samples, and evaluate the pinball loss on the rest of the samples. The results are shown in Table~\ref{tab:bmi_single}. Our ConquerNet achieves better performance in most cases, showing the
good generalizability.

\begin{table}[htpb]
\centering
\caption{Pinball loss performance for BMI (body mass index) prediction.}
\label{tab:bmi_single}
\begin{adjustbox}{width=0.6\textwidth,center}
\begin{tabular}{@{}cclllll@{}}
\toprule
\multirow{2}{*}{Gender} & \multirow{2}{*}{Method} & \multicolumn{5}{c}{Quantile Level} \\ 
\cmidrule(l){3-7}
                        &                         & \multicolumn{1}{c}{$\tau$=0.05} & \multicolumn{1}{c}{$\tau$=0.25} & \multicolumn{1}{c}{$\tau$=0.5} & \multicolumn{1}{c}{$\tau$=0.75} & \multicolumn{1}{c}{$\tau$=0.95} \\ \midrule
\multirow{4}{*}{Male}   & Baseline                & 0.5231                          & 2.0638                          & 2.7387                         & 2.0534                          & 0.5190                          \\ \cmidrule(l){2-7}
                        & Gaussian                & \textbf{0.5221}                 & \textbf{2.0609}                 & \textbf{2.7384}                & \textbf{2.0513}                 & \textbf{0.5189}                 \\
                        & Uniform                 & \textbf{0.5217}                 & \textbf{2.0618}                 & 2.7403                         & \textbf{2.0521}                 & 0.5192                          \\
                        & Epanechnikov            & \textbf{0.5218}                 & \textbf{2.0636}                 & 2.7389                         & \textbf{2.0518}                 & 0.5205                          \\ \midrule
\multirow{4}{*}{Female} & Baseline                & 0.5218                          & 2.0596                          & 2.7366                         & 2.0555                          & 0.5249                          \\ \cmidrule(l){2-7}
                        & Gaussian                & \textbf{0.5202}                 & \textbf{2.0446}                 & \textbf{2.7323}                & \textbf{2.0531}                 & 0.5259                          \\
                        & Uniform                 & \textbf{0.5211}                 & \textbf{2.0497}                 & \textbf{2.7311}                & 2.0642                          & 0.5251                          \\
                        & Epanechnikov            & \textbf{0.5205}                 & \textbf{2.0454}                 & \textbf{2.7276}                & \textbf{2.0523}                 & \textbf{0.5248}                 \\ \bottomrule
\end{tabular}
\end{adjustbox}
\end{table}

\subsection{Plots of loss landscape}
{To express the benefit of the ConquerNet more intuitively, we tried to plot the loss landscape of the baseline network and the ConquerNet based on \citet{li2018visualizing}, which are shown in Figure \ref{fig:loss-landscape-S2-q0.5}. The networks have the same structures, consisting of 20 layers, each with 35 nodes. We generated two random directions and ensured the models shared the same directions. Then we used the filter-wise normalization method in \citet{li2018visualizing} and calculated the loss surfaces. 
The loss landscape plot shows that the baseline network (see the subplot on the left) has at least three local minima. In contrast, the ConquerNet in the right subplot has one minimum, which implies that the ConquerNet improves the training dynamics.
}

\begin{figure}[htpb]
    \centering
        \begin{subfigure}[b]{0.4\textwidth} 
        \centering
        \includegraphics[width=\textwidth]{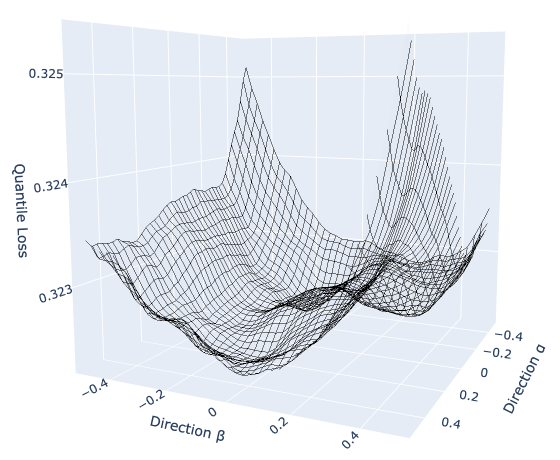} 
        \caption{Baseline}
    \end{subfigure}
    \begin{subfigure}[b]{0.4\textwidth}
        \centering
        \includegraphics[width=\textwidth]{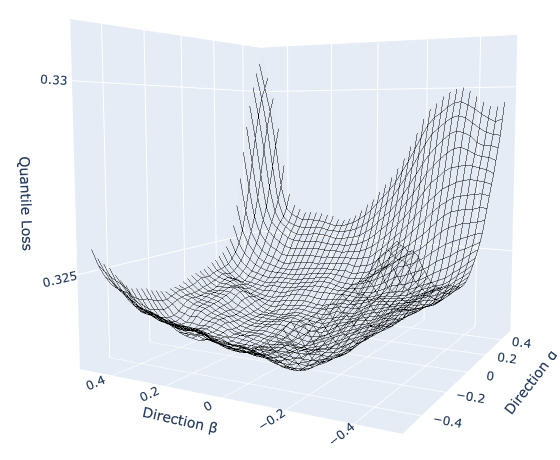}
        \caption{ConquerNet}
    \end{subfigure}

    \caption{Plot of 3D loss landscape of scenario 2, $\tau=0.5$. The subplot on the left represents the baseline model, and the right subplot represents the ConquerNet smoothed by the Gaussian kernel with $h=0.1$. Both networks consist of 20 layers, each with 35 nodes.}
    \label{fig:loss-landscape-S2-q0.5}
\end{figure}

\section{Discussion}
\label{sec:discussion}
{Our smoothing principle can potentially be extended beyond quantile regression to other distributional objectives such as CRPS and Wasserstein distances. As discussed in \cite{berrisch2023crps}, CRPS can be expressed as an integral of quantile losses over $\tau \in [0,1]$. 
Therefore, one can consider the smoothed objective
\[
\min_{f_\tau\in\mathcal{F}(P,U,L),\, \tau\in[0,1]} 
\frac{1}{n} \sum_{i=1}^n \int_0^1 \ell_h(y_i - f_\tau(x_i)) \, \mathrm{d}\tau,
\]
where $\ell_h$ is our convolution-type smoothed quantile loss. }

{Regarding the potential extension to Wasserstein-based objectives, consider the empirical Wasserstein-$p$ distance between the predictive CDF $F_f(\cdot|x_i)$ ($F_f^{-1}(\tau|x_i)=f_\tau(x_i)$) and the target CDF $G(\cdot|x_i)$
$$
 W_p(F_f, G)=\sum_{i=1}^n\left( \int_0^1 \big| f_\tau(x_i) - G^{-1}(\tau|x_i) \big|^p \, d\tau \right)^{1/p}.
$$
It is important to note that directly minimizing a Wasserstein distance is not straightforward in our conditional quantile regression, since the target $G^{-1}(\tau|x_i)$ is unknown. One alternative is to use the empirical distribution $G_n$ as the target, i.e.,
$$
 \hat W_p(F_f, G_n)=\sum_{i=1}^n\left( \int_0^1 \big| f_\tau(x_i) - G_n^{-1}(\tau|x_i) \big|^p \, d\tau \right)^{1/p},
$$
where
$$
G_n^{-1}(\tau|x_i) = \inf \Big\{ y : \frac{1}{|\mathcal{N}_i|} \sum_{j \in \mathcal{N}_i} \mathbf{1}_{\{y_j \le y\}} \ge \tau \Big\}, 
\quad \mathcal{N}_i = \{ j : x_j = x_i \}.
$$
However, this requires repeated observations $(y_j,x_j)$ at the same covariate $x_j=x_i$. 
Such a setting arises in reinforcement learning \citep{dabney2018distributional}, 
where multiple returns $y_i$ for a fixed state $x_i$ can be observed under different rollouts. In Bayesian settings, \cite{zhang2020propagation} chain quantile regression with Wasserstein-based objectives for posterior inference. 
Some existing literature provides partial connections between the quantile loss and Wasserstein objectives. 
\cite{lheritie2022cramer} discussed that optimization based on the quantile loss can be interpreted as a 1-Wasserstein projection in certain settings, 
while \cite{yang2024multiple} introduces a Wasserstein-improved objective for composite quantile regression. While a direct application is beyond the scope of our current work, these studies highlight potential directions for future research.}



\bibliography{our_main}
\bibliographystyle{tmlr}

\newpage
\appendix
\onecolumn

\section{Proofs}
\label{sec:proof}
We first state several definitions to develop our Empirical process theorems and auxiliary lemmas. Define the empirical loss function as
$$
\hat{M}_n(f)=\sum_{i=1}^n \hat M_{n,i}(f),\quad \hat M_{n,i}(f)=\frac{1}{n}(\ell_h(y_i-f(\mf x_i))-\ell_h(y_i-f_n(\mf x_i))),
$$
and we set 
$$
M_{n}(f)=\mb E\left\{\frac{1}{n}\sum_{i=1}^n\left[\ell_h(y_i-f(\mf x_i))-\ell_h(y_i-f_n(\mf x_i))\right]\right\}.
$$
For $\epsilon > 0$ and a metric $\text{dist}(\cdot, \cdot)$ on the class of functions $\mathcal{F}$, we define the covering number $\mathcal{N}(\epsilon, \mathcal{F}, \text{dist}(\cdot, \cdot))$ as the minimum number of balls of the form $\{ g : \text{dist}(g, f) \leq \epsilon \}$, with $f \in \mathcal{F}$, needed to cover $\mathcal{F}$ (see Definition 2.2 in \cite{sen2018gentle} for details). We write $\tilde{\mathcal{I}}(L,W,S,B)=\{f:f\in\mathcal{I}(L,W,S,B),\|f\|_\infty\leq F\}$.
For simplicity, we consider $[0,1]^d$ instead of $[-H,H]^d$ in Assumptions 3-4.

\subsection{Auxiliary Lemmas}

\begin{lemma}
\label{lem:f-f_n}
Suppose that $\left\|f_n-f_\tau^*\right\|_{\infty} \leq c$ for a small enough constant $c$ we have
\begin{equation}
\Delta^2(f_n,f) \leq C\left[\mathbb{E}\left(\ell_h(Y-f(X))-\ell_h\left(Y-f_n(X)\right)\right)+(\left\|f_n-f_\tau^*\right\|_{\infty} +h^2)\Delta\left(f, f_n\right) \sqrt{F}\right].\label{eq:f-f_n delta} 
\end{equation}
and
\begin{equation}
\left\|f-f_n\right\|_{\ell_2}^2 \leq C\max\{1,F\}\left[\mathbb{E}\left(\ell_h(Y-f(X))-\ell_h\left(Y-f_n(X)\right)\right)+(\left\|f_n-f_\tau^*\right\|_{\infty}+h^2)\left\|f-f_n\right\|_{\ell_2} \sqrt{F}\right],\label{eq:f-f_n l2}    
\end{equation}
for any $f \in \tilde{\mathcal{I}}(L, W, S, B)$ and for some constant $C>0$.    
\end{lemma}

\begin{proof}
By Knight identity \citep{knight1998limiting}, 
\begin{align}
&\rho_\tau(Y-f(X)+s)-\rho_\tau\left(Y-f_n(X)+s\right)\notag\\
&=-\left(f(X)-f_n(X)\right)\left(\tau-\mf1\left\{Y+s \leq f_n(X)\right\}\right)+\int_0^{f(X)-f_n(X)}\left[\mf1\left\{Y+s \leq f_n(X)+z\right\}-\mf1\left\{Y+s \leq f_n(X)\right\}\right] \mr d z ,\notag\\
&= -\left(f(X)-f_n(X)\right)\left(\tau-\mf1\left\{Y+s \leq f_\tau^*(X)\right\}\right)-\left(f(X)-f_n(X)\right) \left(\mf1\left\{Y+s \leq f_\tau^*(X)\right\}-\mf1\left\{Y+s \leq f_n(X)\right\}\right)\notag\\
&+\int_0^{f(X)-f_n(X)}\left[\mf1\left\{Y+s \leq f_n(X)+z\right\}-\mf1\left\{Y+s \leq f_n(X)\right\}\right] \mr d z.\label{eq:decomp rho diff}
\end{align}
By \eqref{eq:assump1 lip lower} in Assumption 2 and mean value expansion, applying Fubini's theorem and the fact $\int sK_h(s)\mr ds=0,\int s^2K_h(s)\mr ds=\sigma_K^2 h^2$, we have for some constant $\underline{c},c_\tau>0$,
\begin{align}
&\mb E\left[\int K_h(s) \int_0^{f(X)-f_n(X)}\mf1\left\{Y +s\leq f_n(X)+z\right\}-\mf1\left\{Y+s \leq f_n(X)\right\}\mr dz\mr ds\Bigg| X\right],\notag\\
&=\mb E\left[\int K_h(s) \int_0^{f(X)-f_n(X)} F_{Y|X}(f_n(X)+z-s)-F_{Y|X}(f_n(X)-s)\mr dz\mr ds \right],
\notag\\
&\geq  \mb E \left[ \underline{p}\int_0^{f(X)-f_n(X)} \min\{z,\kappa\} \mr dz \right]-\underline{c}h^2\mb E|f(X)-f_n(X)|, \notag\\
&\geq c_\tau \left(\mb E\left[D^2(f(X) -f_n(X) )\right]-h^2\mb E|f(X)-f_n(X)|\right),  \label{eq:decomp int lower}
\end{align}
Combining \eqref{eq:key ell-rho}, \eqref{eq:decomp rho diff}, and \eqref{eq:decomp int lower}, by Fubini's theorem, we have
\begin{align}
&\mb E\left\{\ell_h(Y-f(X))-\ell_h(Y-f_n(X))\right\}\notag\\
&=\int K_h(s)   
    \mb E\left\{\rho_\tau(Y+s-f(X))-\rho_\tau\left(Y+s-f_n(X)\right)\right\}\mr ds,\notag\\
    &\geq  -C\mb E\left\{|f(X)-f_n(X)|\cdot(|f_\tau^*(X)-f_n(X)|+h^2)\right\}+c_\tau \left(\mb E\left[D^2(f(X) -f_n(X) )\right]\right),\notag\\
    &\geq -C (\|f_\tau^*-f_n\|_\infty+h^2) \sqrt{F\Delta^2(f,f_n)}+c_\tau \Delta^2(f,f_n),
\end{align}
which yields \eqref{eq:f-f_n delta}. Furthermore, note that 
$\frac{\|f-f_n\|^2_{\ell^2}}{\max\{F,1\}}\leq \Delta^2(f,f_n)\leq \|f-f_n\|^2_{\ell^2}$, then \eqref{eq:f-f_n l2} holds.
\end{proof}

\begin{lemma}
\label{lem:f-f_n empirical}
Suppose that $f_n\in\tilde{\mathcal{I}}(L,W,S,B)$ and $\left\|f_n-f_\tau^*\right\|_{\infty} \leq c$ for a sufficiently small constant $c>0$. The estimator $\hat{f}_h$ defined in \eqref{eq:estimator besov} satisfies for some constant $C>0$,
\begin{equation}
    \Delta^2\left(\hat{f}_h, f_n\right) \leq C\left[ M_n(\hat{f}_h)-\hat{M}_n(\hat{f}_h)+(\left\|f_n-f_\tau^*\right\|_{\infty}+h^2) \Delta\left(\hat{f}_h, f_n\right) \sqrt{F}\right].\label{eq:f-f_n delta empirical}
\end{equation}
Furthermore,
\begin{equation}
    \left\|\hat{f}_h-f_n\right\|_{\ell_2}^2 \leq C\max\{1,F\}\left[M_n(\hat{f}_h)-\hat{M}_n(\hat{f}_h)+(\left\|f_n-f_\tau^*\right\|_{\infty}+h^2)\left\|\hat{f}_h-f_n\right\|_{\ell_2} \sqrt{F}\right] .\label{eq:f-f_n delta l2 empirical}
\end{equation}
\end{lemma}

\begin{proof}
Since $\hat{f}_h$ in \eqref{eq:estimator besov} satisfies $\hat{f}_h\in\tilde{\mathcal{I}}(L,W,S,B)$, then by Lemma \ref{lem:f-f_n},
\begin{align}
    \Delta^2(\hat{f}_h,f_n)&\leq  C\left[\mathbb{E}\left(\ell_h(Y-\hat{f}_h(X))-\ell_h\left(Y-f_n(X)\right)\right)+\left\|f_n-f_\tau^*\right\|_{\infty} \Delta\left(\hat{f}_h, f_n\right) \sqrt{F}+h^2\right],\notag\\
    &\leq C\left[ M_n(\hat{f}_h)-\hat{M}_n(\hat{f}_h)+\left\|f_n-f_\tau^*\right\|_{\infty} \Delta\left(\hat{f}_h, f_n\right) \sqrt{F}+h^2 \right],
\end{align}
where the last inequality is obtained by the fact $\hat{M}_n(\hat{f}_h)\leq 0$. Similarly, \eqref{eq:f-f_n delta l2 empirical} can also hold by Lemma \ref{lem:f-f_n} and the optimality of $\hat{f}_h$.  

\end{proof}

\begin{lemma}
\label{lem:radius rademacher}
Suppose that
\begin{equation}
3 \mathbb{E}\left(\sup _{f \in \tilde{\mathcal{I}}(L, W, S, B),\left\|f-f_n\right\|_{\ell_2}^2 \leq r^2} \frac{1}{n} \sum_{i=1}^n \xi_i\left(f\left(\mf x_i\right)-f_n\left(\mf x_i\right)\right)^2\right) \leq r^2,\label{eq:3rademacher}    
\end{equation}
for $\left\{\xi_i\right\}_{i=1}^n$ Rademacher variables independent of $\left\{\left(\mf x_i, y_i\right)\right\}_{i=1}^n$, and
\begin{equation}
2 F \sqrt{\frac{7\gamma}{3 n}}\leq r
\label{eq:lower radius}    
\end{equation}
Then with probability at least $1-e^{-\gamma},\left\|f-f_n\right\|_{\ell_2}^2 \leq r^2$ with $f \in \tilde{\mathcal{I}}(L, W, S, B)$ implies
$$
\left\|f-f_n\right\|_n^2 \leq(2 r)^2.
$$
\end{lemma}

\begin{proof}
Using the fact $(a-b)^2 \leq 2(a^2+b^2)$, we have for $f\in \tilde{\mathcal{I}}(L,W,S,B)$
\begin{align}
    \operatorname{Var}\left[(f(X)-f_n(X))^2\right]&\leq \mb E(f(X)-f_n(X))^4,\notag\\
    &\leq 2(F^2+\|f_n\|_\infty^2)\mb E(f(X)-f_n(X))^2,\notag\\
    &\leq 4F^2\|f-f_n\|_{\ell_2}^2.
\end{align}
Note that $0\leq (f(\mf x)-f_n(\mf x))^2\leq 4F^2$, then by Theorem 2.1 in \cite{bartlett2005local}, for every $\gamma>0$, with probability at least $1-e^{-\gamma}$,
\begin{align}
& \sup _{f \in \tilde{\mathcal{I}}(L, W, S, B),\left\|f-f_n\right\|_{\ell_2}^2 \leq r^2}\left\{\left\|f-f_n\right\|_n^2-\left\|f-f_n\right\|_{\ell_2}^2\right\} \notag\\
& \leq 3 \mathbb{E}\left(\sup _{f \in \tilde{\mathcal{I}}(L, W, S, B),\left\|f-f_n\right\|_{\ell_2}^2 \leq r^2} \frac{1}{n} \sum_{i=1}^n \xi_i\left(f\left(X_i\right)-f_n\left(X_i\right)\right)^2\right)+r\cdot 2F \sqrt{\frac{2\gamma}{n}}+\frac{28 F^2 \gamma}{3 n},\notag\\
&\leq 3r^2,\notag
\end{align}
where the last inequality is obtained by \eqref{eq:3rademacher} and \eqref{eq:lower radius}. Then with probability at least $1-e^{-\gamma},\left\|f-f_n\right\|_{\ell_2}^2 \leq r^2$ with $f \in \tilde{\mathcal{I}}(L, W, S, B)$ implies $\left\|f-f_n\right\|_n^2 \leq 4r^2.$

\end{proof}

\begin{lemma}
\label{lem:f-f_n l2 high prob}
Suppose that $h^2\lesssim n^{-\frac{s}{2s+d}}$ and $\left\|\hat{f}_h-f_n\right\|_{\ell_2} \leq r_0$, with $r_0$ satisfying \eqref{eq:3rademacher}, \eqref{eq:lower radius}, and Assumption 4 holds. Also, with the notation of Assumption 3, suppose that for the class $\mathcal{I}(L, W, S, B)$ the parameters are chosen as
\begin{align}
& L=3+2\left\lceil\log _2\left(\frac{3^{\max \{d, m\}}}{\epsilon c_{d, m}}\right)+5\right\rceil\left\lceil\log _2 \max \{d, m\}\right\rceil, \quad W=W_0 N \label{eq:cond1 lem:f-f_n l2 high prob}\\
& S=(L-1) W_0^2 N+N, \quad B=O\left(N^{\left(v^{-1}+d^{-1}\right)(\max \{1,(d / p-s)+\})}\right)\label{eq:cond2 lem:f-f_n l2 high prob}
\end{align}
for a constant $c_{d, m}$ that depends on $d$ and $m$, a constant $W_0$, and where $v=(s-\delta) / \delta$,
\begin{equation}
\delta=\frac{d}{p}, \quad N \asymp n^{\frac{d}{2 s+d}}.\label{eq:cond3 lem:f-f_n l2 high prob}
\end{equation}
Then there exists a universal constant $C_0>0$ such that
$$
\begin{aligned}
\left\|\hat{f}_h-f_n\right\|_{\ell_2}^2 \leq &  C_0\left[r_0 F^{5/2} \sqrt{\frac{\gamma}{n}}+\frac{F^{5/2} \gamma}{n}+\right. \\
&\left. r_0 F \sqrt{\frac{N(\log N)^2}{n}}+r_0 F \sqrt{\frac{N\left[(\log N)^2+\log r_0^{-1}+\log n\right]}{n}}+N^{-s / d} r_0 F^{3/2}\right]
\end{aligned}
$$
with probability at least $1-3e^{-\gamma}$, where $N \asymp n^{\frac{d}{2 s+d}}$.
\end{lemma}

\begin{proof}
Let
$$
\mathcal{G}=\left\{g: g(\mf x, y)=\ell_h(y-f(\mf x))-\ell_h\left(y-f_n(\mf x)\right), \quad f \in \tilde{\mathcal{I}}(L, W, S, B),\left\|f-f_n\right\|_{\ell_2} \leq r_0\right\}.
$$
Then for $\xi_1, \ldots, \xi_n$ independent Rademacher variables independent of $\left\{\left(\mf x_i, y_i\right)\right\}_{i=1}^n$, by Theorem 2.1 in \cite{bartlett2005local}, with probability at least $1-2e^{-\gamma}$, we have that
\begin{align}
&  M_n(\hat{f}_h)-\hat{M}_n(\hat{f}_h)+(\left\|f_n-f_\tau^*\right\|_{\infty}+h^2)\left\|\hat{f}_h-f_n\right\|_{\ell_2} \sqrt{F}\notag\\
\lesssim & \mathbb{E}\left(\left.\sup _{g \in \mathcal{G}} \frac{1}{n} \sum_{i=1}^n \xi_i g\left(\mf x_i, y_i\right) \right\rvert\,\left(\mf x_1, y_1\right), \ldots,\left(\mf x_n, y_n\right)\right)\notag \\
& +4 r_0 F^{3/2} \sqrt{\frac{\gamma}{n}}+\frac{100 F^{3/2} \gamma}{3 n}+(\left\|f_n-f_\tau^*\right\|_{\infty}+h^2)\left\|\hat{f}_h-f_n\right\|_{\ell_2} F^{1/2}.\label{eq:f-f_n l2 high prob}
\end{align}

Denote $\mb E_\xi$ as the expectation with respect to $\xi_1,\dots,\xi_n$. Let $$
\varphi_{f,i}(t_i)=\ell_h(y_i-(t_i+f_n(\mf x_i)))-\ell_h(y_i-f_n(\mf x_i)),
$$
where $t_i=f(\mf x_i)-f_n(\mf x_i)$. Note that $\ell_h(\cdot)$ is 1-Lipschitz continuous and $\varphi_{f,i}(0)=0$, by Talagrand's inequality (\cite{ledoux2013probability}), Lemma \ref{lem:radius rademacher}, with probability at least $1-e^{-\gamma}$, we have
\begin{align}
\mathbb{E}_{\xi}\left(\sup _{g \in \mathcal{G}} \frac{1}{n} \sum_{i=1}^n \xi_i g\left(\mf x_i, y_i\right)\right)&=\mathbb{E}_{\xi}\left(\sup _{f \in \mathcal{I}(L, W, S, B),\|f\|_{\infty} \leq F,\left\|f-f_n\right\|_{\ell_2} \leq r_0} \frac{1}{n} \sum_{i=1}^n \xi_i\varphi_{f,i}(t_i) \right),  \notag\\
& \leq \mathbb{E}_{\xi}\left(\sup _{f \in \mathcal{I}(L, W, S, B),\|f\|_{\infty} \leq F,\left\|f-f_n\right\|_{\ell_2} \leq r_0} \frac{1}{n} \sum_{i=1}^n \xi_i\left(f\left(\mf x_i\right)-f_n\left(\mf x_i\right)\right)\right), \notag\\
& \leq \mathbb{E}_{\xi}\left(\sup _{f \in \mathcal{I}(L, W, S, B),\|f\|_{\infty} \leq F,\left\|f_n-f\right\|_n \leq 2 r_0} \frac{1}{n} \sum_{i=1}^n \xi_i\left(f\left(\mf x_i\right)-f_n\left(\mf x_i\right)\right)\right).\label{eq:rademacher high prob}
\end{align}
By Dudley's chaining inequality and arguments in the proof of Theorem 2 in \cite{Suzuki2018AdaptivityOD}, we further have for some constant $C>0$,
\begin{align}
&\mathbb{E}_{\xi}\left(\sup _{f \in \mathcal{I}(L, W, S, B),\|f\|_{\infty} \leq F,\left\|f_n-f\right\|_n \leq 2 r_0} \frac{1}{n} \sum_{i=1}^n \xi_i\left(f\left(\mf x_i\right)-f_n\left(\mf x_i\right)\right)\right)\notag\\
& \leq \inf _{0<\alpha<r_0}\left\{4 \alpha+\frac{24 r_0}{\sqrt{n}} \sqrt{\log \mathcal{N}\left(\alpha, \tilde{\mathcal{I}}(L, W, S, B),\|\cdot\|_{\infty}\right)}\right\} \notag\\
& \leq C \inf _{0<\alpha<r_0}\left\{\alpha+r_0 \sqrt{\frac{N\log (N)\left[\log^2 N+\log \alpha^{-1}\right]}{n}}\right\}\label{eq:rademacher chaining}
\end{align}
By Proposition 1 in \cite{Suzuki2018AdaptivityOD} and $N \asymp n^{\frac{d}{2 s+d}},h^2\lesssim n^{-\frac{s}{2s+d}}$, we have $\left\|f_n-f_\tau^*\right\|_{\infty}+h^2 \lesssim N^{-s / d}$ . Let
$$
\alpha=r_0 \sqrt{\frac{N(\log N)^2}{n}},
$$
together with Lemma \ref{lem:f-f_n empirical}, \eqref{eq:rademacher chaining}, \eqref{eq:f-f_n l2 high prob}, and \eqref{eq:rademacher high prob}, we have for some constant $C>0$,
$$
\begin{aligned}
\left\|\hat{f}_h-f_\tau^*\right\|_{\ell_2}^2 \leq &  C\left[4 r_0 F^{5/2} \sqrt{\frac{\gamma}{n}}+\frac{100 F^{5/2} \gamma}{3 n}+\right. \\
& \left. r_0 F \sqrt{\frac{N(\log N)^2}{n}}+ r_0 F \sqrt{\frac{N\log (N)\left[(\log N)^2+\log r_0^{-1}+\log n\right]}{n}}+ r_0 N^{-s / d} F^{3/2}\right]
\end{aligned}
$$
with probability at least $1-3 e^{-\gamma}$.
\end{proof}

\begin{lemma}[Lemma 20 in \cite{padilla2022quantile}]
    \label{lem:radius rademacher star}
Let $r^*$ be defined as
$$
r^*=\inf \left\{r>0: 3 \mathbb{E}\left(\sup _{f \in \tilde{\mathcal{I}}(L, W, S, B),\left\|f-f_n\right\|_{\ell_2} \leq s} \frac{1}{n} \sum_{i=1}^n \xi_i\left(f\left(x_i\right)-f_n\left(x_i\right)\right)^2\right)<s^2, \forall s \geq r\right\},
$$
for $\left\{\xi_i\right\}_{i=1}^n$ Rademacher variables independent of $\left\{\left(x_i, y_i\right)\right\}_{i=1}^n$. Then under the conditions \eqref{eq:cond1 lem:f-f_n l2 high prob}, \eqref{eq:cond2 lem:f-f_n l2 high prob}, and \eqref{eq:cond3 lem:f-f_n l2 high prob} of Lemma \ref{lem:f-f_n l2 high prob},
\begin{equation}
    r^* \leq \tilde{C}\left[\sqrt{\frac{N(\log N)^2}{n}}+\sqrt{\frac{N\log (N)\left[(\log N)^2+\log n\right]}{n}}\right],\label{eq:r star upper bound lemma}
\end{equation}
for a constant $\tilde{C}>0$ and with $N$ satisfying $N \asymp n^{\frac{d}{2 s+d}}$.
\end{lemma}

Let $\mathcal{H}$ be a class of functions from $\mathcal{X}$ to $\mathbb{R}$. We define the pseudodimension of $\mathcal{H}$, denoted as $\operatorname{Pdim}(\mathcal{H})$, as the largest integer $m$ for which there exist $\left(a_1, b_1\right) \ldots,\left(a_m, b_m\right) \in \mathcal{X} \times \mathbb{R}$ such that for all $\eta \in\{0,1\}^m$ there exists $f \in \mathcal{H}$ such that
$$
f\left(a_i\right)>b_i \Longleftrightarrow \eta_i,
$$
for $i=1, \ldots, m$.
\begin{lemma}[Theorem 12 from \cite{padilla2022quantile}]
\label{lem:pseudodimension}
With the notation from before, for the neural network function class $\mathcal{F}(P,U,L)$, we have
$$
\operatorname{Pdim}(\mathcal{F}(P, U, L))=O(L P \log (U)) .
$$ 
\end{lemma}

\subsection{Proof of Theorem \ref{thm:minimax rate}}

\begin{proof}
Write the ball centered in $f$ of radius $r$
$$
\mr B(f, \| \cdot \|_{\ell_2}, r) = \{ g : \| f - g \|_{\ell_2} \leq r \}.
$$
We divide the space $\tilde{\mathcal{I}}(L,W,S,B)=\{f:f\in\mathcal{I}(L,W,S,B),\|f\|_\infty\leq F\}$ into sets of increasing radius
\[
\mr B(f_n, \| \cdot \|_{\ell_2}, \bar r), \mr B(f_n, \| \cdot \|_{\ell_2}, 2\bar r) \setminus \mr B(f_n, \| \cdot \|_{\ell_2}, \bar r), \dots, \mr B(f_n, \| \cdot \|_{\ell_2}, 2^{l}\bar r)\setminus\mr B(f_n, \| \cdot \|_{\ell_2}, 2^{l-1}\bar r),
\]
where
\[
l = \left\lfloor \log_2 \left( \frac{2F}{\sqrt{\log n / n}} \right) \right\rfloor.
\]
If for some $j \leq l$,
\[
\hat{f}_h \in \mr B(f_n, \| \cdot \|_{\ell_2}, 2^{j}\bar r),
\]
then by Lemma \ref{lem:f-f_n l2 high prob}, with probability at least $1 - 3e^{-\gamma}$, we have for some constant $\tilde C>0$,
\begin{align}
\left\|\hat{f}_h-f_n\right\|_{\ell_2}^2 \leq & \tilde C\left(2^j\bar r F^{5/2} \sqrt{\frac{\gamma}{n}}+\frac{F^{5/2} \gamma}{n}+\right. \notag\\
&\left. 2^j\bar r F \sqrt{\frac{N(\log N)^2}{n}}+2^j\bar r F \sqrt{\frac{N\log(N)\left[(\log N)^2+2\log n\right]}{n}}+N^{-s / d} 2^j\bar r F^{3/2}\right)\label{eq:localization}.
\end{align}
Recall the $r^*$ defined in Lemma \ref{lem:radius rademacher star}. We set 
\begin{align}
    \bar r &= 8\tilde C \left[ F^{5/2} \sqrt{\frac{\gamma}{n}} + F \sqrt{\frac{N(\log N)^2}{n}} + F \sqrt{\frac{N\log(N)[(\log N)^2 +2 \log n]}{n}} + N^{- s/d} F^{3/2} \right] \notag\\
    &+ 2\sqrt{2\tilde C}\cdot \sqrt{  \frac{F^{5/2} \gamma}{n}}+r^*.\label{eq:r bar setting}    
\end{align}
By Lemma \ref{lem:radius rademacher star} and $N\asymp n^{d/(2s+d)}$, setting  $\gamma\asymp \log n$, it follows that $\tilde{\mathcal{I}}(L,W,S,B)\subset\mr B(f_n,\|\cdot\|_{\ell_2},2^l\bar r)$ when $n$ is sufficiently large. Therefore, $\hat f_h\in \mr B(f_n,\|\cdot\|_{\ell_2},2^l\bar r)$ with probability 1 if $n$ is sufficiently large.

Elementary calculation yields $\bar r > r^*$ and for all $0\leq j\leq l$,
\begin{equation}
    \frac{2^{j} \bar r}{8}\geq  \tilde C \left[ F^{5/2} \sqrt{\frac{\gamma}{n}}+ F\sqrt{\frac{N(\log N)^2}{n}} + F\sqrt{\frac{N\log (N)[(\log N)^2 + 2 \log n]}{n}} + N^{- s/d} F^{3/2} \right],\label{eq:localization r1}
\end{equation}
and
\begin{equation}
    \frac{2^{2j} \bar r^2}{8}\geq \tilde{C} \left(\frac{ F^{5/2} \gamma}{n} \right) .\label{eq:localization r2}
\end{equation}

Combining \eqref{eq:localization r1}, \eqref{eq:localization r2}, and \eqref{eq:localization}, we have with probability at least $1-3e^{-\gamma}$,
\begin{equation}
\|\hat{f}_h-f_n\|_{\ell_2}\leq 2^{j-1} \bar r.  
\end{equation}

Now we begin from the first step of our localization procedure. Note that $\bar r > r^*$, by Lemmas \ref{lem:radius rademacher} and \ref{lem:f-f_n l2 high prob}, with probability at least $1-e^{-\gamma}$, we have that 
$$
\left\|\hat{f}_h-f_n\right\|_{\ell_2} \leq 2^l \bar{r} \quad \text {implies} \quad\left\|\hat{f}_h-f_n\right\|_n \leq 2^{l+1} \bar{r}.
$$
Then by arguments above, with probability at least $1-4e^{-\gamma}$, $\left\|\hat{f}_h-f_n\right\|_{\ell_2}\leq 2^l \bar{r}$ implies that
$$
\left\|\hat{f}_h-f_n\right\|_{\ell_2} \leq 2^{l-1} \bar{r}, \quad \text { and } \quad\left\|\hat{f}_h-f_n\right\|_n \leq 2^l \bar{r}.
$$
Continue recursively, we arrive at
$$
\left\|\hat{f}_h-f_n\right\|_{\ell_2} \leq \bar{r}, \quad \text { and } \quad\left\|\hat{f}_h-f_n\right\|_n \leq 2 \bar{r},
$$
with probability approaching one. 

Specifically, this procedure can be formulated as
\begin{align}
    \mr P(\hat{f}_h\in \mr B(f_n,\|\cdot\|_{\ell_2},\bar r))&=\mr P(\hat{f}_h\in \mr B(f_n,\|\cdot\|_{\ell_2},2\bar r))-\mr P(\hat{f}_h\in \mr B(f_n,\|\cdot\|_{\ell_2},2\bar r)\setminus \mr B(f_n,\|\cdot\|_{\ell_2},\bar r) ),\notag\\
    &\geq \mr P(\hat{f}_h\in \mr B(f_n,\|\cdot\|_{\ell_2},2\bar r))-4e^{-\gamma},\notag\\
    &\geq \cdots,\notag\\
    &\geq 1-4(l+1)e^{-\gamma}=1-o(1),
\end{align}
with setting $\gamma\asymp \log n$.

By Lemma \ref{lem:radius rademacher star} and \eqref{eq:r bar setting}, we have
$$
\begin{aligned}
\bar{r} \leq & 8 \tilde{C}\left[F^{5/2} \sqrt{\frac{\gamma}{n}}+F \sqrt{\frac{N(\log N)^2}{n}}+F \sqrt{\frac{N\log(N)\left[(\log N)^2+2 \log n\right]}{n}}+N^{-s / d} F^{3/2}\right]\\
&+2\sqrt{2}\cdot\sqrt{\frac{\tilde{C} F^{5/2} \gamma}{n}}+ \tilde{C}\left[\sqrt{\frac{N(\log N)^2}{n}}+\sqrt{\frac{N\log(N)\left[(\log N)^2+\log n\right]}{n}}\right].
\end{aligned}
$$
Then the claim follows by $N\asymp n^{d/(2s+d)}$ and $\left\|f_n-f_\tau^*\right\|_{\infty} \lesssim N^{-s / d}$ in Proposition 1 of \cite{Suzuki2018AdaptivityOD}.
\end{proof}

\subsection{Proof of Theorem \ref{thm:general upper bound}}
\begin{proof}
Throughout this proof, the covariates $\mf x_1, \ldots, \mf x_n$ are fixed. For simplicity, we write $f\in\mathcal{F}$ instead of $f\in\mathcal{F}(P,U,L)$ within the proof. Let $\hat{\delta}_i=\hat{f}_h\left(\mf x_i\right)-f_\tau^*\left(\mf x_i\right)$ for $i=1, \ldots, n$. 
    Note that $\ell_h (u)=\int_{-\infty}^\infty \rho_\tau(v)K_h(v-u)\mr dv=\int_{-\infty}^\infty \rho_\tau(u+s)K_h(s)\mr ds$, then for any $x,y\in\mathbb R$,
\begin{equation}
    \ell_h(x)-\ell_h(y)=\int_{-\infty}^\infty K_h(s) \left(\rho_\tau(x+s)-\rho_\tau(y+s)\right) \mr ds. \label{eq:key ell-rho}   
\end{equation}
By Knight identity \citep{knight1998limiting}, for any $\delta\in\mb R$,
\begin{align}
&\rho_\tau(y_i-(f_\tau^*(\mf x_i)+\hat \delta_i)+s )-\rho_\tau\left(y_i-f_\tau^*(\mf x_i)+s\right)\notag\\
&=-\hat\delta_i\left(\tau-\mf1\left\{y_i+s \leq f_\tau^*(\mf x_i)\right\}\right)+ \int_0^{\hat\delta_i} (\mf 1\{y_i\leq f_\tau^*(\mf x_i)+z-s\}-\mf 1\{y_i\leq f_\tau^*(\mf x_i)-s\})\mr dz. \label{eq:rho diff}
\end{align}
By Assumption 1, Fubini's theorem and mean value expansion, using the fact $\int sK_h(s)\mr ds=0,\int s^2K_h(s)\mr ds=\sigma_K^2h^2$, we have
\begin{align}
    &\mb E\left[\frac{1}{n}\sum_{i=1}^n \int K_h(s)\int_0^{\hat \delta_i} (\mf 1\{y_i\leq f_\tau^*(\mf x_i)+z-s\}-\mf 1\{y_i\leq f_\tau^*(\mf x_i)-s\})\mr dz\mr ds\Bigg|\mf x_1,\dots,\mf x_n  \right]\notag\\
    &=\mb E\left[\frac{1}{n}\sum_{i=1}^n\int K_h(s)\int_0^{\hat \delta_i} F_{y_i|\mf x_i}(f_\tau^*(\mf x_i)+z-s)-F_{y_i|\mf x_i}(f_\tau^*(\mf x_i)-s)\mr dz\mr ds  \Bigg|\mf x_1,\dots,\mf x_n \right],\notag\\
    &\geq \mb E\left[\frac{1}{n}\sum_{i=1}^n\left(\underline{p}\int_0^{|\hat\delta_i|} \min\{z,\kappa\}\mr dz-\underline{c}|\hat\delta_i| h^2\right)\Bigg|\mf x_1,\dots,\mf x_n \right],\label{eq:int lower}
\end{align}
where the constants $\underline{p},\underline{c}>0$ are uniform for all $i=1,\dots, n$.

Similarly, we also have
\begin{align}
    &\mb E\left[ \frac{1}{n}\sum_{i=1}^n-\hat \delta_i\int K_h(s)(\tau-\mf 1\{y_i+s\leq f_\tau^*(\mf x_i)\})\mr ds  \right]\notag\\
    &=\mb E\left[\frac{1}{n}\sum_{i=1}^n-\hat\delta_i\int K_h(s) (F_{y_i|\mf x_i}(f_\tau^*(\mf x_i))-F_{y_i|\mf x_i}(f_\tau^*(\mf x_i)-s))\mr ds   \right]\notag\\
    &\geq -\underline{c}\mb E\left(\frac{1}{n}\sum_{i=1}^n|\hat\delta_i| h^2\right).\label{eq:lower int start}
\end{align}
Combining \eqref{eq:key ell-rho}, \eqref{eq:rho diff}, \eqref{eq:int lower}, and \eqref{eq:lower int start}, for $D_h(t)=\min\{|t|,t^2\}-h^2|t|$, we have
\begin{align}
&\mathbb{E}\left[\frac{1}{n} \sum_{i=1}^n D_h^2\left\{f_\tau^*\left(\mf x_i\right)-\hat{f}_h\left(\mf x_i\right)\right\}\right] \notag\\
& \leq \frac{1}{c_\tau n} \mathbb{E}\left(\sum_{i=1}^n \mathbb{E}\left[\ell_h\left\{y_i-f_\tau^*\left(\mf x_i\right)-\hat{\delta}_i\right\}\right]-\sum_{i=1}^n \mathbb{E}\left[\ell_h\left\{y_i-f_\tau^*\left(\mf x_i\right)\right\}\right]\right),\notag\\
&=\frac{1}{c_\tau}\mb E\left\{M_n(\hat{f}_h)\right\}+O(err_1).\label{eq:decomp general bound}
\end{align}
Next, we only need to bound $E\left\{M_n(\hat{f}_h)\right\}$. By symmetrization Lemma 10 in \cite{Padilla2021riskbound}, and Talagrand's inequality \citep{ledoux2013probability} using the fact $\ell_h(\cdot)$ is 1-Lipschitz continuous, for i.i.d. Rademacher variables $\xi_i,i=1,\dots,n$, 
\begin{align}
\mb E\left\{M_n(\hat{f}_h)\right\} & \leq \mb E\left\{ \sup_{f\in\mathcal{F},\|f\|_\infty\leq F}  \left[M_n(f)-\hat M_n(f)\right]  \right\},\notag\\
&\leq 2\mb E\left\{ \sup_{f\in\mathcal{F},\|f\|_\infty\leq F} \sum_{i=1}^n \xi_i\hat M_{n,i}(f)  \right\},\notag\\
&\leq 2\mb E\left\{ \sup_{f\in\mathcal{F},\|f\|_\infty\leq 2F} \frac{1}{n}\sum_{i=1}^n\xi_i f(\mf x_i) \right\}. \label{eq:symmetrization}
\end{align}
By Dudley's theorem and Lemma 4 in \cite{farrell2021deep}, we further have
\begin{align}
    F\mb E\left\{ \sup_{f\in\mathcal{F},\|f\|_\infty\leq F} \frac{1}{n}\sum_{i=1}^n\xi_i \frac{f(\mf x_i)}{F} \right\}&\leq \frac{C F}{\sqrt{n}} \int_0^2 \sqrt{\log \mathcal{N}\left(\mu, \mathcal{F} / F,\|\cdot\|_n\right)} \mr d \mu,\notag\\
& \leq \frac{C F}{\sqrt{n}} \int_0^2 \sqrt{\log \left(\left(\frac{2 \cdot e \cdot n}{\mu \cdot \operatorname{Pdim}(\mathcal{F})}\right)^{\operatorname{Pdim}(\mathcal{F})}\right)} \mr d \mu. \label{eq:dudley's bound}
\end{align}
Combining Lemma \ref{lem:pseudodimension}, \eqref{eq:symmetrization} and \eqref{eq:dudley's bound}, for some constant $\tilde{C}>0$, we have
\begin{align}
\mathbb{E}\left\{M_n(\hat{f}_h)\right\} &  \leq \tilde{C} F \sqrt{\frac{L P \log U \cdot \log n}{n}},
\end{align}
which shows \eqref{eq:general bound delta} combining with \eqref{eq:decomp general bound}. 

For \eqref{eq:general bound l2}, when $h^2=o(\sqrt{\mb E\|\hat{f}_h-f_\tau^*\|_{n}^2})$, by $\sum_i^n |a_i|/n\leq \sqrt{\sum_{i=1}^n a_i^2/n}$ and Jensen's inequality,
\begin{align}
    h^2\mb E\|\hat{f}_h-f_\tau^*\|_{n,1}\leq h^2 \mb E\sqrt{\|\hat{f}_h-f_\tau^*\|_n^2}\leq h^2 \sqrt{\mb E\|\hat{f}_h-f_\tau^*\|_n^2}=o(\mb E\|\hat{f}_h-f_\tau^*\|_n^2).
\end{align}
Then combining with the fact $\|\hat{f}_h-f_\tau^*\|_n^2\leq \max\{1,F\}\Delta_n^2(\hat{f}_h,f_\tau^*)$, \eqref{eq:general bound l2} can also hold.
\end{proof}

\section{Additional experiments}
\label{sec:additional experiments}

We state additional details of the experiment in this section. The whole experiments are implemented in PyTorch \citep{paszke2019pytorch}. We use stochastic gradient descent (SGD) with the Nesterov method of momentum factor 0.9. For each sample size, we keep $1/10$ of the data for validation. We start the learning rate at 0.1, and use scheduler \verb|ReduceLROnPlateau| with factor 0.5 and patience 5 to adjust the learning rate dynamically. We also implement gradient calculation for three different convolution-type smoothed loss functions, see Remark 3.1 in \citet{he2023smoothed}. We manually calculate gradients of the loss functions at the backpropagation step, and keep automatic differentiation for the rest of the computational graph. We also show MSE results without a stop criterion, that is, to train the models with full 100 epochs. Compared to the results in Table \ref{tab:mse}, there is no obvious advantage to training without stopping. On the contrary, using a stop criterion can improve generalization ability and reduce training time.

\subsection{Tables}

In this section, we first present MSE results for our ConquerNet with a sample size of 10000 and different bandwidth choices, as shown in Table \ref{tab:mse with h}. Bandwidths $h=\{0.001,0.005,0.01,0.05,0.1\}$ are considered. 
The results show that our ConquerNet outperform the baseline model over a wide range of bandwidths. 

Table \ref{tab:mse} and Figure \ref{fig:time_tau=0.05} show that our ConquerNet obtain better performance and higher training efficiency. The faster training mainly comes from our stopping strategy, that is to stop training if the learning rate is lower than a threshold and the validation loss does not decrease for several
consecutive epochs, while the baseline models are trained using all epochs according to the codes provided in \citet{padilla2022quantile}. Therefore, we also train the ConquerNet without the stop criterion and study the MSE results, which are shown in Table~\ref{tab:mse nostop}. We can conclude that our ConquerNet still outperform the baseline models.

\begin{table}[htpb]
\centering
\caption{MSE performances for scenario 1-3, model A and B, sample size 10000 under different bandwidths, quantile levels, and smoothing kernels. The MSEs are averaged over 50 independent trials.}
\label{tab:mse with h}
\begin{adjustbox}{height=0.45\textheight,center,keepaspectratio}
\begin{tabular}{@{}ccccccc|ccccc@{}}
\toprule
\multirow{2}{*}{$h$}   & \multirow{2}{*}{Method} & \multicolumn{5}{c|}{Scenario 1, Model A}                                                & \multicolumn{5}{c}{Scenario 1, Model B}                                                 \\ \cmidrule(l){3-12} 
                       &                         & $\tau$=0.05     & $\tau$=0.25     & $\tau$=0.5      & $\tau$=0.75     & $\tau$=0.95     & $\tau$=0.05     & $\tau$=0.25     & $\tau$=0.5      & $\tau$=0.75     & $\tau$=0.95     \\ \midrule
                       & Baseline                & \fbox{0.0618}          & 0.0067          & 0.0047          & 0.0063          & 0.1366          & 0.1691          & 0.0099          & 0.0066          & 0.0082          & 0.3882          \\ \cmidrule(l){2-12} 
\multirow{3}{*}{0.001} & Gaussian                & 0.0678          & \textbf{0.0064} & \textbf{0.0034} & \textbf{0.0055} & \textbf{0.0625} & \textbf{0.0654} & 0.0101          & \textbf{0.0061} & 0.0083          & \textbf{0.0708} \\
                       & Uniform                 & 0.0789          & \textbf{0.0058} & \textbf{0.0036} & \textbf{0.0057} & \textbf{0.0543} & \textbf{0.1356} & \textbf{0.0094} & \textbf{0.0060} & 0.0089          & \textbf{0.0679} \\
                       & Epanechnikov            & 0.0653          & \textbf{0.0061} & \textbf{0.0037} & \textbf{0.0053} & \textbf{0.0665} & \textbf{0.1225} & \textbf{0.0098} & \textbf{0.0056} & 0.0098          & \textbf{0.0699} \\ \cmidrule(l){2-12} 
\multirow{3}{*}{0.005} & Gaussian                & 0.0716          & \textbf{0.0061} & \textbf{0.0033} & \textbf{0.0054} & \textbf{0.0580} & \textbf{0.1062} & \textbf{0.0095} & \textbf{0.0055} & 0.0086          & \textbf{0.0726} \\
                       & Uniform                 & 0.0645          & \textbf{0.0056} & \textbf{0.0038} & \textbf{0.0056} & \textbf{0.0549} & \textbf{0.0974} & \textbf{0.0093} & \textbf{0.0055} & \textbf{0.0080} & \textbf{0.0736} \\
                       & Epanechnikov            & 0.0676          & \textbf{0.0062} & \textbf{0.0036} & \textbf{0.0053} & \textbf{0.0587} & \textbf{0.1294} & \textbf{0.0089} & \textbf{0.0057} & 0.0092          & \textbf{0.0663} \\ \cmidrule(l){2-12} 
\multirow{3}{*}{0.01}  & Gaussian                & 0.0624          & \textbf{0.0062} & \textbf{0.0037} & \textbf{0.0060} & \textbf{0.0662} & \textbf{0.0755} & \textbf{0.0091} & \textbf{0.0058} & 0.0087          & \textbf{0.0722} \\
                       & Uniform                 & 0.0687          & \textbf{0.0057} & \textbf{0.0037} & \textbf{0.0056} & \textbf{0.0805} & \textbf{0.0717} & \textbf{0.0090} & \textbf{0.0053} & 0.0088          & \textbf{0.0617} \\
                       & Epanechnikov            & 0.0800          & \textbf{0.0059} & \textbf{0.0035} & \textbf{0.0056} & \textbf{0.0749} & \textbf{0.0846} & \textbf{0.0093} & \textbf{0.0056} & \textbf{0.0081} & \textbf{0.0737} \\ \cmidrule(l){2-12} 
\multirow{3}{*}{0.05}  & Gaussian                & 0.0913          & 0.0081          & \textbf{0.0034} & 0.0074          & \textbf{0.0623} & \textbf{0.1055} & 0.0120          & \textbf{0.0059} & 0.0107          & \textbf{0.0748} \\
                       & Uniform                 & 0.0718          & \textbf{0.0066} & \textbf{0.0036} & \textbf{0.0060} & \textbf{0.0582} & \textbf{0.0746} & \textbf{0.0094} & \textbf{0.0055} & 0.0090          & \textbf{0.0878} \\
                       & Epanechnikov            & 0.1112          & \textbf{0.0063} & \textbf{0.0033} & \textbf{0.0054} & \textbf{0.0744} & \textbf{0.0801} & \textbf{0.0093} & \textbf{0.0056} & 0.0096          & \textbf{0.0607} \\ \cmidrule(l){2-12} 
\multirow{3}{*}{0.1}   & Gaussian                & 0.1153          & 0.0183          & \textbf{0.0035} & 0.0174          & \textbf{0.1076} & \textbf{0.1593} & 0.0215          & \textbf{0.0057} & 0.0203          & \textbf{0.1167} \\
                       & Uniform                 & 0.0687          & 0.0101          & \textbf{0.0037} & 0.0090          & \textbf{0.0758} & \textbf{0.1179} & 0.0135          & \textbf{0.0060} & 0.0130          & \textbf{0.0767} \\
                       & Epanechnikov            & 0.0690          & 0.0085          & \textbf{0.0036} & 0.0071          & \textbf{0.0681} & \textbf{0.1006} & 0.0122          & \textbf{0.0054} & 0.0107          & \textbf{0.0872} \\ \midrule
\multirow{2}{*}{$h$}   & \multirow{2}{*}{Method} & \multicolumn{5}{c|}{Scenario 2,   Model A}                                              & \multicolumn{5}{c}{Scenario 2, Model B}                                                 \\ \cmidrule(l){3-12} 
                       &                         & $\tau$=0.05     & $\tau$=0.25     & $\tau$=0.5      & $\tau$=0.75     & $\tau$=0.95     & $\tau$=0.05     & $\tau$=0.25     & $\tau$=0.5      & $\tau$=0.75     & $\tau$=0.95     \\ \midrule
                       & Baseline                & 0.1704          & 0.0205          & 0.0145          & 0.0223          & 0.1670          & 0.1323          & 0.0202          & 0.0129          & 0.0220          & 0.1358          \\ \cmidrule(l){2-12} 
\multirow{3}{*}{0.001} & Gaussian                & \textbf{0.1316} & \textbf{0.0176} & \textbf{0.0128} & \textbf{0.0193} & \textbf{0.1537} & \textbf{0.1085} & \textbf{0.0160} & \textbf{0.0119} & \textbf{0.0178} & \textbf{0.1144} \\
                       & Uniform                 & \textbf{0.1457} & \textbf{0.0180} & \textbf{0.0128} & \textbf{0.0191} & \textbf{0.1467} & \textbf{0.1172} & \textbf{0.0151} & \textbf{0.0118} & \textbf{0.0175} & \textbf{0.1342} \\
                       & Epanechnikov            & \textbf{0.1430} & \textbf{0.0162} & \textbf{0.0126} & \textbf{0.0191} & \textbf{0.1388} & \textbf{0.1038} & \textbf{0.0146} & \textbf{0.0107} & \textbf{0.0173} & \textbf{0.1302} \\ \cmidrule(l){2-12} 
\multirow{3}{*}{0.005} & Gaussian                & \textbf{0.1428} & \textbf{0.0184} & \textbf{0.0134} & \textbf{0.0182} & \textbf{0.1371} & \textbf{0.1029} & \textbf{0.0148} & \textbf{0.0118} & \textbf{0.0164} & \textbf{0.1264} \\
                       & Uniform                 & \textbf{0.1457} & \textbf{0.0184} & \textbf{0.0128} & \textbf{0.0185} & \textbf{0.1423} & \textbf{0.1250} & \textbf{0.0151} & \textbf{0.0115} & \textbf{0.0171} & \textbf{0.1275} \\
                       & Epanechnikov            & \textbf{0.1551} & \textbf{0.0167} & \textbf{0.0121} & \textbf{0.0185} & \textbf{0.1262} & \textbf{0.1262} & \textbf{0.0186} & 0.0129          & \textbf{0.0172} & \textbf{0.1326} \\ \cmidrule(l){2-12} 
\multirow{3}{*}{0.01}  & Gaussian                & \textbf{0.1315} & \textbf{0.0165} & \textbf{0.0123} & \textbf{0.0172} & \textbf{0.1413} & \textbf{0.1166} & \textbf{0.0154} & \textbf{0.0112} & \textbf{0.0177} & \textbf{0.1117} \\
                       & Uniform                 & \textbf{0.1590} & \textbf{0.0169} & \textbf{0.0109} & \textbf{0.0180} & \textbf{0.1493} & \textbf{0.1079} & \textbf{0.0128} & \textbf{0.0108} & \textbf{0.0187} & \textbf{0.1142} \\
                       & Epanechnikov            & \textbf{0.1364} & \textbf{0.0184} & \textbf{0.0118} & \textbf{0.0197} & \textbf{0.1366} & \textbf{0.1118} & \textbf{0.0152} & \textbf{0.0111} & \textbf{0.0184} & \textbf{0.1119} \\ \cmidrule(l){2-12} 
\multirow{3}{*}{0.05}  & Gaussian                & \textbf{0.1478} & \textbf{0.0178} & \textbf{0.0122} & \textbf{0.0191} & \textbf{0.1553} & \textbf{0.1151} & \textbf{0.0146} & \textbf{0.0100} & \textbf{0.0193} & \textbf{0.1333} \\
                       & Uniform                 & \textbf{0.1328} & \textbf{0.0171} & \textbf{0.0131} & \textbf{0.0195} & \textbf{0.1400} & \textbf{0.1302} & \textbf{0.0135} & \textbf{0.0120} & \textbf{0.0174} & \textbf{0.1071} \\
                       & Epanechnikov            & \textbf{0.1650} & \textbf{0.0162} & \textbf{0.0125} & \textbf{0.0202} & \textbf{0.1363} & \textbf{0.1093} & \textbf{0.0167} & \textbf{0.0109} & \textbf{0.0212} & \textbf{0.1247} \\ \cmidrule(l){2-12} 
\multirow{3}{*}{0.1}   & Gaussian                & \textbf{0.1558} & \textbf{0.0179} & \textbf{0.0121} & \textbf{0.0193} & \textbf{0.1341} & \textbf{0.0980} & \textbf{0.0157} & \textbf{0.0106} & \textbf{0.0190} & \textbf{0.1256} \\
                       & Uniform                 & \textbf{0.1474} & \textbf{0.0176} & \textbf{0.0121} & \textbf{0.0188} & \textbf{0.1347} & \textbf{0.1198} & \textbf{0.0153} & \textbf{0.0126} & \textbf{0.0192} & \textbf{0.1169} \\
                       & Epanechnikov            & \textbf{0.1533} & \textbf{0.0181} & \textbf{0.0120} & \textbf{0.0196} & \textbf{0.1421} & \textbf{0.1170} & \textbf{0.0162} & \textbf{0.0118} & \textbf{0.0201} & \textbf{0.1302} \\ \midrule
\multirow{2}{*}{$h$}   & \multirow{2}{*}{Method} & \multicolumn{5}{c|}{Scenario 3,   Model A}                                              & \multicolumn{5}{c}{Scenario 3, Model B}                                                 \\ \cmidrule(l){3-12} 
                       &                         & $\tau$=0.05     & $\tau$=0.25     & $\tau$=0.5      & $\tau$=0.75     & $\tau$=0.95     & $\tau$=0.05     & $\tau$=0.25     & $\tau$=0.5      & $\tau$=0.75     & $\tau$=0.95     \\ \midrule
                       & Baseline                & 0.9232          & 0.2420          & 0.1459          & 0.2413          & 0.9960          & 0.7786          & 0.2306          & 0.1391          & 0.2380          & 0.7249          \\ \cmidrule(l){2-12} 
\multirow{3}{*}{0.001} & Gaussian                & \textbf{0.8395} & \textbf{0.2040} & \textbf{0.1295} & \textbf{0.2145} & \textbf{0.7477} & 0.8256          & 0.2518          & 0.1416          & 0.2444          & 0.7536          \\
                       & Uniform                 & \textbf{0.8790} & \textbf{0.2064} & \textbf{0.1300} & \textbf{0.2146} & \textbf{0.7824} & \textbf{0.7167} & 0.2409          & \textbf{0.1372} & 0.2493          & \textbf{0.6920} \\
                       & Epanechnikov            & \textbf{0.8230} & \textbf{0.2129} & \textbf{0.1349} & \textbf{0.2206} & \textbf{0.7998} & \textbf{0.7502} & 0.2454          & \textbf{0.1390} & \textbf{0.2315} & \textbf{0.6683} \\ \cmidrule(l){2-12} 
\multirow{3}{*}{0.005} & Gaussian                & \textbf{0.8707} & \textbf{0.2077} & \textbf{0.1308} & \textbf{0.2263} & \textbf{0.8742} & \textbf{0.7740} & 0.2385          & 0.1413          & 0.2528          & \textbf{0.6860} \\
                       & Uniform                 & \textbf{0.8739} & \textbf{0.2021} & \textbf{0.1276} & \textbf{0.2249} & \textbf{0.8332} & \textbf{0.7142} & \textbf{0.2279} & \textbf{0.1374} & \textbf{0.2292} & 0.7574          \\
                       & Epanechnikov            & \textbf{0.8138} & \textbf{0.2084} & \textbf{0.1307} & \textbf{0.2166} & \textbf{0.8117} & \textbf{0.7051} & 0.2376          & \textbf{0.1344} & 0.2529          & 0.7390          \\ \cmidrule(l){2-12} 
\multirow{3}{*}{0.01}  & Gaussian                & \textbf{0.8995} & \textbf{0.1998} & \textbf{0.1275} & \textbf{0.2313} & \textbf{0.7526} & \textbf{0.7052} & 0.2473          & \textbf{0.1362} & 0.2381          & \textbf{0.7225} \\
                       & Uniform                 & \textbf{0.8739} & \textbf{0.2058} & \textbf{0.1289} & \textbf{0.2074} & \textbf{0.8256} & \textbf{0.7404} & 0.2411          & \textbf{0.1384} & 0.2390          & 0.7370          \\
                       & Epanechnikov            & \textbf{0.8079} & \textbf{0.2167} & \textbf{0.1317} & \textbf{0.2294} & \textbf{0.7893} & 0.8089          & 0.2330          & 0.1393          & 0.2497          & \textbf{0.6786} \\ \cmidrule(l){2-12} 
\multirow{3}{*}{0.05}  & Gaussian                & \textbf{0.8527} & \textbf{0.2083} & \textbf{0.1356} & \textbf{0.2211} & \textbf{0.7853} & \textbf{0.7434} & 0.2357          & \textbf{0.1332} & \textbf{0.2350} & \textbf{0.6987} \\
                       & Uniform                 & \textbf{0.8250} & \textbf{0.2164} & \textbf{0.1352} & \textbf{0.2280} & \textbf{0.7966} & \textbf{0.7513} & 0.2485          & \textbf{0.1349} & 0.2485          & \textbf{0.6703} \\
                       & Epanechnikov            & \textbf{0.7879} & \textbf{0.2079} & \textbf{0.1310} & \textbf{0.2128} & \textbf{0.7526} & \textbf{0.6812} & 0.2372          & \textbf{0.1379} & 0.2439          & \textbf{0.7231} \\ \cmidrule(l){2-12} 
\multirow{3}{*}{0.1}   & Gaussian                & \textbf{0.8220} & \textbf{0.2125} & \textbf{0.1348} & \textbf{0.2114} & \textbf{0.8293} & \textbf{0.7644} & \textbf{0.2302} & \textbf{0.1326} & 0.2467          & \textbf{0.6604} \\
                       & Uniform                 & \textbf{0.8100} & \textbf{0.2118} & \textbf{0.1371} & \textbf{0.2253} & \textbf{0.8218} & \textbf{0.7501} & \textbf{0.2295} & \textbf{0.1391} & 0.2393          & \textbf{0.7201} \\
                       & Epanechnikov            & \textbf{0.9015} & \textbf{0.2132} & \textbf{0.1310} & \textbf{0.2204} & \textbf{0.7852} & \textbf{0.6904} & 0.2471          & \textbf{0.1302} & \textbf{0.2262} & \textbf{0.6741} \\ \bottomrule
\end{tabular}
\end{adjustbox}
\end{table}

\begin{table}[htpb]
\centering
\caption{MSE performances without stop criterion for scenario 1-3, Model A and B under different sample sizes, quantile levels and smoothing kernels. The MSEs are averaged over 50 independent trials.}
\label{tab:mse nostop}
\begin{adjustbox}{width=\textwidth,center}
\begin{tabular}{@{}ccclllllllllllllll@{}}
\toprule
\multicolumn{2}{c}{\multirow{2}{*}{}}          & \multirow{2}{*}{Method} & \multicolumn{5}{c}{$n$=1000}                                                                                                                                                & \multicolumn{5}{c}{$n$=5000}                                                                                                                                                & \multicolumn{5}{c}{$n$=10000}                                                                                                                                          \\ \cmidrule(l){4-18} 
\multicolumn{2}{c}{}                           &                         & \multicolumn{1}{c}{$\tau$=0.05} & \multicolumn{1}{c}{$\tau$=0.25} & \multicolumn{1}{c}{$\tau$=0.5} & \multicolumn{1}{c}{$\tau$=0.75} & \multicolumn{1}{c|}{$\tau$=0.95}     & \multicolumn{1}{c}{$\tau$=0.05} & \multicolumn{1}{c}{$\tau$=0.25} & \multicolumn{1}{c}{$\tau$=0.5} & \multicolumn{1}{c}{$\tau$=0.75} & \multicolumn{1}{c|}{$\tau$=0.95}     & \multicolumn{1}{c}{$\tau$=0.05} & \multicolumn{1}{c}{$\tau$=0.25} & \multicolumn{1}{c}{$\tau$=0.5} & \multicolumn{1}{c}{$\tau$=0.75} & \multicolumn{1}{c}{$\tau$=0.95} \\ \midrule
\multirow{8}{*}{S1} & \multirow{4}{*}{Model A} & Baseline                & \fbox{0.3820}                          & 0.0402                          & 0.0278                         & 0.0374                          & \multicolumn{1}{l|}{0.3784}          & \fbox{0.0996 }                         & 0.0120                          & 0.0086                         & 0.0108                          & \multicolumn{1}{l|}{\fbox{0.0939}}          & \fbox{0.0618}                          & 0.0067                          & 0.0047                         & 0.0063                          & 0.1366                          \\ \cmidrule(l){3-18} 
                    &                          & Gaussian                & 0.4271                          & \textbf{0.0379}                 & \textbf{0.0224}                & \textbf{0.0374}                 & \multicolumn{1}{l|}{0.5027}          & 0.1118                          & \textbf{0.0088}                 & \textbf{0.0053}                & \textbf{0.0095}                 & \multicolumn{1}{l|}{0.1081}          & 0.0669                          & \textbf{0.0063}                 & \textbf{0.0033}                & \textbf{0.0054}                 & \textbf{0.0626}                 \\
                    &                          & Uniform                 & 0.4413                          & \textbf{0.0378}                 & \textbf{0.0224}                & \textbf{0.0326}                 & \multicolumn{1}{l|}{\textbf{0.3704}} & 0.1072                          & \textbf{0.0087}                 & \textbf{0.0059}                & \textbf{0.0101}                 & \multicolumn{1}{l|}{0.1312}          & 0.0789                          & \textbf{0.0058}                 & \textbf{0.0035}                & \textbf{0.0055}                 & \textbf{0.0530}                 \\
                    &                          & Epanechnikov            & 0.6718                          & \textbf{0.0390}                 & \textbf{0.0218}                & \textbf{0.0371}                 & \multicolumn{1}{l|}{0.5913}          & 0.1249                          & \textbf{0.0091}                 & \textbf{0.0055}                & \textbf{0.0095}                 & \multicolumn{1}{l|}{0.1153}          & 0.0648                          & \textbf{0.0060}                 & \textbf{0.0037}                & \textbf{0.0053}                 & \textbf{0.0662}                 \\ \cmidrule(l){3-18} 
                    & \multirow{4}{*}{Model B} & Baseline                & \fbox{0.3842}                          & \fbox{0.0527}                          & \fbox{0.0319}                         & \fbox{0.0475}                          & \multicolumn{1}{l|}{\fbox{0.4222}}          & 0.1143                          & 0.0149                          & 0.0107                         & 0.0151                          & \multicolumn{1}{l|}{0.1277}          & 0.1691                          & 0.0099                          & 0.0066                         & 0.0082                          & 0.3882                          \\ \cmidrule(l){3-18} 
                    &                          & Gaussian                & 0.4158                          & 0.0666                          & 0.0383                         & 0.0632                          & \multicolumn{1}{l|}{0.5080}          & 0.1172                          & \textbf{0.0143}                 & \textbf{0.0098}                & \textbf{0.0141}                 & \multicolumn{1}{l|}{\textbf{0.1067}} & \textbf{0.1063}                 & \textbf{0.0094}                 & \textbf{0.0055}                & 0.0084                          & \textbf{0.0729}                 \\
                    &                          & Uniform                 & 0.5645                          & 0.0610                          & 0.0376                         & 0.0616                          & \multicolumn{1}{l|}{0.5685}          & \textbf{0.1034}                 & \textbf{0.0143}                 & \textbf{0.0091}                & 0.0153                          & \multicolumn{1}{l|}{\textbf{0.1255}} & \textbf{0.0975}                 & \textbf{0.0093}                 & \textbf{0.0055}                & \textbf{0.0080}                 & \textbf{0.0738}                 \\
                    &                          & Epanechnikov            & 0.4247                          & 0.0576                          & 0.0383                         & 0.0623                          & \multicolumn{1}{l|}{0.6050}          & \textbf{0.1109}                 & \textbf{0.0144}                 & \textbf{0.0092}                & \textbf{0.0146}                 & \multicolumn{1}{l|}{\textbf{0.1164}} & \textbf{0.1291}                 & \textbf{0.0089}                 & \textbf{0.0058}                & 0.0092                          & \textbf{0.0662}                 \\ \midrule
\multirow{8}{*}{S2} & \multirow{4}{*}{Model A} & Baseline                & 0.8292                          & 0.0868                          & 0.0619                         & 0.0839                          & \multicolumn{1}{l|}{\fbox{0.7874}}          & 0.2752                          & \fbox{0.0275}                          & 0.0222                         & 0.0308                          & \multicolumn{1}{l|}{0.2747}          & 0.1704                          & 0.0205                          & 0.0145                         & 0.0223                          & 0.1670                          \\ \cmidrule(l){3-18} 
                    &                          & Gaussian                & \textbf{0.7994}                 & \textbf{0.0711}                 & \textbf{0.0587}                & \textbf{0.0778}                 & \multicolumn{1}{l|}{1.1450}          & \textbf{0.2202}                 & 0.0276                          & \textbf{0.0169}                & \textbf{0.0273}                 & \multicolumn{1}{l|}{\textbf{0.2598}} & \textbf{0.1316}                 & \textbf{0.0176}                 & \textbf{0.0128}                & \textbf{0.0193}                 & \textbf{0.1537}                 \\
                    &                          & Uniform                 & 0.8890                          & \textbf{0.0722}                 & \textbf{0.0472}                & 0.0968                          & \multicolumn{1}{l|}{0.8566}          & \textbf{0.2320}                 & 0.0286                          & \textbf{0.0181}                & \textbf{0.0275}                 & \multicolumn{1}{l|}{\textbf{0.2586}} & \textbf{0.1457}                 & \textbf{0.0180}                 & \textbf{0.0128}                & \textbf{0.0191}                 & \textbf{0.1467}                 \\
                    &                          & Epanechnikov            & 0.9192                          & \textbf{0.0787}                 & \textbf{0.0522}                & 0.1048                          & \multicolumn{1}{l|}{0.9065}          & \textbf{0.2415}                 & 0.0284                          & \textbf{0.0177}                & \textbf{0.0265}                 & \multicolumn{1}{l|}{\textbf{0.2492}} & \textbf{0.1430}                 & \textbf{0.0162}                 & \textbf{0.0125}                & \textbf{0.0191}                 & \textbf{0.1388}                 \\ \cmidrule(l){3-18} 
                    & \multirow{4}{*}{Model B} & Baseline                & 0.4930                          & \fbox{0.0583}                          & \fbox{0.0493}                         & 0.0840                          & \multicolumn{1}{l|}{\fbox{0.5898}}          & 0.1732                          & 0.0257                          & 0.0178                         & 0.0300                          & \multicolumn{1}{l|}{0.1966}          & 0.1323                          & 0.0202                          & 0.0129                         & 0.0220                          & 0.1358                          \\ \cmidrule(l){3-18} 
                    &                          & Gaussian                & 0.7345                          & 0.0633                          & 0.0499                         & \textbf{0.0759}                 & \multicolumn{1}{l|}{0.6273}          & 0.1800                          & \textbf{0.0246}                 & \textbf{0.0155}                & \textbf{0.0245}                 & \multicolumn{1}{l|}{0.2157}          & \textbf{0.1085}                 & \textbf{0.0160}                 & \textbf{0.0119}                & \textbf{0.0178}                 & \textbf{0.1144}                 \\
                    &                          & Uniform                 & \textbf{0.4506}                 & 0.0716                          & 0.0504                         & 0.0935                          & \multicolumn{1}{l|}{0.6030}          & \textbf{0.1707}                 & \textbf{0.0216}                 & \textbf{0.0176}                & \textbf{0.0241}                 & \multicolumn{1}{l|}{0.2243}          & \textbf{0.1172}                 & \textbf{0.0151}                 & \textbf{0.0118}                & \textbf{0.0175}                 & \textbf{0.1342}                 \\
                    &                          & Epanechnikov            & 0.5800                          & 0.0723                          & 0.0545                         & \textbf{0.0806}                 & \multicolumn{1}{l|}{0.7471}          & 0.2151                          & 0.0263                          & \textbf{0.0159}                & \textbf{0.0289}                 & \multicolumn{1}{l|}{\textbf{0.1792}} & \textbf{0.1038}                 & \textbf{0.0146}                 & \textbf{0.0107}                & \textbf{0.0173}                 & \textbf{0.1302}                 \\ \midrule
\multirow{8}{*}{S3} & \multirow{4}{*}{Model A} & Baseline                & 3.0766                          & 0.7564                          & 0.5350                         & 0.7407                          & \multicolumn{1}{l|}{\fbox{2.8969}}          & 1.2870                          & 0.3854                          & 0.2146                         & 0.3387                          & \multicolumn{1}{l|}{1.3495}          & 0.9232                          & 0.2420                          & 0.1459                         & 0.2413                          & 0.9960                          \\ \cmidrule(l){3-18} 
                    &                          & Gaussian                & \textbf{2.9080}                 & 0.7769                          & \textbf{0.4784}                & \textbf{0.7056}                 & \multicolumn{1}{l|}{3.4181}          & 1.3139                          & \textbf{0.3379}                 & \textbf{0.1994}                & \textbf{0.3229}                 & \multicolumn{1}{l|}{\textbf{1.1897}} & \textbf{0.8395}                 & \textbf{0.2038}                 & \textbf{0.1295}                & \textbf{0.2145}                 & \textbf{0.7477}                 \\
                    &                          & Uniform                 & 3.3730                          & 0.7690                          & \textbf{0.4876}                & 0.7832                          & \multicolumn{1}{l|}{3.3576}          & \textbf{1.1593}                 & \textbf{0.3449}                 & \textbf{0.1912}                & \textbf{0.3088}                 & \multicolumn{1}{l|}{\textbf{1.2392}} & \textbf{0.8790}                 & \textbf{0.2064}                 & \textbf{0.1300}                & \textbf{0.2146}                 & \textbf{0.7824}                 \\
                    &                          & Epanechnikov            & 3.4479                          & \textbf{0.7308}                 & \textbf{0.4679}                & 0.8076                          & \multicolumn{1}{l|}{3.3156}          & \textbf{1.1229}                 & \textbf{0.3616}                 & \textbf{0.1980}                & \textbf{0.3099}                 & \multicolumn{1}{l|}{\textbf{1.3312}} & \textbf{0.8230}                 & \textbf{0.2123}                 & \textbf{0.1337}                & \textbf{0.2206}                 & \textbf{0.7998}                 \\ \cmidrule(l){3-18} 
                    & \multirow{4}{*}{Model B} & Baseline                & 2.8166                          & \fbox{0.7558 }                         & 0.5175                         & 0.7665                          & \multicolumn{1}{l|}{2.2839}          & 1.0061                          & \fbox{0.3596}                          & \fbox{0.2196}                         & 0.3551                          & \multicolumn{1}{l|}{1.0990}          & 0.7786                          & \fbox{0.2306 }                         & 0.1391                         & 0.2380                          & 0.7249                          \\ \cmidrule(l){3-18} 
                    &                          & Gaussian                & \textbf{2.3173}                 & 0.8406                          & \textbf{0.5142}                & \textbf{0.7536}                 & \multicolumn{1}{l|}{2.6542}          & 1.0602                          & 0.4255                          & 0.2302                         & \textbf{0.3526}                 & \multicolumn{1}{l|}{\textbf{1.0681}} & 0.8244                          & 0.2516                          & 0.1416                         & 0.2443                          & 0.7536                          \\
                    &                          & Uniform                 & \textbf{2.1946}                 & 0.8319                          & 0.5193                         & \textbf{0.7352}                 & \multicolumn{1}{l|}{2.5765}          & 1.2056                          & 0.4038                          & 0.2292                         & \textbf{0.3372}                 & \multicolumn{1}{l|}{\textbf{1.0528}} & \textbf{0.7167}                 & 0.2411                          & \textbf{0.1359}                & 0.2493                          & \textbf{0.6890}                 \\
                    &                          & Epanechnikov            & 2.9684                          & 0.9008                          & \textbf{0.4892}                & 0.8704                          & \multicolumn{1}{l|}{\textbf{2.1302}} & \textbf{1.0053}                 & 0.3741                          & 0.2297                         & \textbf{0.3473}                 & \multicolumn{1}{l|}{\textbf{1.0896}} & \textbf{0.7474}                 & 0.2453                          & \textbf{0.1390}                & \textbf{0.2315}                 & \textbf{0.6683}                 \\ \bottomrule
\end{tabular}
\end{adjustbox}
\end{table}

{We also study the effect of residual-based neural networks. We add one residual block for every hidden layer in the baseline and ConquerNet to study the MSE performance. We compare the baseline model performance without residual blocks to that using residual-based structures. The results are shown in Table \ref{tab:mse-single-nores-vs-res-baseline}. The numbers with brackets in the table represent the standard error of the MSEs over the 50 trials, and the bold font means that the model has a smaller MSE. We find from the results that the residual-based neural networks have no advantage in reducing MSE in many scenarios. 
We also make the comparison for ConquerNet, as shown in Table \ref{tab:mse-single-nores-vs-res-gaussian} for the Gaussian kernel, for example. Residual blocks also fail to significantly reduce the MSE in the Gaussian ConquerNet. However, we find that for high ($\tau=0.95$) or low ($\tau=0.05$) quantile levels, the models with residual structures tend to have smaller MSEs and standard errors, which increases the training stability. We provide the table of the MSE performances for baseline models and ConquerNet with residual blocks, see Table \ref{tab:mse-single-res}. We can see that for residual-based networks, the ConquerNet still have better performance than the baseline models.}

\begin{table}[htpb]
\centering
\caption{Mean squared error (MSE) and its standard error performances for baseline, scenario 1-3, model A and B with residual blocks under different sample sizes, quantile levels. ``Original'' means the original baseline network without residual blocks, and ``+Res'' means the baseline network with residual blocks. The MSEs are averaged over 50 independent trials. The bracket means the standard error of the MSEs over the trials.}
\label{tab:mse-single-nores-vs-res-baseline}
\begin{adjustbox}{width=\textwidth,center}
\begin{tabular}{@{}cccccccccccccccccc@{}}
\toprule
\multicolumn{3}{c}{\multirow{2}{*}{Baseline}}                              & \multicolumn{5}{c}{$n$=1000}                                                                                 & \multicolumn{5}{c}{$n$=5000}                                                                                 & \multicolumn{5}{c}{$n$=10000}                                                           \\ \cmidrule(l){4-18} 
\multicolumn{3}{c}{}                                                       & $\tau$=0.05     & $\tau$=0.25     & $\tau$=0.5      & $\tau$=0.75     & \multicolumn{1}{c|}{$\tau$=0.95}     & $\tau$=0.05     & $\tau$=0.25     & $\tau$=0.5      & $\tau$=0.75     & \multicolumn{1}{c|}{$\tau$=0.95}     & $\tau$=0.05     & $\tau$=0.25     & $\tau$=0.5      & $\tau$=0.75     & $\tau$=0.95     \\ \midrule
\multirow{8}{*}{S1} & \multirow{4}{*}{Model A} & \multirow{2}{*}{Original} & 0.3820          & \textbf{0.0402} & \textbf{0.0278} & \textbf{0.0374} & \multicolumn{1}{c|}{0.3784}          & \textbf{0.0996} & \textbf{0.0120} & \textbf{0.0086} & \textbf{0.0108} & \multicolumn{1}{c|}{0.0939}          & \textbf{0.0618} & \textbf{0.0067} & \textbf{0.0047} & \textbf{0.0063} & 0.1366          \\
                    &                          &                           & (0.0916)        & (0.0025)        & (0.0019)        & (0.0024)        & \multicolumn{1}{c|}{(0.0524)}        & (0.0172)        & (0.0005)        & (0.0007)        & (0.0005)        & \multicolumn{1}{c|}{(0.0207)}        & (0.0048)        & (0.0003)        & (0.0002)        & (0.0004)        & (0.0585)        \\
                    &                          & \multirow{2}{*}{+ Res}    & \textbf{0.3316} & 0.0575          & 0.0324          & 0.0376          & \multicolumn{1}{c|}{\textbf{0.3508}} & 0.1051          & 0.0221          & 0.0120          & 0.0121          & \multicolumn{1}{c|}{\textbf{0.0797}} & 0.0661          & 0.0150          & 0.0088          & 0.0076          & \textbf{0.0596} \\
                    &                          &                           & (0.0343)        & (0.0029)        & (0.0018)        & (0.0025)        & \multicolumn{1}{c|}{(0.0729)}        & (0.0071)        & (0.0010)        & (0.0006)        & (0.0008)        & \multicolumn{1}{c|}{(0.0057)}        & (0.0037)        & (0.0006)        & (0.0004)        & (0.0005)        & (0.0070)        \\ \cmidrule(l){4-18} 
                    & \multirow{4}{*}{Model B} & \multirow{2}{*}{Original} & 0.3842          & \textbf{0.0527} & \textbf{0.0319} & 0.0475          & \multicolumn{1}{c|}{0.4222}          & \textbf{0.1143} & 0.0149          & \textbf{0.0107} & \textbf{0.0151} & \multicolumn{1}{c|}{0.1277}          & 0.1691          & \textbf{0.0099} & \textbf{0.0066} & \textbf{0.0082} & \textbf{0.3882} \\
                    &                          &                           & (0.0705)        & (0.0036)        & (0.0019)        & (0.0032)        & \multicolumn{1}{c|}{(0.0709)}        & (0.0160)        & (0.0006)        & (0.0004)        & (0.0008)        & \multicolumn{1}{c|}{(0.0234)}        & (0.0656)        & (0.0006)        & (0.0003)        & (0.0004)        & (0.3125)        \\
                    &                          & \multirow{2}{*}{+ Res}    & \textbf{0.3457} & 0.0716          & 0.0373          & \textbf{0.0424} & \multicolumn{1}{c|}{\textbf{0.3231}} & 0.1249          & 0.0293          & 0.0191          & 0.0162          & \multicolumn{1}{c|}{\textbf{0.0620}} & \textbf{0.0783} & 0.0200          & 0.0121          & 0.0118          & 0.0455          \\
                    &                          &                           & (0.0365)        & (0.0040)        & (0.0017)        & (0.0031)        & \multicolumn{1}{c|}{(0.0515)}        & (0.0081)        & (0.0011)        & (0.0008)        & (0.0011)        & \multicolumn{1}{c|}{(0.0048)}        & (0.0057)        & (0.0007)        & (0.0005)        & (0.0007)        & (0.0037)        \\ \midrule
\multirow{8}{*}{S2} & \multirow{4}{*}{Model A} & \multirow{2}{*}{Original} & 0.8292          & 0.0868          & \textbf{0.0619} & \textbf{0.0839} & \multicolumn{1}{c|}{0.7874}          & 0.2752          & \textbf{0.0275} & \textbf{0.0222} & \textbf{0.0308} & \multicolumn{1}{c|}{0.2747}          & 0.1704          & \textbf{0.0205} & \textbf{0.0145} & \textbf{0.0223} & \textbf{0.1670} \\
                    &                          &                           & (0.1170)        & (0.0162)        & (0.0046)        & (0.0061)        & \multicolumn{1}{c|}{(0.1149)}        & (0.0264)        & (0.0021)        & (0.0016)        & (0.0020)        & \multicolumn{1}{c|}{(0.0216)}        & (0.0121)        & (0.0013)        & (0.0006)        & (0.0013)        & (0.0076)        \\
                    &                          & \multirow{2}{*}{+ Res}    & \textbf{0.7892} & \textbf{0.0801} & 0.0655          & 0.1017          & \multicolumn{1}{c|}{\textbf{0.7819}} & \textbf{0.2290} & 0.0315          & 0.0244          & 0.0354          & \multicolumn{1}{c|}{\textbf{0.2551}} & \textbf{0.1285} & 0.0213          & 0.0157          & 0.0231          & 0.1723          \\
                    &                          &                           & (0.1507)        & (0.0061)        & (0.0059)        & (0.0060)        & \multicolumn{1}{c|}{(0.1265)}        & (0.0220)        & (0.0022)        & (0.0017)        & (0.0019)        & \multicolumn{1}{c|}{(0.0164)}        & (0.0106)        & (0.0015)        & (0.0008)        & (0.0012)        & (0.0096)        \\ \cmidrule(l){4-18} 
                    & \multirow{4}{*}{Model B} & \multirow{2}{*}{Original} & \textbf{0.4930} & \textbf{0.0583} & \textbf{0.0493} & \textbf{0.0840} & \multicolumn{1}{c|}{\textbf{0.5898}} & \textbf{0.1732} & \textbf{0.0257} & \textbf{0.0178} & \textbf{0.0300} & \multicolumn{1}{c|}{\textbf{0.1966}} & 0.1323          & 0.0202          & \textbf{0.0129} & 0.0220          & 0.1358          \\
                    &                          &                           & (0.0742)        & (0.0038)        & (0.0043)        & (0.0062)        & \multicolumn{1}{c|}{(0.0712)}        & (0.0114)        & (0.0015)        & (0.0009)        & (0.0018)        & \multicolumn{1}{c|}{(0.0164)}        & (0.0146)        & (0.0017)        & (0.0006)        & (0.0013)        & (0.0070)        \\
                    &                          & \multirow{2}{*}{+ Res}    & 0.5210          & 0.0683          & 0.0651          & 0.0869          & \multicolumn{1}{c|}{0.7145}          & 0.1847          & 0.0261          & 0.0205          & 0.0335          & \multicolumn{1}{c|}{0.2108}          & \textbf{0.1254} & \textbf{0.0174} & 0.0142          & \textbf{0.0204} & \textbf{0.1302} \\
                    &                          &                           & (0.1177)        & (0.0056)        & (0.0084)        & (0.0054)        & \multicolumn{1}{c|}{(0.1019)}        & (0.0214)        & (0.0015)        & (0.0008)        & (0.0023)        & \multicolumn{1}{c|}{(0.0160)}        & (0.0102)        & (0.0010)        & (0.0007)        & (0.0009)        & (0.0090)        \\ \midrule
\multirow{8}{*}{S3} & \multirow{4}{*}{Model A} & \multirow{2}{*}{Original} & \textbf{3.0766} & \textbf{0.7564} & \textbf{0.5350} & \textbf{0.7407} & \multicolumn{1}{c|}{\textbf{2.8969}} & 1.2870          & 0.3854          & \textbf{0.2146} & \textbf{0.3387} & \multicolumn{1}{c|}{1.3495}          & 0.9232          & \textbf{0.2420} & \textbf{0.1459} & \textbf{0.2413} & 0.9960          \\
                    &                          &                           & (0.2728)        & (0.0388)        & (0.0248)        & (0.0330)        & \multicolumn{1}{c|}{(0.2179)}        & (0.0844)        & (0.0216)        & (0.0051)        & (0.0109)        & \multicolumn{1}{c|}{(0.0976)}        & (0.0348)        & (0.0082)        & (0.0028)        & (0.0067)        & (0.0636)        \\
                    &                          & \multirow{2}{*}{+ Res}    & 3.1729          & 0.8820          & 0.5587          & 0.8456          & \multicolumn{1}{c|}{3.0478}          & \textbf{1.1872} & \textbf{0.3741} & 0.2298          & 0.3626          & \multicolumn{1}{c|}{\textbf{1.1419}} & \textbf{0.7821} & 0.2505          & 0.1708          & 0.2539          & \textbf{0.7933} \\
                    &                          &                           & (0.4222)        & (0.0415)        & (0.0139)        & (0.0403)        & \multicolumn{1}{c|}{(0.2550)}        & (0.0536)        & (0.0106)        & (0.0057)        & (0.0107)        & \multicolumn{1}{c|}{(0.0664)}        & (0.0382)        & (0.0066)        & (0.0035)        & (0.0062)        & (0.0389)        \\ \cmidrule(l){4-18} 
                    & \multirow{4}{*}{Model B} & \multirow{2}{*}{Original} & 2.8166          & \textbf{0.7558} & 0.5175          & 0.7665          & \multicolumn{1}{c|}{2.2839}          & 1.0061          & 0.3596          & 0.2196          & 0.3551          & \multicolumn{1}{c|}{1.0990}          & 0.7786          & \textbf{0.2306} & \textbf{0.1391} & \textbf{0.2380} & 0.7249          \\
                    &                          &                           & (0.2967)        & (0.0383)        & (0.0205)        & (0.0509)        & \multicolumn{1}{c|}{(0.2600)}        & (0.0579)        & (0.0131)        & (0.0054)        & (0.0109)        & \multicolumn{1}{c|}{(0.0780)}        & (0.0315)        & (0.0075)        & (0.0036)        & (0.0084)        & (0.0330)        \\
                    &                          & \multirow{2}{*}{+ Res}    & \textbf{2.2566} & 0.7801          & \textbf{0.5042} & \textbf{0.7446} & \multicolumn{1}{c|}{\textbf{2.2770}} & \textbf{0.9813} & \textbf{0.3330} & \textbf{0.2148} & \textbf{0.3146} & \multicolumn{1}{c|}{\textbf{0.9632}} & \textbf{0.7292} & 0.2358          & 0.1554          & 0.2451          & \textbf{0.6422} \\
                    &                          &                           & (0.2357)        & (0.0379)        & (0.0171)        & (0.0301)        & \multicolumn{1}{c|}{(0.1769)}                             & (0.0587)        & (0.0103)        & (0.0043)        & (0.0083)        & \multicolumn{1}{c|}{(0.0474)}                             & (0.0358)        & (0.0068)        & (0.0025)        & (0.0061)        & (0.0267)        \\ \bottomrule
\end{tabular}
\end{adjustbox}
\end{table}

\begin{table}[htpb]
\centering
\caption{Mean squared error (MSE) and its standard error performances for Gaussian ConquerNet, scenario 1-3, model A and B with residual blocks under different sample sizes, quantile levels. ``Original'' means the original Gaussian ConquerNet without residual blocks, and ``+Res'' means the Gaussian ConquerNet with residual blocks. The MSEs are averaged over 50 independent trials. The bracket means the standard error of the MSEs over the trials.}
\label{tab:mse-single-nores-vs-res-gaussian}
\begin{adjustbox}{width=\textwidth,center}
\begin{tabular}{@{}cccccccccccccccccc@{}}
\toprule
\multicolumn{3}{c}{\multirow{2}{*}{Gaussian}}                              & \multicolumn{5}{c}{$n$=1000}                                                                                 & \multicolumn{5}{c}{$n$=5000}                                                                                 & \multicolumn{5}{c}{$n$=10000}                                                           \\ \cmidrule(l){4-18} 
\multicolumn{3}{c}{}                                                       & $\tau$=0.05     & $\tau$=0.25     & $\tau$=0.5      & $\tau$=0.75     & \multicolumn{1}{c|}{$\tau$=0.95}     & $\tau$=0.05     & $\tau$=0.25     & $\tau$=0.5      & $\tau$=0.75     & \multicolumn{1}{c|}{$\tau$=0.95}     & $\tau$=0.05     & $\tau$=0.25     & $\tau$=0.5      & $\tau$=0.75     & $\tau$=0.95     \\ \midrule
\multirow{8}{*}{S1} & \multirow{4}{*}{Model A} & \multirow{2}{*}{Original} & 0.4354          & \textbf{0.0383} & 0.0224          & 0.0382          & \multicolumn{1}{c|}{0.5035}          & 0.1124          & \textbf{0.0087} & \textbf{0.0055} & 0.0097          & \multicolumn{1}{c|}{0.1081}          & 0.0678          & \textbf{0.0064} & \textbf{0.0034} & \textbf{0.0055} & 0.0625          \\
                    &                          &                           & (0.0767)        & (0.0025)        & (0.0014)        & (0.0030)        & \multicolumn{1}{c|}{(0.0858)}        & (0.0120)        & (0.0004)        & (0.0003)        & (0.0007)        & \multicolumn{1}{c|}{(0.0148)}        & (0.0088)        & (0.0005)        & (0.0001)        & (0.0004)        & (0.0060)        \\
                    &                          & \multirow{2}{*}{+ Res}    & \textbf{0.3787} & 0.0620          & \textbf{0.0217} & \textbf{0.0262} & \multicolumn{1}{c|}{\textbf{0.4817}} & \textbf{0.1006} & 0.0249          & 0.0118          & \textbf{0.0090} & \multicolumn{1}{c|}{\textbf{0.0889}} & \textbf{0.0641} & 0.0245          & 0.0122          & 0.0100          & \textbf{0.0465} \\
                    &                          &                           & (0.0573)        & (0.0030)        & (0.0014)        & (0.0021)        & \multicolumn{1}{c|}{(0.0825)}        & (0.0058)        & (0.0011)        & (0.0006)        & (0.0006)        & \multicolumn{1}{c|}{(0.0105)}        & (0.0037)        & (0.0009)        & (0.0006)        & (0.0008)        & (0.0044)        \\ \cmidrule(l){4-18} 
                    & \multirow{4}{*}{Model B} & \multirow{2}{*}{Original} & 0.4158          & \textbf{0.0665} & 0.0383          & 0.0633          & \multicolumn{1}{c|}{0.5202}          & \textbf{0.1172} & \textbf{0.0144} & \textbf{0.0097} & 0.0142          & \multicolumn{1}{c|}{0.1066}          & 0.1062          & \textbf{0.0095} & \textbf{0.0055} & \textbf{0.0086} & 0.0726          \\
                    &                          &                           & (0.0341)        & (0.0043)        & (0.0014)        & (0.0024)        & \multicolumn{1}{c|}{(0.0427)}        & (0.0113)        & (0.0006)        & (0.0004)        & (0.0005)        & \multicolumn{1}{c|}{(0.0122)}        & (0.0307)        & (0.0005)        & (0.0003)        & (0.0005)        & (0.0124)        \\
                    &                          & \multirow{2}{*}{+ Res}    & \textbf{0.3525} & 0.0797          & \textbf{0.0306} & \textbf{0.0308} & \multicolumn{1}{c|}{\textbf{0.5173}} & 0.1360          & 0.0386          & 0.0187          & \textbf{0.0131} & \multicolumn{1}{c|}{\textbf{0.0709}} & \textbf{0.0829} & 0.0276          & 0.0156          & 0.0110          & \textbf{0.0443} \\
                    &                          &                           & (0.0380)        & (0.0036)        & (0.0018)        & (0.0021)        & \multicolumn{1}{c|}{(0.0915)}        & (0.0094)        & (0.0016)        & (0.0011)        & (0.0009)        & \multicolumn{1}{c|}{(0.0077)}        & (0.0052)        & (0.0011)        & (0.0006)        & (0.0007)        & (0.0044)        \\ \midrule
\multirow{8}{*}{S2} & \multirow{4}{*}{Model A} & \multirow{2}{*}{Original} & 0.7994          & 0.0711          & 0.0587          & \textbf{0.0778} & \multicolumn{1}{c|}{1.1466}          & 0.2202          & 0.0276          & \textbf{0.0169} & \textbf{0.0273} & \multicolumn{1}{c|}{0.2598}          & 0.1316          & 0.0176          & \textbf{0.0128} & 0.0193          & 0.1537          \\
                    &                          &                           & (0.1005)        & (0.0068)        & (0.0088)        & (0.0053)        & \multicolumn{1}{c|}{(0.2029)}        & (0.0174)        & (0.0025)        & (0.0010)        & (0.0017)        & \multicolumn{1}{c|}{(0.0284)}        & (0.0091)        & (0.0011)        & (0.0007)        & (0.0016)        & (0.0102)        \\
                    &                          & \multirow{2}{*}{+ Res}    & \textbf{0.7481} & \textbf{0.0704} & \textbf{0.0577} & 0.0906          & \multicolumn{1}{c|}{\textbf{0.7275}} & \textbf{0.2076} & \textbf{0.0260} & 0.0223          & 0.0308          & \multicolumn{1}{c|}{\textbf{0.2394}} & \textbf{0.1302} & \textbf{0.0152} & 0.0134          & \textbf{0.0187} & \textbf{0.1354} \\
                    &                          &                           & (0.1358)        & (0.0082)        & (0.0043)        & (0.0064)        & \multicolumn{1}{c|}{(0.0984)}        & (0.0219)        & (0.0015)        & (0.0028)        & (0.0017)        & \multicolumn{1}{c|}{(0.0315)}        & (0.0075)        & (0.0009)        & (0.0006)        & (0.0010)        & (0.0092)        \\ \cmidrule(l){4-18} 
                    & \multirow{4}{*}{Model B} & \multirow{2}{*}{Original} & 0.7367          & \textbf{0.0639} & \textbf{0.0500} & \textbf{0.0759} & \multicolumn{1}{c|}{0.6273}          & 0.1800          & 0.0246          & \textbf{0.0155} & \textbf{0.0245} & \multicolumn{1}{c|}{0.2157}          & \textbf{0.1085} & 0.0160          & \textbf{0.0119} & \textbf{0.0178} & 0.1144          \\
                    &                          &                           & (0.1652)        & (0.0057)        & (0.0049)        & (0.0053)        & \multicolumn{1}{c|}{(0.0972)}        & (0.0171)        & (0.0020)        & (0.0009)        & (0.0024)        & \multicolumn{1}{c|}{(0.0236)}        & (0.0087)        & (0.0017)        & (0.0007)        & (0.0013)        & (0.0078)        \\
                    &                          & \multirow{2}{*}{+ Res}    & \textbf{0.7154} & 0.0700          & 0.0553          & 0.0827          & \multicolumn{1}{c|}{\textbf{0.5483}} & \textbf{0.1430} & \textbf{0.0237} & 0.0195          & 0.0284          & \multicolumn{1}{c|}{\textbf{0.1835}} & 0.1008          & \textbf{0.0152} & 0.0120          & 0.0187          & \textbf{0.1036} \\
                    &                          &                           & (0.1862)        & (0.0083)        & (0.0043)        & (0.0053)        & \multicolumn{1}{c|}{(0.1012)}        & (0.0134)        & (0.0011)        & (0.0013)        & (0.0014)        & \multicolumn{1}{c|}{(0.0138)}        & (0.0093)        & (0.0010)        & (0.0006)        & (0.0011)        & (0.0070)        \\ \midrule
\multirow{8}{*}{S3} & \multirow{4}{*}{Model A} & \multirow{2}{*}{Original} & 2.8841          & \textbf{0.7776} & \textbf{0.4788} & \textbf{0.7056} & \multicolumn{1}{c|}{3.4222}          & 1.3139          & 0.3379          & \textbf{0.1993} & 0.3269          & \multicolumn{1}{c|}{1.1897}          & 0.8395          & \textbf{0.2040} & \textbf{0.1295} & \textbf{0.2145} & 0.7477          \\
                    &                          &                           & (0.3062)        & (0.0746)        & (0.0206)        & (0.0500)        & \multicolumn{1}{c|}{(0.3403)}        & (0.1159)        & (0.0135)        & (0.0067)        & (0.0100)        & \multicolumn{1}{c|}{(0.0783)}        & (0.0565)        & (0.0055)        & (0.0027)        & (0.0055)        & (0.0402)        \\
                    &                          & \multirow{2}{*}{+ Res}    & \textbf{2.8543} & 0.9541          & 0.5828          & 0.7702          & \multicolumn{1}{c|}{\textbf{2.7395}} & \textbf{1.1222} & \textbf{0.3339} & 0.2143          & \textbf{0.3045} & \multicolumn{1}{c|}{\textbf{0.9996}} & \textbf{0.7556} & 0.2403          & 0.1515          & 0.2281          & \textbf{0.6931} \\
                    &                          &                           & (0.2822)        & (0.0714)        & (0.0455)        & (0.0347)        & \multicolumn{1}{c|}{(0.2494)}        & (0.0464)        & (0.0117)        & (0.0045)        & (0.0084)        & \multicolumn{1}{c|}{(0.0518)}        & (0.0360)        & (0.0084)        & (0.0027)        & (0.0056)        & (0.0286)        \\ \cmidrule(l){4-18} 
                    & \multirow{4}{*}{Model B} & \multirow{2}{*}{Original} & \textbf{2.3193} & 0.8405          & \textbf{0.5142} & \textbf{0.7543} & \multicolumn{1}{c|}{2.6520}          & 1.0602          & 0.4263          & 0.2304          & 0.3525          & \multicolumn{1}{c|}{1.0681}          & 0.8256          & 0.2518          & \textbf{0.1416} & 0.2444          & 0.7536          \\
                    &                          &                           & (0.2530)        & (0.0573)        & (0.0182)        & (0.0441)        & \multicolumn{1}{c|}{(0.5055)}        & (0.0791)        & (0.0202)        & (0.0070)        & (0.0122)        & \multicolumn{1}{c|}{(0.0822)}        & (0.0460)        & (0.0132)        & (0.0047)        & (0.0085)        & (0.0450)        \\
                    &                          & \multirow{2}{*}{+ Res}    & 2.5481          & \textbf{0.7569} & 0.5322          & 0.8249          & \multicolumn{1}{c|}{\textbf{2.0108}} & \textbf{0.9965} & \textbf{0.3196} & \textbf{0.2130} & \textbf{0.3027} & \multicolumn{1}{c|}{\textbf{0.8176}} & \textbf{0.6785} & \textbf{0.2169} & 0.1470          & \textbf{0.2303} & \textbf{0.6120} \\
                    &                          &                           & (0.2880)        & (0.0426)        & (0.0230)        & (0.0643)        & \multicolumn{1}{c|}{(0.1294)}        & (0.0655)        & (0.0112)        & (0.0083)        & (0.0108)        & \multicolumn{1}{c|}{(0.0462)}        & (0.0374)        & (0.0068)        & (0.0035)        & (0.0064)        & (0.0310)        \\ \bottomrule
\end{tabular}
\end{adjustbox}
\end{table}

\begin{table}[htpb]
\centering
\caption{Mean squared error (MSE) performances for scenario 1-3, model A and B with residual blocks under different sample sizes, quantile levels, and smoothing kernels. The MSEs are averaged over 50 independent trials. The baseline model is boxed if it outperforms our ConquerNet.}
\label{tab:mse-single-res}
\begin{adjustbox}{width=\textwidth,center}
\begin{tabular}{@{}cccccccccccccccccc@{}}
\toprule
\multicolumn{2}{c}{\multirow{2}{*}{}}          & \multirow{2}{*}{Method} & \multicolumn{5}{c}{$n$=1000}                                                                                 & \multicolumn{5}{c}{$n$=5000}                                                                                 & \multicolumn{5}{c}{$n$=10000}                                                           \\ \cmidrule(l){4-18} 
\multicolumn{2}{c}{}                           &                         & $\tau$=0.05     & $\tau$=0.25     & $\tau$=0.5      & $\tau$=0.75     & \multicolumn{1}{c|}{$\tau$=0.95}     & $\tau$=0.05     & $\tau$=0.25     & $\tau$=0.5      & $\tau$=0.75     & \multicolumn{1}{c|}{$\tau$=0.95}     & $\tau$=0.05     & $\tau$=0.25     & $\tau$=0.5      & $\tau$=0.75     & $\tau$=0.95     \\ \midrule
\multirow{8}{*}{S1} & \multirow{4}{*}{Model A} & Baseline                & \fbox{0.3316}          & \fbox{0.0575}          & 0.0324          & 0.0376          & \multicolumn{1}{c|}{\fbox{0.3508}}          & 0.1051          & \fbox{0.0221}         & 0.0120          & 0.0121         & \multicolumn{1}{c|}{\fbox{0.0797}}          & 0.0661          & \fbox{0.0150}          & \fbox{0.0088}          & \fbox{0.0076}          & 0.0596          \\ \cmidrule(l){3-18} 
                    &                          & Gaussian                & 0.3787          & 0.0620          & \textbf{0.0217} & \textbf{0.0262} & \multicolumn{1}{c|}{0.4817}          & \textbf{0.1006} & 0.0249          & \textbf{0.0118} & \textbf{0.0090} & \multicolumn{1}{c|}{0.0889}          & \textbf{0.0641} & 0.0245          & 0.0122          & 0.0100          & \textbf{0.0465} \\
                    &                          & Uniform                 & 0.3719          & 0.0581          & \textbf{0.0218} & \textbf{0.0234} & \multicolumn{1}{c|}{0.4008}          & 0.1139          & 0.0255          & \textbf{0.0113} & \textbf{0.0092} & \multicolumn{1}{c|}{0.0881}          & 0.0689          & 0.0239          & 0.0127          & 0.0099          & \textbf{0.0483} \\
                    &                          & Epanechnikov            & 0.3356          & 0.0629          & \textbf{0.0225} & \textbf{0.0259} & \multicolumn{1}{c|}{0.3973}          & 0.1149          & 0.0259          & \textbf{0.0115} & \textbf{0.0090} & \multicolumn{1}{c|}{0.0928}          & \textbf{0.0627} & 0.0254          & 0.0127          & 0.0094          & \textbf{0.0434} \\ \cmidrule(l){3-18} 
                    & \multirow{4}{*}{Model B} & Baseline                & \fbox{0.3457}          & \fbox{0.0716}          & 0.0373          & 0.0424          & \multicolumn{1}{c|}{\fbox{0.3231}}          & 0.1249          & \fbox{0.0293}          & 0.0191          & 0.0162          & \multicolumn{1}{c|}{\fbox{0.0620}}          & 0.0783          & \fbox{0.0200}          & \fbox{0.0121}          & 0.0118          & 0.0455          \\ \cmidrule(l){3-18} 
                    &                          & Gaussian                & 0.3525          & 0.0797          & \textbf{0.0306} & \textbf{0.0308} & \multicolumn{1}{c|}{0.5173}          & 0.1360          & 0.0386          & \textbf{0.0187} & \textbf{0.0131} & \multicolumn{1}{c|}{0.0709}          & 0.0829          & 0.0276          & 0.0156          & \textbf{0.0110} & \textbf{0.0443} \\
                    &                          & Uniform                 & 0.3596          & 0.0795          & \textbf{0.0322} & \textbf{0.0352} & \multicolumn{1}{c|}{0.4551}          & \textbf{0.1240} & 0.0389          & 0.0196          & \textbf{0.0137} & \multicolumn{1}{c|}{0.0698}          & \textbf{0.0769} & 0.0275          & 0.0153          & 0.0123          & \textbf{0.0440} \\
                    &                          & Epanechnikov            & 0.3680          & 0.0871          & \textbf{0.0348} & \textbf{0.0304} & \multicolumn{1}{c|}{0.4675}          & 0.1272          & 0.0379          & \textbf{0.0190} & \textbf{0.0131} & \multicolumn{1}{c|}{0.0669}          & 0.0810          & 0.0269          & 0.0157          & \textbf{0.0113} & \textbf{0.0379} \\ \midrule
\multirow{8}{*}{S2} & \multirow{4}{*}{Model A} & Baseline                & 0.7892          & 0.0801          & 0.0655          & 0.1017          & \multicolumn{1}{c|}{0.7819}          & 0.2290          & 0.0315          & 0.0244          & 0.0354          & \multicolumn{1}{c|}{0.2551}          & 0.1285          & 0.0213          & 0.0157          & 0.0231          & 0.1723          \\ \cmidrule(l){3-18} 
                    &                          & Gaussian                & \textbf{0.7481} & \textbf{0.0704} & \textbf{0.0577} & \textbf{0.0906} & \multicolumn{1}{c|}{\textbf{0.7275}} & \textbf{0.2076} & \textbf{0.0260} & \textbf{0.0223} & \textbf{0.0308} & \multicolumn{1}{c|}{\textbf{0.2394}} & 0.1302          & \textbf{0.0152} & \textbf{0.0134} & \textbf{0.0187} & \textbf{0.1354} \\
                    &                          & Uniform                 & 0.8190          & \textbf{0.0734} & \textbf{0.0647} & 0.1053          & \multicolumn{1}{c|}{0.8890}          & \textbf{0.2000} & \textbf{0.0240} & \textbf{0.0209} & \textbf{0.0353} & \multicolumn{1}{c|}{\textbf{0.2544}} & \textbf{0.1136} & \textbf{0.0168} & \textbf{0.0137} & \textbf{0.0188} & \textbf{0.1168} \\
                    &                          & Epanechnikov            & \textbf{0.6887} & 0.0827          & 0.0716          & 0.1185          & \multicolumn{1}{c|}{\textbf{0.6827}} & \textbf{0.1907} & \textbf{0.0265} & \textbf{0.0232} & \textbf{0.0302} & \multicolumn{1}{c|}{0.2563}          & \textbf{0.1227} & \textbf{0.0168} & \textbf{0.0137} & \textbf{0.0197} & \textbf{0.1233} \\ \cmidrule(l){3-18} 
                    & \multirow{4}{*}{Model B} & Baseline                & \fbox{0.5210}          & 0.0683          & 0.0651          & 0.0869          & \multicolumn{1}{c|}{0.7145}          & 0.1847          & 0.0261          & 0.0205          & 0.0335          & \multicolumn{1}{c|}{0.2108}          & 0.1254          & 0.0174          & 0.0142          & 0.0204          & 0.1302          \\ \cmidrule(l){3-18} 
                    &                          & Gaussian                & 0.7154          & 0.0700          & \textbf{0.0553} & \textbf{0.0827} & \multicolumn{1}{c|}{\textbf{0.5483}} & \textbf{0.1430} & \textbf{0.0237} & \textbf{0.0195} & \textbf{0.0284} & \multicolumn{1}{c|}{\textbf{0.1835}} & \textbf{0.1008} & \textbf{0.0152} & \textbf{0.0120} & \textbf{0.0187} & \textbf{0.1036} \\
                    &                          & Uniform                 & 0.6440          & \textbf{0.0605} & \textbf{0.0523} & \textbf{0.0808} & \multicolumn{1}{c|}{\textbf{0.5217}} & \textbf{0.1445} & \textbf{0.0215} & \textbf{0.0180} & \textbf{0.0318} & \multicolumn{1}{c|}{\textbf{0.1884}} & \textbf{0.1043} & \textbf{0.0149} & \textbf{0.0113} & \textbf{0.0182} & \textbf{0.1074} \\
                    &                          & Epanechnikov            & 0.6414          & 0.0879          & \textbf{0.0580} & \textbf{0.0868} & \multicolumn{1}{c|}{\textbf{0.5577}} & \textbf{0.1682} & \textbf{0.0235} & \textbf{0.0179} & \textbf{0.0268} & \multicolumn{1}{c|}{\textbf{0.1823}} & \textbf{0.1049} & \textbf{0.0143} & \textbf{0.0116} & \textbf{0.0203} & \textbf{0.1154} \\ \midrule
\multirow{8}{*}{S3} & \multirow{4}{*}{Model A} & Baseline                & 3.1729          & 0.8820          & \fbox{0.5587}          & 0.8456          & \multicolumn{1}{c|}{3.0478}          & 1.1872          & 0.3741          & 0.2298          & 0.3626          & \multicolumn{1}{c|}{1.1419}          & 0.7821          & 0.2505          & 0.1708          & 0.2539          & 0.7933          \\ \cmidrule(l){3-18} 
                    &                          & Gaussian                & \textbf{2.8543} & 0.9541          & 0.5828          & \textbf{0.7702} & \multicolumn{1}{c|}{\textbf{2.7395}} & \textbf{1.1222} & \textbf{0.3339} & \textbf{0.2143} & \textbf{0.3045} & \multicolumn{1}{c|}{\textbf{0.9996}} & \textbf{0.7556} & \textbf{0.2403} & \textbf{0.1515} & \textbf{0.2281} & \textbf{0.6931} \\
                    &                          & Uniform                 & 3.9049          & \textbf{0.8399} & 0.5665          & \textbf{0.7472} & \multicolumn{1}{c|}{\textbf{2.8420}} & \textbf{1.1147} & \textbf{0.3521} & \textbf{0.2104} & \textbf{0.3096} & \multicolumn{1}{c|}{\textbf{1.1170}} & 0.8147          & \textbf{0.2228} & \textbf{0.1468} & \textbf{0.2301} & \textbf{0.7264} \\
                    &                          & Epanechnikov            & \textbf{3.0467} & \textbf{0.8131} & 0.5638          & \textbf{0.7606} & \multicolumn{1}{c|}{\textbf{2.6823}} & 1.2444          & \textbf{0.3520} & \textbf{0.2145} & \textbf{0.3247} & \multicolumn{1}{c|}{1.2093}          & \textbf{0.7679} & \textbf{0.2242} & \textbf{0.1531} & \textbf{0.2232} & \textbf{0.7798} \\ \cmidrule(l){3-18} 
                    & \multirow{4}{*}{Model B} & Baseline                & 2.2566          & 0.7801          & 0.5042          & 0.7446          & \multicolumn{1}{c|}{2.2770}          & 0.9813          & 0.3330          & 0.2148          & 0.3146          & \multicolumn{1}{c|}{0.9632}          & 0.7292          & 0.2358          & 0.1554          & 0.2451          & 0.6422          \\ \cmidrule(l){3-18} 
                    &                          & Gaussian                & 2.5481          & \textbf{0.7569} & 0.5322          & 0.8249          & \multicolumn{1}{c|}{\textbf{2.0108}} & 0.9965          & \textbf{0.3196} & \textbf{0.2130} & \textbf{0.3027} & \multicolumn{1}{c|}{\textbf{0.8176}} & \textbf{0.6785} & \textbf{0.2169} & \textbf{0.1470} & \textbf{0.2303} & \textbf{0.6120} \\
                    &                          & Uniform                 & \textbf{1.9425} & 0.8181          & \textbf{0.4824} & 0.7644          & \multicolumn{1}{c|}{\textbf{1.9375}} & \textbf{0.9431} & 0.3394          & \textbf{0.2002} & \textbf{0.3013} & \multicolumn{1}{c|}{\textbf{0.8535}} & \textbf{0.6307} & \textbf{0.2177} & \textbf{0.1462} & \textbf{0.2086} & \textbf{0.6415} \\
                    &                          & Epanechnikov            & 2.4130          & 0.7807          & 0.5403          & \textbf{0.7333} & \multicolumn{1}{c|}{\textbf{1.9167}} & \textbf{0.8839} & \textbf{0.3262} & \textbf{0.2112} & 0.3211          & \multicolumn{1}{c|}{\textbf{0.9264}} & \textbf{0.7278} & \textbf{0.2058} & \textbf{0.1471} & \textbf{0.2277} & \textbf{0.5911} \\ \bottomrule
\end{tabular}
\end{adjustbox}
\end{table}

{We also implement the simulations for the joint estimation of multiple quantile levels under non-crossing constraints, see \citet{padilla2022quantile} for the estimation of baseline networks. For the ConquerNet, for multiple quantile levels are given in the set $\Gamma\subset(0,1)$, we estimate the quantile functions by solving
\begin{equation*}
    \begin{aligned}
        \{\hat{f}_{\tau}\}_{\tau \in \Gamma} = &\operatorname*{arg\,min}_{\{f_{\tau}\}_{\tau \in \Gamma}} \sum_{\tau\in\Gamma}\sum_{i=1}^n \ell_{\tau}(y_i - f_{\tau}(\mf x_i))\\
        & \text{subject to}\quad f_{\tau}(\mf x_i)\leq f_{\tau^\prime}(\mf x_i)\quad\forall\tau<\tau^\prime, \tau,\tau^\prime\in\Gamma, i=1,\ldots,n.
    \end{aligned}
\end{equation*}
To solve the problem, we let $\tau_0<\cdots<\tau_m$ be the elements of $\Gamma$ and solve
\begin{equation*}
\{\hat{g}_\tau\}_{\tau \in \Gamma} = \operatorname*{arg\,min}_{\{g_\tau\}_{\tau \in \Gamma}} \sum_{i=1}^n \ell_h(y_i - g_{\tau_0}(\mf x_i)) + \sum_{j=1}^m \sum_{i=1}^n \ell_h \left\{ y_i - g_{\tau_0}(\mf x_i) - \sum_{l=1}^j \log \left( 1 + e^{g_{\tau_l}(\mf x_i)} \right) \right\}
\end{equation*}
and set
\begin{equation*}
\hat{f}_{\tau_0}(\mf x) = \hat{g}_{\tau_0}(\mf x), \quad \text{and} \quad \hat{f}_{\tau_j}(\mf x) = \hat{g}_{\tau_0}(\mf x) + \sum_{l=1}^j \log \left( 1 + e^{\hat{g}_{\tau_l}(\mf x)} \right) \quad \text{for} \quad j = 1, \ldots, m,
\end{equation*}
where we recall that $\ell_h(\cdot)$ is the convolution-type smoothed quantile loss for the quantile $\tau$ in equation \eqref{eq:estimator besov}.
For the simulation data, we study the MSE performance, see Table \ref{tab:mse-multi-nores} below. In contrast to Table~\ref{tab:mse} in the paper, we can see that the joint estimation of multiple quantile levels performs much better for high ($\tau$=0.95) and low ($\tau$=0.05) quantiles, while the performance for $\tau=0.25/0.5/0.75$ declines a little bit in exchange.}

\begin{table}[htpb]
\centering
\caption{Mean squared error (MSE) performances for scenario 1-3, model A and B under different sample sizes, quantile levels, and smoothing kernels. Multiple quantile levels are trained jointly under the non-crossing constraint. The MSEs are averaged over 50 independent trials.}
\label{tab:mse-multi-nores}
\begin{adjustbox}{width=\textwidth,center}
\begin{tabular}{@{}cccccccccccccccccc@{}}
\toprule
\multicolumn{2}{c}{\multirow{2}{*}{}}          & \multirow{2}{*}{Method} & \multicolumn{5}{c}{$n$=1000}                                                                                 & \multicolumn{5}{c}{$n$=5000}                                                                                 & \multicolumn{5}{c}{$n$=10000}                                                           \\ \cmidrule(l){4-18} 
\multicolumn{2}{c}{}                           &                         & $\tau$=0.05     & $\tau$=0.25     & $\tau$=0.5      & $\tau$=0.75     & \multicolumn{1}{c|}{$\tau$=0.95}     & $\tau$=0.05     & $\tau$=0.25     & $\tau$=0.5      & $\tau$=0.75     & \multicolumn{1}{c|}{$\tau$=0.95}     & $\tau$=0.05     & $\tau$=0.25     & $\tau$=0.5      & $\tau$=0.75     & $\tau$=0.95     \\ \midrule
\multirow{8}{*}{S1} & \multirow{4}{*}{Model A} & Baseline                & 0.1312          & 0.0514          & 0.0278          & 0.0417          & \multicolumn{1}{c|}{\fbox{0.1221}}          & 0.0494          & 0.0147          & 0.0100          & 0.0146          & \multicolumn{1}{c|}{0.0397}          & \fbox{0.0305}          & 0.0088          & 0.0056          & 0.0078          & 0.0234          \\ \cmidrule(l){3-18} 
                    &                          & Gaussian                & \textbf{0.1258} & \textbf{0.0427} & \textbf{0.0224} & \textbf{0.0395} & \multicolumn{1}{c|}{0.1235}          & \textbf{0.0477} & \textbf{0.0111} & \textbf{0.0070} & \textbf{0.0120} & \multicolumn{1}{c|}{\textbf{0.0361}} & 0.0338          & \textbf{0.0076} & \textbf{0.0040} & \textbf{0.0060} & \textbf{0.0215} \\
                    &                          & Uniform                 & \textbf{0.1212} & \textbf{0.0424} & \textbf{0.0240} & \textbf{0.0444} & \multicolumn{1}{c|}{0.1331}          & \textbf{0.0471} & \textbf{0.0111} & \textbf{0.0067} & \textbf{0.0113} & \multicolumn{1}{c|}{\textbf{0.0355}} & 0.0340          & \textbf{0.0077} & \textbf{0.0039} & \textbf{0.0059} & \textbf{0.0222} \\
                    &                          & Epanechnikov            & \textbf{0.1204} & \textbf{0.0453} & \textbf{0.0235} & \textbf{0.0408} & \multicolumn{1}{c|}{0.1289}          & 0.0495          & \textbf{0.0123} & \textbf{0.0079} & \textbf{0.0128} & \multicolumn{1}{c|}{\textbf{0.0372}} & 0.0361          & \textbf{0.0080} & \textbf{0.0043} & \textbf{0.0063} & \textbf{0.0217} \\ \cmidrule(l){3-18} 
                    & \multirow{4}{*}{Model B} & Baseline                & \fbox{0.1911}          & \fbox{0.0623}          & \fbox{0.0352}          & \fbox{0.0591}          & \multicolumn{1}{c|}{\fbox{0.1821}}          & \fbox{0.0767}          & 0.0192          & 0.0119          & 0.0191          & \multicolumn{1}{c|}{0.0621}          & \fbox{0.0547}          & \fbox{0.0123}          & 0.0074          & 0.0106          & 0.0362          \\ \cmidrule(l){3-18} 
                    &                          & Gaussian                & 0.2105          & 0.0711          & 0.0378          & 0.0681          & \multicolumn{1}{c|}{0.2041}          & 0.0886          & \textbf{0.0178} & \textbf{0.0100} & \textbf{0.0169} & \multicolumn{1}{c|}{\textbf{0.0617}} & 0.0674          & 0.0126          & \textbf{0.0063} & \textbf{0.0096} & \textbf{0.0362} \\
                    &                          & Uniform                 & 0.2145          & 0.070           & 0.0378          & 0.0662          & \multicolumn{1}{c|}{0.2066}          & 0.0864          & \textbf{0.0171} & \textbf{0.0097} & \textbf{0.0159} & \multicolumn{1}{c|}{\textbf{0.0597}} & 0.0711          & 0.0134          & \textbf{0.0067} & \textbf{0.0096} & 0.0375          \\
                    &                          & Epanechnikov            & 0.2136          & 0.073           & 0.0412          & 0.0685          & \multicolumn{1}{c|}{0.2034}          & 0.0883          & \textbf{0.0183} & \textbf{0.0103} & \textbf{0.0168} & \multicolumn{1}{c|}{\textbf{0.0605}} & 0.0726          & 0.0126          & \textbf{0.0064} & \textbf{0.0099} & 0.0391          \\ \midrule
\multirow{8}{*}{S2} & \multirow{4}{*}{Model A} & Baseline                & 0.1830          & 0.0829          & 0.0712          & 0.0910          & \multicolumn{1}{c|}{0.2613}          & 0.0650          & 0.0272          & 0.0248          & 0.0321          & \multicolumn{1}{c|}{0.0888}          & 0.0469          & 0.0199          & 0.0182          & 0.0234          & 0.0641          \\ \cmidrule(l){3-18} 
                    &                          & Gaussian                & \textbf{0.1736} & \textbf{0.0718} & \textbf{0.0645} & 0.0948          & \multicolumn{1}{c|}{\textbf{0.2526}} & \textbf{0.0542} & \textbf{0.0222} & \textbf{0.0202} & \textbf{0.0262} & \multicolumn{1}{c|}{\textbf{0.0810}} & \textbf{0.0395} & \textbf{0.0140} & \textbf{0.0129} & \textbf{0.0181} & \textbf{0.0539} \\
                    &                          & Uniform                 & \textbf{0.1669} & \textbf{0.0641} & \textbf{0.0592} & \textbf{0.0827} & \multicolumn{1}{c|}{\textbf{0.2362}} & \textbf{0.0549} & \textbf{0.0233} & \textbf{0.0211} & \textbf{0.0283} & \multicolumn{1}{c|}{\textbf{0.0870}} & \textbf{0.0377} & \textbf{0.0132} & \textbf{0.0122} & \textbf{0.0165} & \textbf{0.0477} \\
                    &                          & Epanechnikov            & 0.1901          & \textbf{0.0786} & 0.0749          & 0.1066          & \multicolumn{1}{c|}{0.2634}          & \textbf{0.0535} & \textbf{0.0222} & \textbf{0.0192} & \textbf{0.0267} & \multicolumn{1}{c|}{\textbf{0.0803}} & \textbf{0.0396} & \textbf{0.0148} & \textbf{0.0130} & \textbf{0.0172} & \textbf{0.0524} \\ \cmidrule(l){3-18} 
                    & \multirow{4}{*}{Model B} & Baseline                & \fbox{0.1358}          & \fbox{0.0494}          & \fbox{0.0481}          & \fbox{0.0690}          & \multicolumn{1}{c|}{\fbox{0.2026}}          & \fbox{0.0645}          & 0.0222          & 0.0191          & 0.0276          & \multicolumn{1}{c|}{\fbox{0.0890}}          & 0.0499          & 0.0164          & 0.0143          & 0.0198          & 0.0560          \\ \cmidrule(l){3-18} 
                    &                          & Gaussian                & 0.1449          & 0.0572          & 0.0537          & 0.0780          & \multicolumn{1}{c|}{0.2060}          & 0.0737          & \textbf{0.0207} & \textbf{0.0171} & \textbf{0.0261} & \multicolumn{1}{c|}{0.0899}          & \textbf{0.0480} & \textbf{0.0126} & \textbf{0.0108} & \textbf{0.0163} & \textbf{0.0540} \\
                    &                          & Uniform                 & 0.1489          & 0.0559          & 0.0546          & 0.0835          & \multicolumn{1}{c|}{0.2252}          & 0.0728          & 0.0237          & \textbf{0.0190} & \textbf{0.0266} & \multicolumn{1}{c|}{0.0906}          & \textbf{0.0492} & \textbf{0.0132} & \textbf{0.0112} & \textbf{0.0165} & \textbf{0.0542} \\
                    &                          & Epanechnikov            & 0.1485          & 0.0567          & 0.0514          & 0.0800          & \multicolumn{1}{c|}{0.2220}          & 0.0777          & 0.0241          & 0.0208          & 0.0298          & \multicolumn{1}{c|}{0.0937}          & \textbf{0.0466} & \textbf{0.0135} & \textbf{0.0110} & \textbf{0.0159} & \textbf{0.0541} \\ \midrule
\multirow{8}{*}{S3} & \multirow{4}{*}{Model A} & Baseline                & 0.9823          & 0.7126          & 0.6113          & 0.7058          & \multicolumn{1}{c|}{1.1469}          & 0.3581          & 0.2787          & 0.2524          & 0.2740          & \multicolumn{1}{c|}{0.4018}          & 0.2699          & 0.1865          & 0.1778          & 0.2012          & 0.2859          \\ \cmidrule(l){3-18} 
                    &                          & Gaussian                & \textbf{0.8659} & \textbf{0.6814} & \textbf{0.5797} & \textbf{0.6803} & \multicolumn{1}{c|}{\textbf{1.0587}} & \textbf{0.3368} & \textbf{0.2476} & \textbf{0.2292} & \textbf{0.2494} & \multicolumn{1}{c|}{\textbf{0.3801}} & \textbf{0.2301} & \textbf{0.1597} & \textbf{0.1502} & \textbf{0.1654} & \textbf{0.2361} \\
                    &                          & Uniform                 & \textbf{0.9403} & 0.7268          & 0.6202          & \textbf{0.6916} & \multicolumn{1}{c|}{\textbf{1.0572}} & \textbf{0.3504} & \textbf{0.2589} & \textbf{0.2351} & \textbf{0.2550} & \multicolumn{1}{c|}{\textbf{0.3745}} & \textbf{0.2297} & \textbf{0.1672} & \textbf{0.1572} & \textbf{0.1720} & \textbf{0.2371} \\
                    &                          & Epanechnikov            & \textbf{0.9368} & 0.7153          & \textbf{0.5725} & \textbf{0.6167} & \multicolumn{1}{c|}{\textbf{1.0125}} & \textbf{0.3208} & \textbf{0.2507} & \textbf{0.2305} & \textbf{0.2518} & \multicolumn{1}{c|}{\textbf{0.3541}} & \textbf{0.2305} & \textbf{0.1668} & \textbf{0.1554} & \textbf{0.1670} & \textbf{0.2398} \\ \cmidrule(l){3-18} 
                    & \multirow{4}{*}{Model B} & Baseline                & \fbox{0.6933}          & \fbox{0.5887}         & \fbox{0.5255}          & \fbox{0.5704}          & \multicolumn{1}{c|}{\fbox{0.7379}}          & \fbox{0.2942}          & \fbox{0.2617}          & \fbox{0.2486}          & \fbox{0.2605}          & \multicolumn{1}{c|}{\fbox{0.3159}}          & 0.1990          & 0.1743          & 0.1704          & 0.1867          & 0.2247          \\ \cmidrule(l){3-18} 
                    &                          & Gaussian                & 0.7480          & 0.6832          & 0.5910          & 0.6510          & \multicolumn{1}{c|}{0.7794}          & 0.3240          & 0.2964          & 0.2767          & 0.2867          & \multicolumn{1}{c|}{0.3347}          & 0.2008          & 0.1824          & 0.1748          & \textbf{0.1826} & \textbf{0.2063} \\
                    &                          & Uniform                 & 0.7214          & 0.6326          & 0.5468          & 0.6127          & \multicolumn{1}{c|}{0.8332}          & 0.3078          & 0.2791          & 0.2623          & 0.2682          & \multicolumn{1}{c|}{0.3108}          & \textbf{0.1916} & \textbf{0.1700} & \textbf{0.1620} & \textbf{0.1698} & \textbf{0.1973} \\
                    &                          & Epanechnikov            & 0.8064          & 0.7124          & 0.6185          & 0.6936          & \multicolumn{1}{c|}{0.8509}          & 0.3292          & 0.2900          & 0.2717          & 0.2821          & \multicolumn{1}{c|}{0.3409}          & \textbf{0.1916} & \textbf{0.1712} & \textbf{0.1650} & \textbf{0.1725} & \textbf{0.2014} \\ \bottomrule
\end{tabular}
\end{adjustbox}
\end{table}


\begin{table}[htpb]
\centering
\caption{Mean squared error (MSE) performances for scenario 2, model A and B with 5-fold cross-validation under different sample sizes, quantile levels, and smoothing kernels. The MSEs are averaged over 50 independent trials.}
\label{tab:mse-single-nores-CV-S2}
\begin{adjustbox}{width=\textwidth,center}
\begin{tabular}{@{}cccccccccccc@{}}
\toprule
\multirow{2}{*}{$n$}   & \multirow{2}{*}{Method} & \multicolumn{5}{c}{Scenario 2, Model A}                                                                      & \multicolumn{5}{c}{Scenario 2, Model B}                                                 \\ \cmidrule(l){3-12} 
                       &                         & $\tau$=0.05     & $\tau$=0.25     & $\tau$=0.5      & $\tau$=0.75     & \multicolumn{1}{c|}{$\tau$=0.95}     & $\tau$=0.05     & $\tau$=0.25     & $\tau$=0.5      & $\tau$=0.75     & $\tau$=0.95     \\ \midrule
\multirow{4}{*}{1000}  & Baseline                & 0.9195          & \fbox{0.0714}          & 0.0586          & 0.0894          & \multicolumn{1}{c|}{0.8167}          & 0.5480          & \fbox{0.0674}          & \fbox{0.0462}          & 0.0838          & \fbox{0.4792}          \\ \cmidrule(l){2-12} 
                       & Gaussian                & \textbf{0.7859} & 0.0763          & \textbf{0.0498} & \textbf{0.0890} & \multicolumn{1}{c|}{0.8510}          & 0.5849          & 0.0751          & 0.0555          & \textbf{0.0814} & 0.6314          \\
                       & Uniform                 & 0.9437          & 0.0827          & \textbf{0.0569} & \textbf{0.0744} & \multicolumn{1}{c|}{\textbf{0.7308}} & \textbf{0.5434} & 0.0740          & 0.0499          & \textbf{0.0802} & 1.0782          \\
                       & Epanechnikov            & 1.0141          & 0.0759          & \textbf{0.0462} & 0.0919          & \multicolumn{1}{c|}{\textbf{0.6248}} & 0.5938          & 0.0688          & 0.0561          & \textbf{0.0787} & 0.6151          \\ \midrule
\multirow{4}{*}{5000}  & Baseline                & 0.3446          & 0.0258          & 0.0241          & 0.0335          & \multicolumn{1}{c|}{0.3439}          & \fbox{0.1629}          & 0.0265          & 0.0188          & 0.0272          & 0.2205          \\ \cmidrule(l){2-12} 
                       & Gaussian                & \textbf{0.2358} & 0.0275 & \textbf{0.0180} & \textbf{0.0302} & \multicolumn{1}{c|}{\textbf{0.3064}} & 0.1724          & \textbf{0.0199} & \textbf{0.0161} & \textbf{0.0261} & 0.2244          \\
                       & Uniform                 & \textbf{0.2477} & 0.0310          & \textbf{0.0184} & \textbf{0.0268} & \multicolumn{1}{c|}{\textbf{0.2732}} & 0.1804          & \textbf{0.0256} & \textbf{0.0182} & \textbf{0.0258} & \textbf{0.2045} \\
                       & Epanechnikov            & \textbf{0.2556} & \textbf{0.0237} & \textbf{0.0169} & \textbf{0.0273} & \multicolumn{1}{c|}{\textbf{0.2657}} & 0.2104          & \textbf{0.0227} & \textbf{0.0161} & 0.0295          & 0.2357          \\ \midrule
\multirow{4}{*}{10000} & Baseline                & 0.1511          & 0.0193          & 0.0164          & 0.0229          & \multicolumn{1}{c|}{0.1853}          & 0.1218          & 0.0152          & 0.0128          & 0.0223          & 0.1522          \\ \cmidrule(l){2-12} 
                       & Gaussian                & 0.1577          & \textbf{0.0162} & \textbf{0.0139} & \textbf{0.0184} & \multicolumn{1}{c|}{\textbf{0.1401}} & \textbf{0.1192} & \textbf{0.0138} & \textbf{0.0113} & \textbf{0.0189} & \textbf{0.1375} \\
                       & Uniform                 & \textbf{0.1489} & \textbf{0.0153} & \textbf{0.0120} & \textbf{0.0210} & \multicolumn{1}{c|}{\textbf{0.1508}} & 0.1255          & 0.0164          & \textbf{0.0112} & \textbf{0.0171} & 0.1592          \\
                       & Epanechnikov            & \textbf{0.1431} & \textbf{0.0165} & \textbf{0.0121} & \textbf{0.0187} & \textbf{0.1852}                      & \textbf{0.1091} & 0.0167          & \textbf{0.0113} & \textbf{0.0166} & \textbf{0.1122} \\ \bottomrule
\end{tabular}
\end{adjustbox}
\end{table}

{In addition, we implemented the MAE loss and the pinball loss in the experiment part. The MAE results for multiple quantile levels joint estimation are shown in Table \ref{tab:mae-multi-nores}. The MAE results for 5-fold cross-validation of Scenario 2 are shown in Table \ref{tab:mae-single-nores-CV-S2}. 
The results for MAE show that our ConquerNet outperform the baseline networks in most cases, which remains consistent compared to the MSE results in Tables \ref{tab:mse-multi-nores} and \ref{tab:mse-single-nores-CV-S2}, respectively. The pinball loss results for real data analysis are presented in Table \ref{tab:bmi}. These table results indicate the stability of our ConquerNet with respect to different loss metrics.}
\begin{table}[htpb]
\centering
\caption{Mean absolute error (MAE) performances for scenario 1-3, model A and B under different sample sizes, quantile levels, and smoothing kernels. Multiple quantile levels are trained jointly under the non-crossing constraint. The MAEs are averaged over 50 independent trials.}
\label{tab:mae-multi-nores}
\begin{adjustbox}{width=\textwidth,center}
\begin{tabular}{@{}cccccccccccccccccc@{}}
\toprule
\multicolumn{2}{c}{\multirow{2}{*}{}}          & \multirow{2}{*}{Method} & \multicolumn{5}{c}{$n$=1000}                                                                                 & \multicolumn{5}{c}{$n$=5000}                                                                                 & \multicolumn{5}{c}{$n$=10000}                                                           \\ \cmidrule(l){4-18} 
\multicolumn{2}{c}{}                           &                         & $\tau$=0.05     & $\tau$=0.25     & $\tau$=0.5      & $\tau$=0.75     & \multicolumn{1}{c|}{$\tau$=0.95}     & $\tau$=0.05     & $\tau$=0.25     & $\tau$=0.5      & $\tau$=0.75     & \multicolumn{1}{c|}{$\tau$=0.95}     & $\tau$=0.05     & $\tau$=0.25     & $\tau$=0.5      & $\tau$=0.75     & $\tau$=0.95     \\ \midrule
\multirow{8}{*}{S1} & \multirow{4}{*}{Model A} & Baseline                & 0.2805          & 0.1798          & 0.1310          & 0.1622          & \multicolumn{1}{c|}{0.2792}          & 0.1718          & 0.0960          & 0.0781          & 0.0948          & \multicolumn{1}{c|}{0.1583}          & \fbox{0.1314}          & 0.0723          & 0.0573          & 0.0694          & 0.1202          \\ \cmidrule(l){3-18} 
                    &                          & Gaussian                & \textbf{0.2775} & \textbf{0.1645} & \textbf{0.1204} & \textbf{0.1587} & \multicolumn{1}{c|}{\textbf{0.2780}} & \textbf{0.1678} & \textbf{0.0826} & \textbf{0.0647} & \textbf{0.0859} & \multicolumn{1}{c|}{\textbf{0.1516}} & 0.1396          & \textbf{0.0675} & \textbf{0.0483} & \textbf{0.0607} & \textbf{0.1144} \\
                    &                          & Uniform                 & \textbf{0.2713} & \textbf{0.1654} & \textbf{0.1246} & 0.1671          & \multicolumn{1}{c|}{0.2926}          & \textbf{0.1663} & \textbf{0.0829} & \textbf{0.0633} & \textbf{0.0828} & \multicolumn{1}{c|}{\textbf{0.1487}} & 0.1400          & \textbf{0.0679} & \textbf{0.0487} & \textbf{0.0605} & \textbf{0.1165} \\
                    &                          & Epanechnikov            & \textbf{0.2707} & \textbf{0.1715} & \textbf{0.1230} & \textbf{0.1607} & \multicolumn{1}{c|}{0.2870}          & \textbf{0.1703} & \textbf{0.0849} & \textbf{0.0673} & \textbf{0.0883} & \multicolumn{1}{c|}{\textbf{0.1510}} & 0.1456          & \textbf{0.0695} & \textbf{0.0507} & \textbf{0.0622} & \textbf{0.1157} \\ \cmidrule(l){3-18} 
                    & \multirow{4}{*}{Model B} & Baseline                & \fbox{0.3414}          & \fbox{0.1974}          & \fbox{0.1484}          & \fbox{0.1948}          & \multicolumn{1}{c|}{\fbox{0.3444}}          & \fbox{0.2151}          & 0.1084          & 0.0856          & 0.1094          & \multicolumn{1}{c|}{0.1983}          & \fbox{0.1807}          & 0.0876          & 0.0675          & 0.0816          & 0.1524          \\ \cmidrule(l){3-18} 
                    &                          & Gaussian                & 0.3552          & 0.2121          & 0.1551          & 0.2063          & \multicolumn{1}{c|}{0.3620}          & 0.2304          & \textbf{0.1039} & \textbf{0.0786} & \textbf{0.1032} & \multicolumn{1}{c|}{0.1993}          & 0.1971          & 0.0880          & \textbf{0.0617} & \textbf{0.0775} & \textbf{0.1516} \\
                    &                          & Uniform                 & 0.3616          & 0.2121          & 0.1546          & 0.2017          & \multicolumn{1}{c|}{0.3659}          & 0.2281          & \textbf{0.1016} & \textbf{0.0772} & \textbf{0.1003} & \multicolumn{1}{c|}{\textbf{0.1976}} & 0.2027          & 0.0910          & \textbf{0.0639} & \textbf{0.0780} & 0.1550          \\
                    &                          & Epanechnikov            & 0.3611          & 0.2133          & 0.1610          & 0.2073          & \multicolumn{1}{c|}{0.3654}          & 0.2312          & \textbf{0.1056} & \textbf{0.0797} & \textbf{0.1027} & \multicolumn{1}{c|}{\textbf{0.1975}} & 0.2062          & \textbf{0.0876} & \textbf{0.0621} & \textbf{0.0791} & 0.1589          \\ \midrule
\multirow{8}{*}{S2} & \multirow{4}{*}{Model A} & Baseline                & 0.3352          & 0.2137          & 0.1999          & 0.2265          & \multicolumn{1}{c|}{0.3765}          & 0.1955          & 0.1250          & 0.1182          & 0.1344          & \multicolumn{1}{c|}{0.2236}          & 0.1674          & 0.1060          & 0.1013          & 0.1159          & 0.1899          \\ \cmidrule(l){3-18} 
                    &                          & Gaussian                & \textbf{0.3271} & \textbf{0.1968} & \textbf{0.1882} & \textbf{0.2235} & \multicolumn{1}{c|}{\textbf{0.3686}} & \textbf{0.1794} & \textbf{0.1141} & \textbf{0.1086} & \textbf{0.1240} & \multicolumn{1}{c|}{\textbf{0.2125}} & \textbf{0.1534} & \textbf{0.0885} & \textbf{0.0843} & \textbf{0.0986} & \textbf{0.1724} \\
                    &                          & Uniform                 & \textbf{0.3256} & \textbf{0.1913} & \textbf{0.1851} & \textbf{0.2173} & \multicolumn{1}{c|}{\textbf{0.3611}} & \textbf{0.1815} & \textbf{0.1143} & \textbf{0.1088} & \textbf{0.1256} & \multicolumn{1}{c|}{\textbf{0.2153}} & \textbf{0.1494} & \textbf{0.0877} & \textbf{0.0832} & \textbf{0.0971} & \textbf{0.1661} \\
                    &                          & Epanechnikov            & 0.3432          & \textbf{0.2054} & 0.2012          & 0.2380          & \multicolumn{1}{c|}{0.3781}          & \textbf{0.1798} & \textbf{0.1117} & \textbf{0.1057} & \textbf{0.1232} & \multicolumn{1}{c|}{\textbf{0.2117}} & \textbf{0.1529} & \textbf{0.0900} & \textbf{0.0851} & \textbf{0.0987} & \textbf{0.1732} \\ \cmidrule(l){3-18} 
                    & \multirow{4}{*}{Model B} & Baseline                & \fbox{0.2947}          & \fbox{0.1700}          & \fbox{0.1680}          & \fbox{0.2033}          & \multicolumn{1}{c|}{0.3517}          & \fbox{0.1998}          & 0.1162          & 0.1073          & 0.1289          & \multicolumn{1}{c|}{0.2301}          & 0.1733          & 0.0971          & 0.0919          & 0.1086          & 0.1814          \\ \cmidrule(l){3-18} 
                    &                          & Gaussian                & 0.3027          & 0.1820          & 0.1770          & 0.2153          & \multicolumn{1}{c|}{\textbf{0.3513}} & 0.2124          & \textbf{0.1109} & \textbf{0.0990} & \textbf{0.1223} & \multicolumn{1}{c|}{\textbf{0.2294}} & \textbf{0.1693} & \textbf{0.0869} & \textbf{0.0794} & \textbf{0.0977} & \textbf{0.1781} \\
                    &                          & Uniform                 & 0.3058          & 0.1780          & 0.1763          & 0.2184          & \multicolumn{1}{c|}{0.3633}          & 0.2112          & \textbf{0.1150} & \textbf{0.1022} & \textbf{0.1233} & \multicolumn{1}{c|}{0.2320}          & \textbf{0.1726} & \textbf{0.0890} & \textbf{0.0807} & \textbf{0.0986} & \textbf{0.1801} \\
                    &                          & Epanechnikov            & 0.3035          & 0.1741          & 0.1713          & 0.2136          & \multicolumn{1}{c|}{0.3596}          & 0.2188          & 0.1194          & 0.1076          & 0.1304          & \multicolumn{1}{c|}{0.2368}          & \textbf{0.1672} & \textbf{0.0890} & \textbf{0.0801} & \textbf{0.0969} & \textbf{0.1791} \\ \midrule
\multirow{8}{*}{S3} & \multirow{4}{*}{Model A} & Baseline                & 0.7794          & 0.6503          & 0.6017          & 0.6386          & \multicolumn{1}{c|}{0.8088}          & 0.4631          & 0.4050          & 0.3822          & 0.3952          & \multicolumn{1}{c|}{0.4868}          & 0.4030          & 0.3326          & 0.3220          & 0.3402          & 0.4076          \\ \cmidrule(l){3-18} 
                    &                          & Gaussian                & \textbf{0.7381} & \textbf{0.6396} & \textbf{0.5890} & \textbf{0.6228} & \multicolumn{1}{c|}{\textbf{0.7545}} & \textbf{0.4453} & \textbf{0.3817} & \textbf{0.3651} & \textbf{0.3784} & \multicolumn{1}{c|}{\textbf{0.4728}} & \textbf{0.3708} & \textbf{0.3071} & \textbf{0.2947} & \textbf{0.3067} & \textbf{0.3738} \\
                    &                          & Uniform                 & \textbf{0.7678} & 0.6559          & 0.6132          & 0.6444          & \multicolumn{1}{c|}{\textbf{0.7743}} & \textbf{0.4575} & \textbf{0.3907} & \textbf{0.3680} & \textbf{0.3797} & \multicolumn{1}{c|}{\textbf{0.4688}} & \textbf{0.3712} & \textbf{0.3133} & \textbf{0.3017} & \textbf{0.3134} & \textbf{0.3742} \\
                    &                          & Epanechnikov            & \textbf{0.7650} & \textbf{0.6417} & \textbf{0.5790} & \textbf{0.6037} & \multicolumn{1}{c|}{\textbf{0.7491}} & \textbf{0.4391} & \textbf{0.3849} & \textbf{0.3672} & \textbf{0.3805} & \multicolumn{1}{c|}{\textbf{0.4604}} & \textbf{0.3700} & \textbf{0.3130} & \textbf{0.2988} & \textbf{0.3085} & \textbf{0.3771} \\ \cmidrule(l){3-18} 
                    & \multirow{4}{*}{Model B} & Baseline                & \fbox{0.6603}          & \fbox{0.5983}          & \fbox{0.5712}          & \fbox{0.5957}          & \multicolumn{1}{c|}{\fbox{0.6654}}          & \fbox{0.4162}          & \fbox{0.3933}          & \fbox{0.3816}          & \fbox{0.3874}          & \multicolumn{1}{c|}{0.4271}          & \fbox{0.3380}          & 0.3167          & 0.3113          & 0.3227          & 0.3520          \\ \cmidrule(l){3-18} 
                    &                          & Gaussian                & 0.6923          & 0.6354          & 0.5995          & 0.6336          & \multicolumn{1}{c|}{0.6862}          & 0.4385          & 0.4147          & 0.3997          & 0.4032          & \multicolumn{1}{c|}{0.4360}          & 0.3422          & 0.3229          & 0.3138          & \textbf{0.3173} & \textbf{0.3404} \\
                    &                          & Uniform                 & 0.6805          & 0.6133          & 0.5824          & 0.6204          & \multicolumn{1}{c|}{0.6978}          & 0.4292          & 0.4047          & 0.3915          & 0.3917          & \multicolumn{1}{c|}{\textbf{0.4207}} & 0.3353          & \textbf{0.3142} & \textbf{0.3030} & \textbf{0.3081} & \textbf{0.3357} \\
                    &                          & Epanechnikov            & 0.7175          & 0.6433          & 0.6115          & 0.6542          & \multicolumn{1}{c|}{0.7074}          & 0.4433          & 0.4119          & 0.3987          & 0.4003          & \multicolumn{1}{c|}{0.4336}          & 0.3335          & \textbf{0.3143} & \textbf{0.3040} & \textbf{0.3074} & \textbf{0.3351} \\ \bottomrule
\end{tabular}
\end{adjustbox}
\end{table}

\begin{table}[htpb]
\centering
\caption{Mean absolute error (MAE) performances for scenario 2, model A and B with 5-fold cross-validation under different sample sizes, quantile levels, and smoothing kernels. The MAEs are averaged over 50 independent trials.}
\label{tab:mae-single-nores-CV-S2}
\begin{adjustbox}{width=\textwidth,center}
\begin{tabular}{@{}cccccccccccc@{}}
\toprule
\multirow{2}{*}{$n$}   & \multirow{2}{*}{Method} & \multicolumn{5}{c}{Scenario 2, Model A}                                                                      & \multicolumn{5}{c}{Scenario 2, Model B}                                                 \\ \cmidrule(l){3-12} 
                       &                         & $\tau$=0.05     & $\tau$=0.25     & $\tau$=0.5      & $\tau$=0.75     & \multicolumn{1}{c|}{$\tau$=0.95}     & $\tau$=0.05     & $\tau$=0.25     & $\tau$=0.5      & $\tau$=0.75     & $\tau$=0.95     \\ \midrule
\multirow{4}{*}{1000}  & Baseline                & 0.6325          & 0.2014          & 0.1820          & 0.2285          & \multicolumn{1}{c|}{0.6348}          & 0.5071          & 0.1994          & \fbox{0.1684}          & 0.2264          & \fbox{0.5078}          \\ \cmidrule(l){2-12} 
                       & Gaussian                & \textbf{0.5974} & 0.2043          & \textbf{0.1710} & \textbf{0.2199} & \multicolumn{1}{c|}{\textbf{0.6320}} & \textbf{0.4969} & 0.2017          & 0.1770          & \textbf{0.2173} & 0.5303          \\
                       & Uniform                 & 0.6401          & 0.2062          & \textbf{0.1794} & \textbf{0.2024} & \multicolumn{1}{c|}{\textbf{0.5837}} & \textbf{0.5030} & \textbf{0.1993} & 0.1694          & \textbf{0.2105} & 0.5735          \\
                       & Epanechnikov            & 0.6348          & \textbf{0.1928} & \textbf{0.1663} & \textbf{0.2211} & \multicolumn{1}{c|}{\textbf{0.5642}} & 0.5129          & \textbf{0.1922} & 0.1778          & \textbf{0.2142} & 0.5409          \\ \midrule
\multirow{4}{*}{5000}  & Baseline                & 0.3797          & 0.1214          & 0.1158          & 0.1371          & \multicolumn{1}{c|}{0.4034}          & 0.2954          & 0.1237          & 0.1053          & 0.1275          & 0.3513          \\ \cmidrule(l){2-12} 
                       & Gaussian                & \textbf{0.3376} & 0.1220          & \textbf{0.1007} & \textbf{0.1302} & \multicolumn{1}{c|}{\textbf{0.4002}} & \textbf{0.2905} & \textbf{0.1085} & \textbf{0.0973} & \textbf{0.1234} & \textbf{0.3410} \\
                       & Uniform                 & \textbf{0.3501} & 0.1276          & \textbf{0.1019} & \textbf{0.1249} & \multicolumn{1}{c|}{\textbf{0.3751}} & 0.3059          & \textbf{0.1220} & \textbf{0.1020} & \textbf{0.1201} & \textbf{0.3337} \\
                       & Epanechnikov            & \textbf{0.3513} & \textbf{0.1154} & \textbf{0.0982} & \textbf{0.1239} & \multicolumn{1}{c|}{\textbf{0.3680}} & 0.3228          & \textbf{0.1151} & \textbf{0.0969} & 0.1297          & \textbf{0.3435} \\ \midrule
\multirow{4}{*}{10000} & Baseline                & 0.2858          & 0.1034          & 0.0951          & 0.1156          & \multicolumn{1}{c|}{0.3194}          & 0.2515          & 0.0949          & 0.0866          & 0.1124          & 0.2863          \\ \cmidrule(l){2-12} 
                       & Gaussian                & \textbf{0.2768} & \textbf{0.0937} & \textbf{0.0858} & \textbf{0.1014} & \multicolumn{1}{c|}{\textbf{0.2762}} & \textbf{0.2373} & \textbf{0.0882} & \textbf{0.0800} & \textbf{0.1022} & \textbf{0.2774} \\
                       & Uniform                 & \textbf{0.2738} & \textbf{0.0937} & \textbf{0.0818} & \textbf{0.1064} & \multicolumn{1}{c|}{\textbf{0.2852}} & 0.2579          & 0.0951          & \textbf{0.0802} & \textbf{0.0988} & 0.2916          \\
                       & Epanechnikov            & \textbf{0.2744} & \textbf{0.0959} & \textbf{0.0818} & \textbf{0.1020} & \textbf{0.3157}                      & \textbf{0.2383} & \textbf{0.0943} & \textbf{0.0794} & \textbf{0.0974} & \textbf{0.2553} \\ \bottomrule
\end{tabular}
\end{adjustbox}
\end{table}

In Section~\ref{sec:empirical study}, we studied the BMI data and presented the advantage of the ConquerNet for several single quantile levels. We also implement the pinball loss performance under multiple non-crossing quantile levels, see Table~\ref{tab:bmi}. Our ConquerNet maintain better performance under multiple quantile levels.

\begin{table}[htpb]
\centering
\caption{Pinball loss performance for BMI (body mass index) prediction. ``Single Quantile'' represents that the models are trained with every single quantile level. ``Multiple Quantiles'' represents that the models are trained with multiple quantile levels jointly under non-crossing constraints.}
\label{tab:bmi}
\begin{adjustbox}{width=\textwidth,center}
\begin{tabular}{@{}ccllllllllll@{}}
\toprule
\multirow{2}{*}{Gender} & \multirow{2}{*}{Method} & \multicolumn{5}{c}{Single Quantile}                                                                                                                                         & \multicolumn{5}{c}{Multiple Quantiles}                                                                                                                                 \\ \cmidrule(l){3-12} 
                        &                         & \multicolumn{1}{c}{$\tau$=0.05} & \multicolumn{1}{c}{$\tau$=0.25} & \multicolumn{1}{c}{$\tau$=0.5} & \multicolumn{1}{c}{$\tau$=0.75} & \multicolumn{1}{c|}{$\tau$=0.95}     & \multicolumn{1}{c}{$\tau$=0.05} & \multicolumn{1}{c}{$\tau$=0.25} & \multicolumn{1}{c}{$\tau$=0.5} & \multicolumn{1}{c}{$\tau$=0.75} & \multicolumn{1}{c}{$\tau$=0.95} \\ \midrule
\multirow{4}{*}{Male}   & Baseline                & 0.5231                          & 2.0638                          & 2.7387                         & 2.0534                          & \multicolumn{1}{l|}{0.5190}          & 0.5218                          & 2.0629                          & 2.7415                         & 2.0523                          & 0.5198                          \\ \cmidrule(l){2-12} 
                        & Gaussian                & \textbf{0.5221}                 & \textbf{2.0609}                 & \textbf{2.7384}                & \textbf{2.0513}                 & \multicolumn{1}{l|}{\textbf{0.5189}} & \textbf{0.5214}                 & \textbf{2.0625}                 & \textbf{2.7407}                & \textbf{2.0520}                 & \textbf{0.5192}                 \\
                        & Uniform                 & \textbf{0.5217}                 & \textbf{2.0618}                 & 2.7403                         & \textbf{2.0521}                 & \multicolumn{1}{l|}{0.5192}          & \textbf{0.5218}                 & \textbf{2.0619}                 & \textbf{2.7395}                & \textbf{2.0517}                 & 0.5201                          \\
                        & Epanechnikov            & \textbf{0.5218}                 & \textbf{2.0636}                 & 2.7389                         & \textbf{2.0518}                 & \multicolumn{1}{l|}{0.5205}          & \textbf{0.5215}                 & \textbf{2.0621}                 & \textbf{2.7399}                & \textbf{2.0521}                 & \textbf{0.5197}                 \\ \cmidrule(l){2-12} 
\multirow{4}{*}{Female} & Baseline                & 0.5218                          & 2.0596                          & 2.7366                         & 2.0555                          & \multicolumn{1}{l|}{0.5249}          & 0.5204                          & 2.0449                          & \fbox{2.7291}                         & 2.0531                          & 0.5251                          \\ \cmidrule(l){2-12} 
                        & Gaussian                & \textbf{0.5202}                 & \textbf{2.0446}                 & \textbf{2.7323}                & \textbf{2.0531}                 & \multicolumn{1}{l|}{0.5259}          & \textbf{0.5203}                 & 2.0450                          & 2.7305                         & \textbf{2.0529}                 & \textbf{0.5250}                 \\
                        & Uniform                 & \textbf{0.5211}                 & \textbf{2.0497}                 & \textbf{2.7311}                & 2.0642                          & \multicolumn{1}{l|}{0.5251}          & \textbf{0.5202}                 & \textbf{2.0449}                 & 2.7295                         & \textbf{2.0526}                 & \textbf{0.5248}                 \\
                        & Epanechnikov            & \textbf{0.5205}                 & \textbf{2.0454}                 & \textbf{2.7276}                & \textbf{2.0523}                 & \multicolumn{1}{l|}{\textbf{0.5248}} & 0.5212                          & 2.0460                          & 2.7325                         & 2.0579                          & 0.5290                          \\ \bottomrule
\end{tabular}
\end{adjustbox}
\end{table}

\subsection{Plots of MSEs by sample sizes}
{We study the plot of MSEs as a function of sample sizes to directly corroborate Theorem \ref{thm:general upper bound}. We plotted the curves of log MSEs by log sample sizes. For scenario 2 and model A, we trained the baseline model and the ConquerNet with the Gaussian smoothing kernel. We took sample sizes of $\{1000,3000,5000,7000,10000\}$ and evaluated the mean MSEs of test sets over 50 trials. The plots are shown in Figure \ref{fig:logMSE-logsamplesize-S2-A}. The curves have linear shapes, and in equation \eqref{eq:general bound l2} of Theorem \ref{thm:general upper bound}, the log MSE also has a nearly linear upper bound with respect to log sample size, which corroborates the theoretical results in Theorem \ref{thm:general upper bound}. Besides, the slopes of both baseline and ConquerNet are larger than $-1$, which is the asymptotic slope of the minimax rate $n^{-2s/(2s+d)}$. We also fit linear models for the curves in the figure and find that the slope of the ConquerNet is smaller than that of the baseline models, indicating that our upper bound is better than the baseline from the simulation.}


\begin{figure}[htpb]
    \centering
    \includegraphics[width=0.8\linewidth]{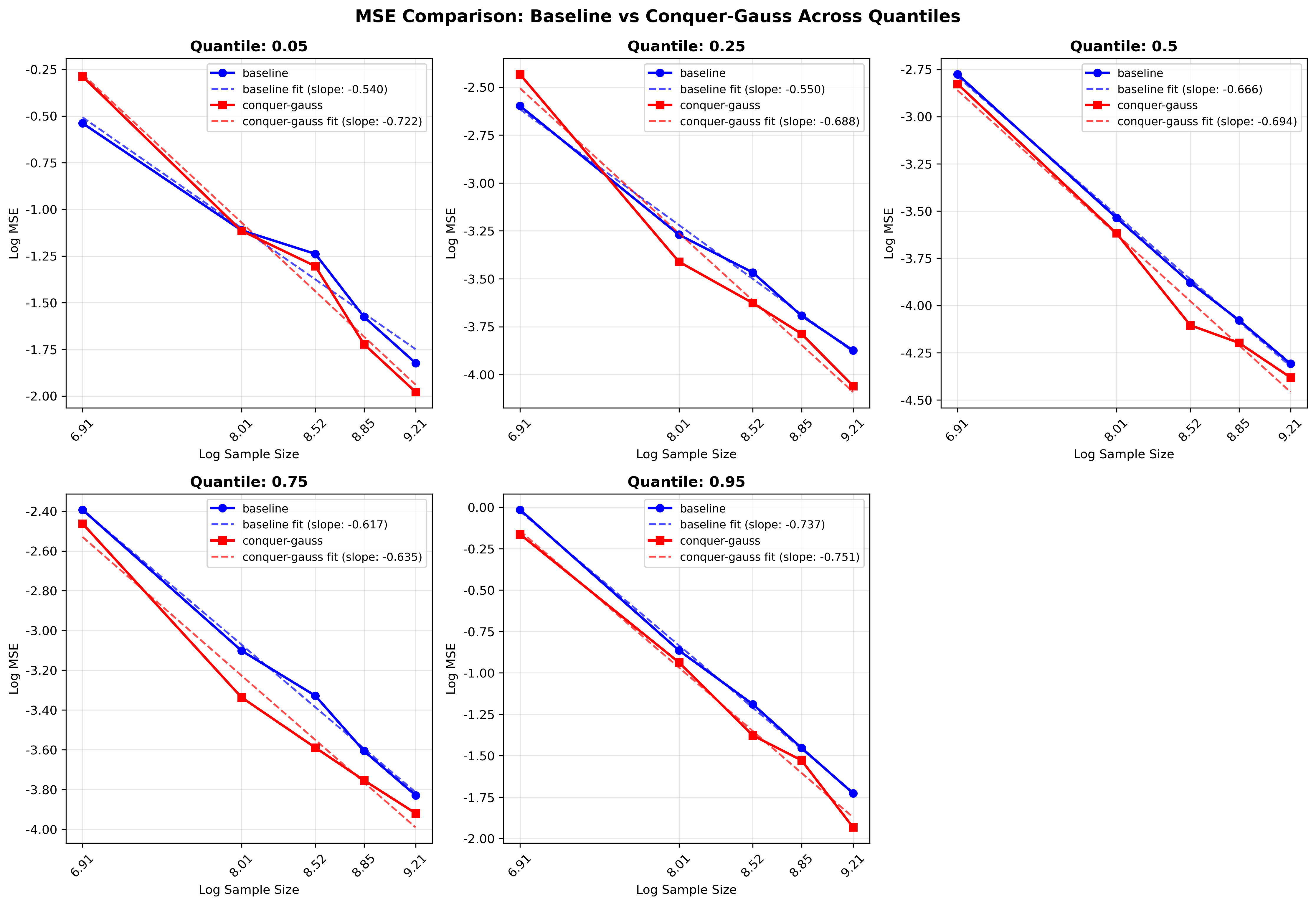}
    \caption{Plots of log MSEs by log sample sizes for scenario 2, model A. The red lines represent the ConquerNet smoothed by the Gaussian kernel with $h=0.1$. The blue lines represent the baseline network. The solid lines with points represent the line charts of simulation results. The dashed lines represent the fitted linear models. The MSEs are averaged over 50 independent trials.}
    \label{fig:logMSE-logsamplesize-S2-A}
\end{figure}

    

\subsection{Plots of training time}

\begin{figure}[htpb]
    \centering
    \includegraphics[width=1\linewidth]{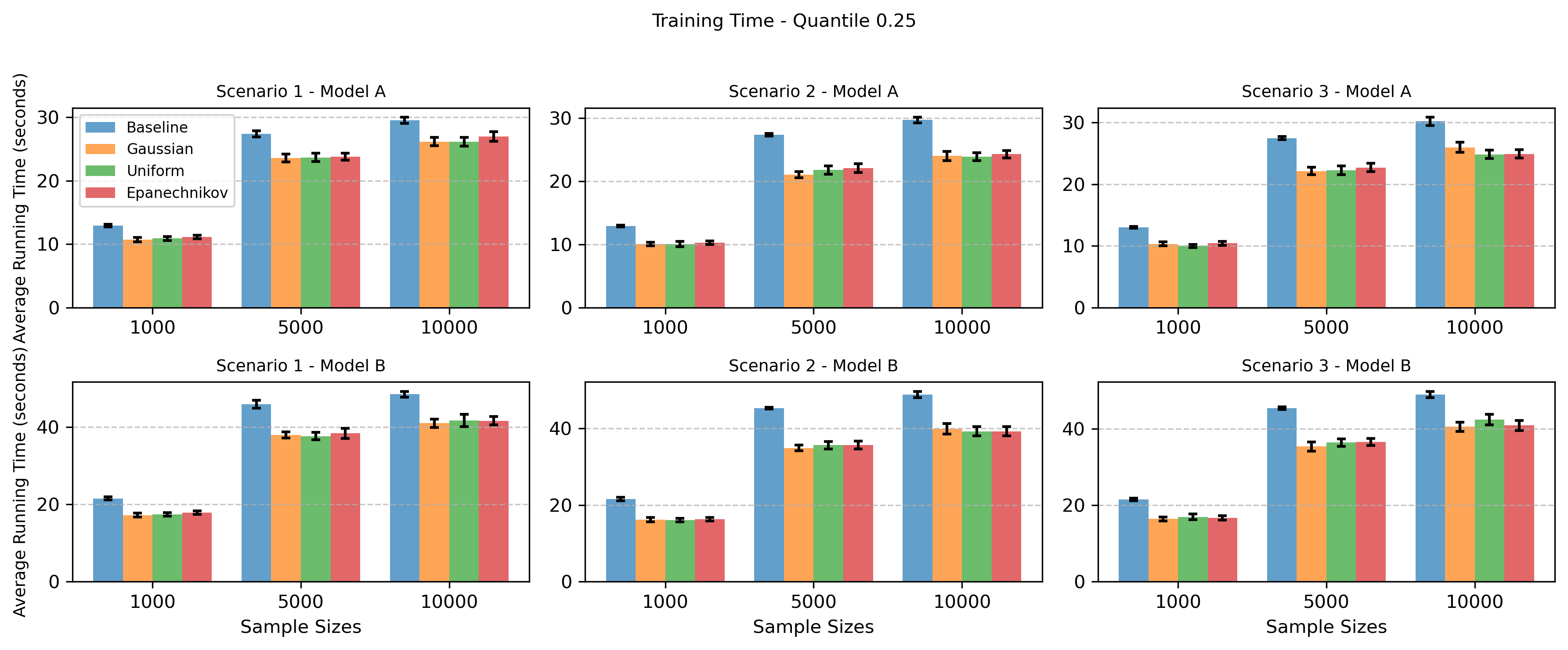}
    \caption{Bar chart with error bars with average training time over 50 trials under quantile level $\tau=0.25$. The error bars represent 95\% confidence intervals of training time for each setting.}
    \label{fig:time_tau=0.25}
\end{figure}
\begin{figure}[htpb]
    \centering
    \includegraphics[width=1\linewidth]{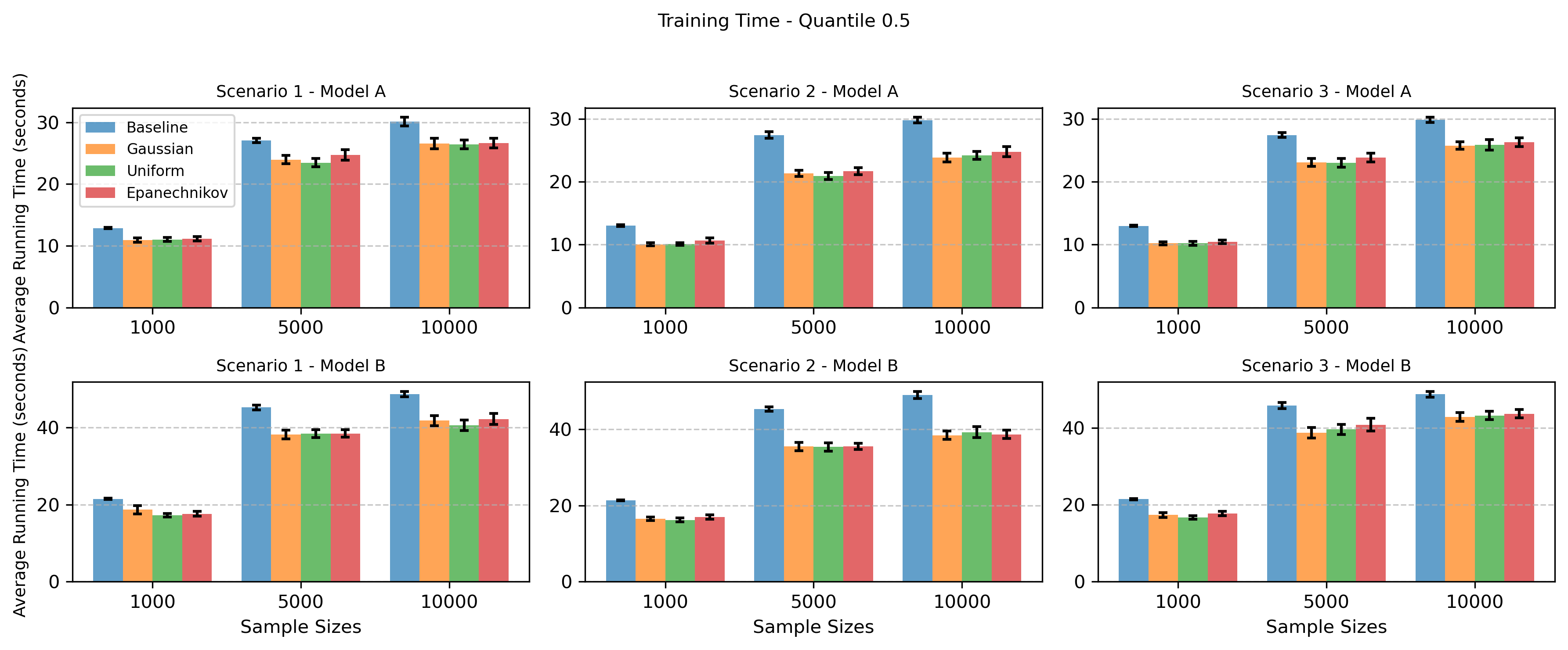}
    \caption{Bar chart with error bars with average training time over 50 trials under quantile level $\tau=0.5$. The error bars represent 95\% confidence intervals of training time for each setting.}
    \label{fig:time_tau=0.5}
\end{figure}
\begin{figure}[htpb]
    \centering
    \includegraphics[width=1\linewidth]{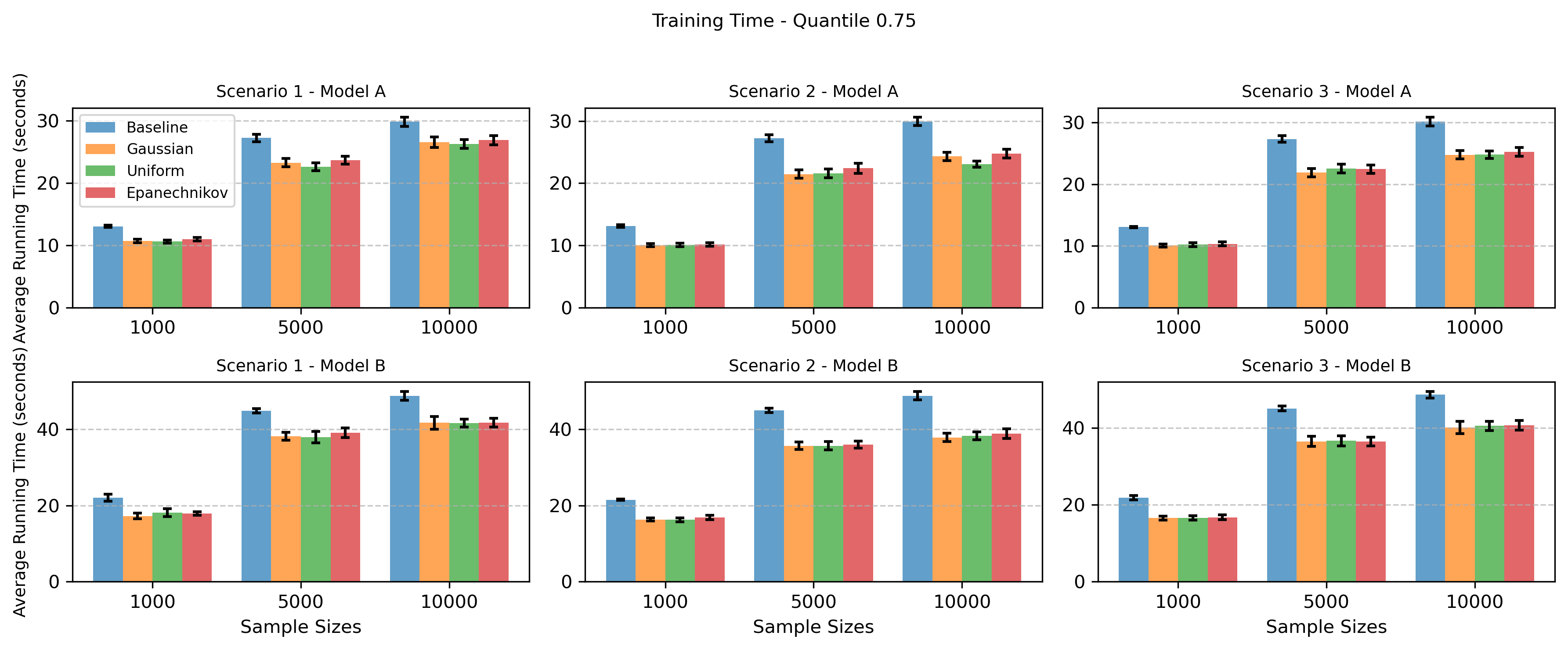}
    \caption{Bar chart with error bars with average training time over 50 trials under quantile level $\tau=0.75$. The error bars represent 95\% confidence intervals of training time for each setting.}
    \label{fig:time_tau=0.75}
\end{figure}
\begin{figure}[htpb]
    \centering
    \includegraphics[width=1\linewidth]{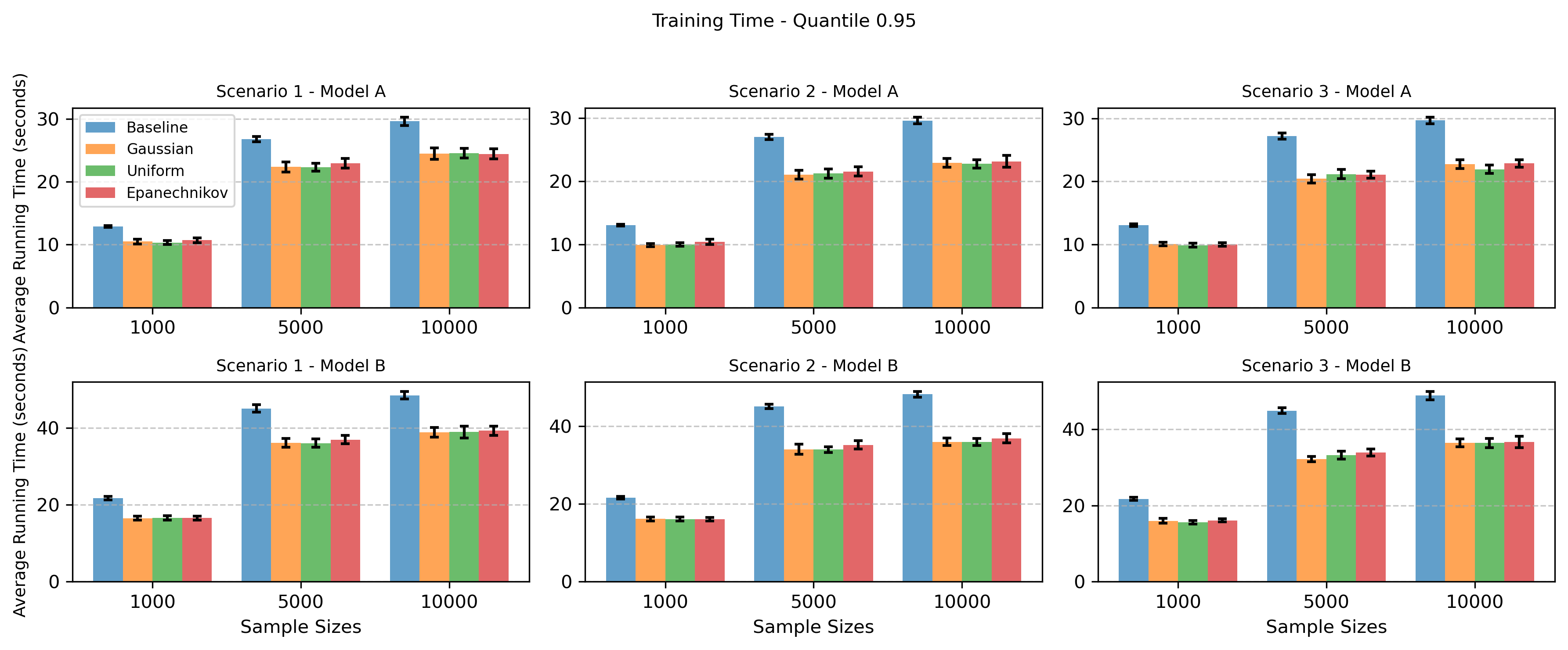}
    \caption{Bar chart with error bars with average training time over 50 trials under quantile level $\tau=0.95$. The error bars represent 95\% confidence intervals of training time for each setting.}
    \label{fig:time_tau=0.95}
\end{figure}

Figures \ref{fig:time_tau=0.05} and \ref{fig:time_tau=0.25}-\ref{fig:time_tau=0.95} show that our ConquerNet require 20\% less training time compared to the baseline method. Within the same scenario, model and sample size, the training times between 3 different kernels and between 5 different quantile levels are close to each other.
We can also conclude that shallower networks have shorter training times when the number of parameters is similar.





\end{document}